\documentclass[lettersize,journal]{IEEEtran}
\usepackage{amsmath}
\usepackage{amsfonts}
\usepackage{algorithmic}
\usepackage{algorithm}
\usepackage{array}
\usepackage[caption=false,font=normalsize,labelfont=sf,textfont=sf]{subfig}
\usepackage{textcomp}
\usepackage{stfloats}
\usepackage{url}
\usepackage{verbatim}
\usepackage{graphicx}
\usepackage{cite}
\usepackage{orcidlink}

\usepackage{multirow}
\usepackage{booktabs}
\usepackage{mathtools}
\usepackage{bm}
\usepackage{bbm}
\usepackage{amssymb} 
\usepackage{hyperref}
\usepackage{makecell}
\usepackage{amsthm}
\usepackage{cleveref}
\usepackage{thmtools}
\usepackage{thm-restate}
\usepackage{breqn}
\usepackage{enumitem}
\usepackage{empheq}
\usepackage{bibunits}

\newtheorem{lemma}{Lemma}

\newtheorem{definition}{Definition}

\crefname{equation}{}{}
\crefname{algorithm}{Alg.}{Alg.}
\crefname{line}{Line}{Lines}
\crefname{section}{Section}{Sections}
\crefname{assumption}{Assumption}{Assumptions}
\crefname{definition}{Definition}{Definitions}
\crefname{lemma}{Lemma}{Lemmas}
\crefname{figure}{Fig.}{Fig.}
\crefname{table}{Table}{Tables}
\crefname{remark}{Remark}{Remarks}
\crefname{corollary}{Corollary}{Corollaries}
\crefname{theorem}{Theorem}{Theorems}
\crefname{appendix}{Appendix}{Appendices}
\crefname{ALC@line}{line}{lines}  
\crefname{ALG@line}{line}{lines}  

\begin{document}
\title{Data Cleansing for GANs}
\author{Naoyuki Terashita\,\orcidlink{0000-0002-8057-1908}, Hiroki Ohashi\,\orcidlink{0000-0001-6970-2412}, and Satoshi Hara\,\orcidlink{0009-0009-8369-9294}
\thanks{This document is the accepted version of a manuscript that will appear in IEEE Transactions on Neural Networks and Learning Systems. In accordance with IEEE Preprint Policy, this version is posted on arXiv for rapid dissemination and does not represent the copyedited and formatted version of record. The final published version is available in the IEEE Xplore Digital Library (DOI: 10.1109/TNNLS.2025.3529540). \copyright~2025 IEEE. All rights reserved.}
\thanks{Satoshi Hara is supported by JST, PRESTO Grant Number JPMJPR20C8, Japan. The experimental results were partially produced using the computational resource of AI Bridging Cloud Infrastructure (ABCI) provided by the National Institute of Advanced Industrial Science and Technology (AIST).}
\thanks{N.~Terashita is with Hitachi, Ltd., Tokyo, Japan, and also with Osaka University, Osaka, Japan.}
\thanks{H.~Ohashi is with Hitachi, Ltd., Tokyo, Japan.}
\thanks{S.~Hara is with University of Electro-Communications, Tokyo, Japan.}
}



\maketitle

\begin{abstract}
As the application of generative adversarial networks (GANs) expands, it becomes increasingly critical to develop a unified approach that improves performance across various generative tasks.
One effective strategy that applies to any machine learning task is identifying harmful instances, whose removal improves the performance.
While previous studies have successfully estimated these harmful training instances in supervised settings, their approaches are not easily applicable to GANs.
The challenge lies in two requirements of the previous approaches that do not apply to GANs.
First, previous approaches require that the absence of a training instance directly affects the parameters. However, in the training for GANs, the instances do not directly affect the generator's parameters since they are only fed into the discriminator.
Second, previous approaches assume that the change in loss directly quantifies the harmfulness of the instance to a model's performance, while common types of GAN losses do not always reflect the generative performance.
To overcome the first challenge, we propose influence estimation methods that use the Jacobian of the generator's gradient with respect to the discriminator's parameters (and vice versa).
Such a Jacobian represents the indirect effect between two models: how removing an instance from the discriminator's training changes the generator's parameters.
Second, we propose an instance evaluation scheme that measures the harmfulness of each training instance based on how a GAN evaluation metric (e.g., Inception score) is expected to change by the instance's removal.
Furthermore, we demonstrate that removing the identified harmful instances significantly improves the generative performance on various GAN evaluation metrics.
The code is available at \hyperlink{https://github.com/hitachi-rd-cv/data-cleansing-for-gans}{https://github.com/hitachi-rd-cv/data-cleansing-for-gans}
\end{abstract}


\section{Introduction}
\IEEEPARstart{G}{enerative} adversarial network (GAN)~\cite{goodfellow2014generative} is a powerful subclass of the generative model, which is computationally more reasonable than recent diffusion-based models~\cite{dhariwal2021diffusion} and has been proven effective in various generating tasks including the super-resolution~\cite{ledig2017photo}, 3D reconstruction~\cite{wu2016learning}, and audio synthesis~\cite{donahue2018adversarial}.
As the applications of GANs expand, developing techniques that can broadly improve their generative performance becomes increasingly important.

One performance improvement technique that can be widely applied to machine learning models is to identify training instances that harm model performance.
Traditionally, statisticians manually screen a dataset for harmful instances, which leads a model to make biased predictions.
Recent \textit{influence estimation} methods \cite{Khanna2019,Hara2019} automated the screening of large datasets for deep learning settings.
Influence estimation enables efficient screening by estimating the effect of removing an individual training instance on a model's prediction without the computationally prohibitive model retraining.

Although previous studies have succeeded in identifying the harmful instances in supervised settings~\cite{Koh2017,Hara2019}, the extension of their approaches to GANs is non-trivial due to two requirements of previous approaches that do not apply to GANs.
Previous approaches require that (i) the existence or absence of a training instance directly affects the model parameters and that (ii) the decrease in the loss value represents the harmfulness of the removed training instance.
In GAN training, however, neither of the requirements is satisfied;
(i) training instances only indirectly affect the generator's parameters as they are only fed into the discriminator,  and (ii) the change in the loss of GAN does not necessarily represent how the removed instance harms the generative performance.
This is because the ability of the loss to evaluate the generator is highly dependent on the performance of the discriminator.

To this end, first, (i) we propose incorporating the Jacobian of the gradient of the discriminator's loss with respect to the generator's parameters (and vice versa) to trace how the absence of an instance in the discriminator's training affects the generator's parameters.
Using such a Jacobian, we derive two influence estimation methods that comprehensively examine the possible extensions of previous approaches.
We also provide theoretical guarantees on their estimation errors.
Second, (ii) we propose a scheme that evaluates the harmfulness of a given training instance based on its \textit{influence on GAN evaluation metric}, that is, how a GAN evaluation metric (e.g., Inception score \cite{inception}) changes after retraining where the training instance is removed from the dataset.
Using our influence estimation methods, we estimate the influence on the GAN evaluation metric without actual retraining.

Finally, we verify that the proposed influence estimation methods accurately estimate the influence on GAN evaluation metrics across different dataset settings, model architectures, and GAN evaluation metrics.
We also demonstrate that the removal of harmful instances identified by the proposed method effectively improves various GAN evaluation metrics.

\IEEEpubidadjcol

Our contributions are summarized as follows:
    \begin{itemize}
        \item We propose two different influence estimation methods that use the Jacobian of the gradient of the discriminator's loss with respect to the generator's parameters (and vice versa), which traces how the absence of an instance in the discriminator's training affects the generator's parameters. We also provide theoretical guarantees for their estimation error bounds.
        \item We propose an evaluation scheme that judges the harmfulness of a training instance based on the influence on GAN evaluation metrics. We show that our influence estimation methods enable efficient approximation of influence on GAN evaluation metrics.
        \item We demonstrate that removing harmful instances suggested by the proposed method effectively improves the generative performance with respect to various GAN evaluation metrics.
\end{itemize}

This paper extends our previous work \cite{terashita2021influence}.
The key improvements include:
\begin{itemize}
    \item Proposing an alternative influence estimator in \cref{sec:influence_estimation_aid,sec:aid_eigem}, which is more memory efficient than the previously proposed estimator presented in \cref{sec:influence_estimation_itd,sec:itd_eigem}.
    \item Providing theoretical error bounds for the previous influence estimator as well as the newly proposed alternative estimator~(\cref{sec:est_error}).
    \item Revealing pros and cons of two estimators both analytically~(\cref{sec:complexity}) and empirically~(\cref{sec:exps}).
\end{itemize}

\section*{Notation}
\noindent For a vector function $\boldsymbol{\phi}: \mathbb{R}^{m} \to \mathbb{R}^{n} $, its partial derivative is denoted by $\partial_{\boldsymbol{x}}\boldsymbol{\phi}(\boldsymbol{x}) \in \mathbb{R}^{n\times m}$.
Let $h: \mathbb{R}^{m} \times \mathbb{R}^{s} \to \mathbb{R}$ be a real-valued function.
Its partial derivative with respect to the first argument is denoted by $\nabla_{\boldsymbol{x}}h(\boldsymbol{x},\boldsymbol{y})\in\mathbb{R}^{m}$, and the Jacobians of  $\nabla_{\boldsymbol{x}}h(\boldsymbol{x},\boldsymbol{y})$ with respect to $\boldsymbol{x}$ and $\boldsymbol{y}$ are denoted by $\nabla_{\boldsymbol{x}}^2{h}(\boldsymbol{x},\boldsymbol{y}) \in\mathbb{R}^{m\times m} $ and $\nabla_{\boldsymbol{x}\boldsymbol{y}}^2{h}(\boldsymbol{x},\boldsymbol{y}) \in\mathbb{R}^{m\times s} $, respectively.
We use \(\left\Vert \cdot \right\Vert\) to denote the spectral norm for a matrix and the \(\ell_2\)-norm for a vector. 

\section{Preliminaries}

\subsection{Generative Adversarial Networks (GANs)}
\noindent We consider an unconditional GAN that consists of the generator \(G\left(\boldsymbol{\varphi},\boldsymbol{z}\right) \in \mathbb{R}^{d_{{\boldsymbol{x}}}}\) and the discriminator \(D\left(\boldsymbol{\psi}, \boldsymbol{x}\right) \in   \mathbb{R}\), where \({\boldsymbol{z}} \in \mathbb{R}^{d_{{\boldsymbol{z}}}}\) and \({\boldsymbol{x}}\in \mathbb{R}^{d_{{x}}}\) are random variables that represent a latent variable and data instance, respectively.
The parameters of generator \({{\boldsymbol{\varphi}}}\in \mathbb{R}^{d_{\varphi}}\) and discriminator \({\boldsymbol{\psi}}_{}\in \mathbb{R}^{d_{\psi}}\) are typically learned through adversarial training; \(G\) tries to generate realistic data while \(D\) tries to identify whether the data is real or generated.

\subsection{Minimax Problem for GANs}
\noindent This section presents the minimax problem for GANs.

We introduce sets of \(N_x\) training instances and \(N_z\) sampled latent variables denoted by \(\mathcal{X}\coloneqq \{\boldsymbol{x}_i\in \mathbb{R}^{d_{{\boldsymbol{x}}}}\mid i=1,2,\ldots,N_x\}\) and \(\mathcal{Z} \coloneqq \{ \boldsymbol{z}_i  \xleftarrow{\mathrm{iid}}\mathcal{N}\left(\boldsymbol{0}\in \mathbb{R}^{d_{{\boldsymbol{z}}}},\boldsymbol{I}\right) \mid i=1,2,\ldots,N_z \}\), respectively.
Here, \(\mathcal{N}\left(\boldsymbol{0},\boldsymbol{I}\right)\) denotes the multivariate normal distribution whose mean vector is the zero vector and covariance matrix is the identity matrix.
We then introduce two empirical losses that take \(\mathcal{X}\) and \(\mathcal{Z}\) as their inputs, denoted as
\begin{gather*}
    \overline{f}(\boldsymbol{\psi}, \mathcal{X}) \coloneqq  \frac{1}{\left\vert {\mathcal{X}} \right\vert} \sum_{\boldsymbol{x}\in{\mathcal{X}}} f\left(D\left(\boldsymbol{\psi},\boldsymbol{x}\right)\right), \\
    \overline{g}(\boldsymbol{\varphi}, \boldsymbol{\psi}, \mathcal{Z}) \coloneqq  \frac{1}{\left\vert {\mathcal{Z}} \right\vert} \sum_{\boldsymbol{z}\in{\mathcal{Z}}} g\left(D\left(\boldsymbol{\psi},G\left(\boldsymbol{\varphi},\boldsymbol{z}\right)\right)\right),
\end{gather*}
respectively.
Here, \(f: \mathbb{R}\to \mathbb{R}\) and \(g: \mathbb{R}\to \mathbb{R}\) denote concave functions.
Using those losses, we formulate the minimax problem for GAN as
\begin{gather}
{{\boldsymbol{\varphi}}}^{*},{{\boldsymbol{\psi}}}^{*}\in \mathrm{arg}\mathop{\min}_{ {{\boldsymbol{\varphi}}}}\mathop{\max}_{{\boldsymbol{\psi}}} V\left({{\boldsymbol{\varphi}}},{\boldsymbol{\psi}}\right), \label{eq:opt} \\
\text{where}~
V\left({{\boldsymbol{\varphi}}},{\boldsymbol{\psi}}\right) 
\coloneqq \overline{f}(\boldsymbol{\psi}, \mathcal{X}) + \overline{g}(\boldsymbol{\varphi}, \boldsymbol{\psi}, \mathcal{Z}) + R\left(\boldsymbol{\varphi}, \boldsymbol{\psi}\right)  \label{eq:def_V}. 
\end{gather}
Here, \(R(\boldsymbol{\varphi}, \boldsymbol{\psi}) \in \mathbb{R}\) is a regularizer that is strongly convex in \(\boldsymbol{\varphi}\) for any \( \boldsymbol{\psi} \), and strongly concave in \(\boldsymbol{\psi}\) for any \( \boldsymbol{\varphi} \), e.g., \(R(\boldsymbol{\varphi}, \boldsymbol{\psi}) = \frac{1}{2} (\lambda_1 \left\Vert\boldsymbol{\varphi}\right\Vert_2^2 - \lambda_2\left\Vert\boldsymbol{\psi}\right\Vert_2^2)\) with \(\lambda_1, \lambda_2 >0\).
We assume the set of solutions of \cref{eq:opt} is non-empty.
\cref{eq:opt}  is general enough to cover most formulations of GANs; the original minimax objective \cite{goodfellow2014generative} can be recovered by choosing \(f(a) = -\log  \left(1/\left(1+\exp\left(-a\right)\right) \right)\) and \(g(a) = f(-a)\), and Wasserstein GAN~\cite{arjovsky2017wasserstein} is also a case of \cref{eq:opt} where \(f(a)=a\) and \(g(a)=-a\).

\subsection{Adversarial Gradient Descent (AGD)}
\noindent We suppose that \cref{eq:opt} is solved by the gradient descent, which we call \textit{adversarial gradient descent} (\textit{AGD}).

For simplicity, this paper considers simultaneous and full-batch training; the generator and discriminator are simultaneously updated at a single step using all elements in \(\mathcal{X} \) and \(\mathcal{Z}\).
Letting \({{\boldsymbol{\theta}}}\coloneqq \left({{\boldsymbol{\varphi}}}{}^\intercal ~ {\boldsymbol{\psi}}{}^\intercal\right){}^\intercal \in \mathbb{R}^{d_{{{\boldsymbol{\theta}}}}=d_{\varphi}+d_{\psi}}\) be the concatenated parameter,  we formulate AGD as the sequence of the gradient descent step,
\begin{equation}
    {{\boldsymbol{\theta}}}^{(t+1)} = {{\boldsymbol{\theta}}}^{(t)} - \eta {\boldsymbol{v}}\left({{\boldsymbol{\theta}}}^{(t)}\right), \label{eq:update}
\end{equation}
for \(t=0,\ldots,T-1\).     
Here, \(\eta\in\mathbb{R}^+\) denotes the learning rate, and \({\boldsymbol{v}}\left({{\boldsymbol{\theta}}}\right)\) denotes a concatenated gradient defined as
 \begin{equation}
  {\boldsymbol{v}}\left({{\boldsymbol{\theta}}}\right) \coloneqq  \begin{pmatrix}
      {\nabla_{\boldsymbol{\varphi}}} V\left({{\boldsymbol{\varphi}}}, {\boldsymbol{\psi}}\right) \\ -{\nabla_{\boldsymbol{\psi}}} V\left({{\boldsymbol{\varphi}}}, {\boldsymbol{\psi}}\right) \label{eq:def_grad}
  \end{pmatrix}.
 \end{equation}
 
\section{Estimating Influence on Parameters of GANs} \label{sec:param_estimator}
\noindent This section explains the first contribution of our paper: proposing influence estimation methods using the Jacobian of gradients, which represents the indirect effect between the discriminator and generator.
Firstly, \cref{sec:def_infl} defines \textit{influence on parameters}, which represents how the removal of a training instance changes the parameters after the retraining.
We then derive two influence estimation methods in \cref{sec:infl_est}   as the extensions from \cite{Hara2019} and \cite{Koh2017}.
Finally,  \cref{sec:est_error} provides the theoretical evaluation and comparison of their estimation error bounds.

\subsection{Influence on Parameters} \label{sec:def_infl}
\noindent To define our notion of influence, we introduce another minimax problem in which a training instance indexed as \(j \in \{1,2,\ldots,N_x\}\) is removed from the dataset:
\begin{gather} 
    {{\boldsymbol{\varphi}}}_{-j}^{*}, {\boldsymbol{\psi}}_{-j}^{*} \in \mathrm{arg}\mathop{\min}_{ {{\boldsymbol{\varphi}}}}\mathop{\max}_{{\boldsymbol{\psi}}} V_{-j}\left({{\boldsymbol{\varphi}}},{\boldsymbol{\psi}}\right),  \label{eq:opt_cf} \\
\text{where}\quad V_{-j}\left({{\boldsymbol{\varphi}}},{\boldsymbol{\psi}}\right)  \coloneqq  V\left({{\boldsymbol{\varphi}}},{\boldsymbol{\psi}}\right)- \frac{1}{\left\vert {\mathcal{X}} \right\vert} f\left(D\left({{\boldsymbol{\psi}}}, \boldsymbol{x}_j\right)\right). \label{eq:def_cfV}
\end{gather}
We then define \textit{counterfactual AGD} as a gradient descent that solves \cref{eq:opt_cf}.
Counterfactual AGD starts its iteration from \({{\boldsymbol{\theta}}}^{(0)}_{-j}={{\boldsymbol{\theta}}}^{(0)}_{}\) and runs the following update step for \(t=0,\ldots,T-1\):
\begin{gather}
            {{\boldsymbol{\theta}}}^{(t+1)}_{-j} = {{\boldsymbol{\theta}}}^{(t)}_{-j} - \eta {\boldsymbol{v}}_{-j}\left({{\boldsymbol{\theta}}}^{(t)}_{-j}\right), \label{eq:update_cf}  \\
        \text{where}\quad    {\boldsymbol{v}}_{-j}\left({{\boldsymbol{\theta}}}\right)  \coloneqq  \begin{pmatrix} {\nabla_{\boldsymbol{\varphi}}} V_{-j}\left({{\boldsymbol{\varphi}}}, {\boldsymbol{\psi}}\right) \\ -{\nabla_{\boldsymbol{\psi}}} V_{-j}\left({{\boldsymbol{\varphi}}}, {\boldsymbol{\psi}}\right) \end{pmatrix}. \label{eq:def_cfgrad}
\end{gather}

We finally introduce our definition of influence on parameters as follows.
\begin{definition}
 Influence on parameters refers to \(\Delta {{\boldsymbol{\theta}}}^{(T)}_{-j} \coloneqq  {{\boldsymbol{\theta}}}^{(T)}_{-j}-{{\boldsymbol{\theta}}}^{(T)}\), denoting the changes in the concatenated parameter at the \(T\)-th step of AGD from the counterfactual AGD.
\end{definition}
\noindent In the next section, we propose two methods that estimate the influence on parameters without evaluating \(    {{\boldsymbol{\theta}}}^{(0)}_{-j}, \ldots,     {{\boldsymbol{\theta}}}^{(T)}_{-j}\).

\subsection{Influence Estimator for GANs} \label{sec:infl_est}
\noindent We propose two influence estimators that cover the possible extensions of previous approaches: \textit{Iterative Differentiation}~(\textit{ITD}) and \textit{Approximate Implicit Differentiation}~(\textit{AID}) influence estimators as the extensions of \cite{Hara2019} and \cite{Koh2017}, respectively. 

\subsubsection{ITD Influence Estimator} \label{sec:influence_estimation_itd}
\noindent ITD influence estimator employs recursive approximations of \(\Delta {{\boldsymbol{\theta}}}^{(t)}_{-j}\) for \(t=0,\ldots,T-1\), adopting the idea from \cite{Hara2019}. 

To apply the linear approximation, we introduce an interpolated gradient between \cref{eq:def_grad,eq:def_cfgrad} using \(\epsilon \in [0,1]\):
\begin{align*}
    {\boldsymbol{v}}_{-j,\epsilon}\left({{\boldsymbol{\theta}}}\right)     &=     \left(1-\epsilon\right){\boldsymbol{v}}\left({{\boldsymbol{\theta}}}\right) + \epsilon  {\boldsymbol{v}}_{-j}\left({{\boldsymbol{\theta}}}\right)\\    &={\boldsymbol{v}}\left({{\boldsymbol{\theta}}}\right) + \frac{\epsilon}{\left\vert {\mathcal{X}} \right\vert} {\nabla_{\boldsymbol{\theta}}} f\left(D\left({{\boldsymbol{\psi}}}, \boldsymbol{x}_j\right)\right) . 
\end{align*}
The linear approximation of \({\boldsymbol{v}}_{-j,1}\left({{\boldsymbol{\theta}}}^{(t)}_{-j}\right)  \) around \(\epsilon=0\) and \({{\boldsymbol{\theta}}}={{\boldsymbol{\theta}}}^{(t)}\) gives the following relation:
\begin{dmath*}
    {\boldsymbol{v}}_{-j}\left({{\boldsymbol{\theta}}}^{(t)}_{-j}\right) -{\boldsymbol{v}}\left({{\boldsymbol{\theta}}}^{(t)}\right)   \approx {\boldsymbol{J}}\left({{\boldsymbol{\theta}}}^{(t)}\right) \Delta {{\boldsymbol{\theta}}}^{(t)}_{-j}  + \frac{1}{\left\vert {\mathcal{X}} \right\vert} {\nabla_{\boldsymbol{\theta}}} f\left(D\left({{\boldsymbol{\psi}}}^{(t)}, \boldsymbol{x}_j\right)\right),
\end{dmath*}
where \({\boldsymbol{J}}\left({{\boldsymbol{\theta}}}\right)
  \coloneqq  {\partial_{\boldsymbol{\theta}}}{\boldsymbol{v}}\left({{\boldsymbol{\theta}}}\right) \).
By using this relation and subtracting \cref{eq:update} from \cref{eq:update_cf}, we have
\begin{align}
\label{eq:recur_asgd}
\Delta {{\boldsymbol{\theta}}}^{(t+1)}_{-j}  &= \Delta {{\boldsymbol{\theta}}}^{(t)}_{-j} - \eta \left({\boldsymbol{v}}_{-j}\left({{\boldsymbol{\theta}}}^{(t)}_{-j}\right) -{\boldsymbol{v}}\left({{\boldsymbol{\theta}}}^{(t)}\right) \right)  \notag \\
&\approx  \left( {\boldsymbol{I}}-\eta {\boldsymbol{J}}\left({{\boldsymbol{\theta}}}^{(t)}\right) \right)\Delta {{\boldsymbol{\theta}}}^{(t)}_{-j} + \Delta\boldsymbol{v}_{-j}\left({{\boldsymbol{\theta}}}^{(t)}\right),
\end{align}
where \(\Delta\boldsymbol{v}_{-j} \left({{\boldsymbol{\theta}}}\right) := -\frac{\eta }{\left\vert {\mathcal{X}} \right\vert} {\nabla_{\boldsymbol{\theta}}} f\left(D\left({{\boldsymbol{\psi}}}, \boldsymbol{x}_j\right)\right)\).
By recursively applying \cref{eq:recur_asgd} from \(\Delta {{\boldsymbol{\theta}}}^{(0)}_{-j}=\boldsymbol{0} \), we obtain the ITD influence estimator \(\widehat{\Delta{{\boldsymbol{\theta}}}_{-j}} \approx   \Delta{{\boldsymbol{\theta}}}^{(T)}_{-j}\) as
\begin{equation}
    \label{eq:estimator_itd}
\widehat{\Delta{{\boldsymbol{\theta}}}_{-j}} \coloneqq \sum_{t=0}^{T-1} \left(\prod_{s=t+1}^{T-1} {\boldsymbol{Z}}\left({{\boldsymbol{\theta}}}^{(s)}\right)\right)\Delta\boldsymbol{v}_{-j}\left({{\boldsymbol{\theta}}}^{(t)}\right) ,
\end{equation}
where \({\boldsymbol{Z}}\left({{\boldsymbol{\theta}}}\right)\coloneqq  {\boldsymbol{I}}-\eta {\boldsymbol{J}}\left({{\boldsymbol{\theta}}}\right)\) and \(\prod\) denotes the product operation with the multiplication order \(\prod_{t=0}^{T-1} \boldsymbol{A}_{t} = \boldsymbol{A}_{T-1} \cdots  \boldsymbol{A}_{0} \).

\subsubsection{AID Influence Estimator} \label{sec:influence_estimation_aid}
\noindent AID influence estimator approximates the influence on parameters at equilibrium, i.e., the difference between two equilibria of \cref{eq:opt,eq:opt_cf}. 
To achieve this, the estimator requires regularity assumptions on the Jacobian of gradients.
\begin{restatable}{assumption}{Jbound}\label{ass:J_bound}
Let \(\mathcal{B}({\boldsymbol{\theta}}^{*}) = \{ {\boldsymbol{\theta}} \in \mathbb{R}^{d_\theta} \mid \| {\boldsymbol{\theta}} - {\boldsymbol{\theta}}^{*} \|_2 \leq \rho \}\) represent the neighborhood around \({\boldsymbol{\theta}}^{*}=\left({{\boldsymbol{\varphi}}^{*}}{}^\intercal ~ {\boldsymbol{\psi}}^{*}{}^\intercal\right){}^\intercal\) where \(\rho > 0\).
There exists \(\mu > 0\) such that \(\frac{1}{2} \left( {\boldsymbol{J}}\left({\boldsymbol{\theta}}\right) + {\boldsymbol{J}}\left({\boldsymbol{\theta}}\right)^\intercal \right) \succeq \mu {\boldsymbol{I}}
\) for any \({\boldsymbol{\theta}} \in \mathcal{B}({\boldsymbol{\theta}}^{*})\) and \(\mathcal{X}\).
\end{restatable}
\begin{restatable}{lemma}{Zbound}  \label{lem:Z_bound}
Suppose that \cref{ass:J_bound} holds and \( \eta < \frac{2\mu}{\lambda^2}\), where \(\lambda \coloneqq  \max_{{\boldsymbol{\theta}}\in\mathcal{B}({\boldsymbol{\theta}}^{*})} \left\Vert{\boldsymbol{J}}\left({{\boldsymbol{\theta}}}\right)\right\Vert\), then \(\sigma_{\mathcal{B}} \coloneqq  \max_{{\boldsymbol{\theta}}\in \mathcal{B}({\boldsymbol{\theta}}^{*})} \left\Vert{\boldsymbol{Z}}\left({{\boldsymbol{\theta}}}\right)\right\Vert< 1\) for any \({\boldsymbol{\theta}}^{*}\) and \(\mathcal{X}\).
\end{restatable}
\begin{restatable}{assumption}{invertible}\label{ass:invertible}
 \({\boldsymbol{J}}\left({\boldsymbol{\theta}}^{*}\right)\) is invertible for any equilibrium \({\boldsymbol{\theta}}^{*}=\left({{\boldsymbol{\varphi}}^{*}}{}^\intercal ~ {\boldsymbol{\psi}}^{*}{}^\intercal\right){}^\intercal\) and \(\mathcal{X}\).
\end{restatable}
\noindent \cref{ass:J_bound} implies that the AGD iteration is locally convergent to a local Nash equilibrium, which can hold under certain regularity conditions\cite{nagarajan2017gradient}.
While local convergence is not guaranteed in general\cite{mescheder2018training}, we do not restrict the application of the AID influence estimator to those regularized settings and investigate its effectiveness beyond training scenarios where \cref{ass:J_bound} may not be met in \cref{sec:exps}.

To approximate the influence on parameters at equilibrium, i.e., \({{\boldsymbol{\theta}}}_{-j}^{*} - {{\boldsymbol{\theta}}}^{*} \) where \({\boldsymbol{\theta}}_{-j}^{*}=\left({{\boldsymbol{\varphi}}_{-j}^{*}}{}^\intercal ~ {\boldsymbol{\psi}}_{-j}^{*}{}^\intercal\right){}^\intercal\), we consider the following minimax problem with \(\epsilon \in [0,1]\), 
\begin{equation*}
       {{\boldsymbol{\varphi}}}_{-j,\epsilon}^{*}, {\boldsymbol{\psi}}_{-j,\epsilon}^{*} \in \mathrm{arg}\mathop{\min}_{ {{\boldsymbol{\varphi}}}}\mathop{\max}_{{\boldsymbol{\psi}}} V\left({{\boldsymbol{\varphi}}},{\boldsymbol{\psi}}\right)- \frac{\epsilon}{\left\vert {\mathcal{X}} \right\vert} f\left(D\left({{\boldsymbol{\psi}}}, \boldsymbol{x}_j\right)\right) ,  
\end{equation*}
which can be seen as an interpolation between \cref{eq:opt} and \cref{eq:opt_cf}.
Let \(    {{\boldsymbol{\theta}}}_{-j,\epsilon}^{*} = \left({{\boldsymbol{\varphi}}_{-j,\epsilon}^{*}}{}^\intercal ~ {\boldsymbol{\psi}}_{-j,\epsilon}^{*}{}^\intercal\right){}^\intercal\).
Since \( {{\boldsymbol{\theta}}}_{-j,0}^{*}={{\boldsymbol{\theta}}}^{*}\) and \( {{\boldsymbol{\theta}}}_{-j,1}^{*}={{\boldsymbol{\theta}}}_{-j}^{*}\), we consider the linear approximation 
\(\mathrm{d}_{\epsilon} {{{\boldsymbol{\theta}}}_{-j,0}^{*}} \approx {{\boldsymbol{\theta}}}_{-j}^{*} - {{\boldsymbol{\theta}}}^{*} \), where \(\mathrm{d}_{\epsilon}\) denotes the total derivative regarding \(\epsilon\).
To obtain \(\mathrm{d}_{\epsilon} {{{\boldsymbol{\theta}}}_{-j,0}^{*}}\),  we use the stationary point equation at the equilibrium:
\begin{equation} \label{eq:stationary_point}
    {{\boldsymbol{\theta}}}_{-j,\epsilon}^{*} = {{\boldsymbol{\theta}}}_{-j,\epsilon}^{*} - \eta \boldsymbol{v}_{-j,\epsilon}\left( {{\boldsymbol{\theta}}}_{-j,\epsilon}^{*}\right),
\end{equation}
where \(\eta > 0\) denotes a scaling coefficient\footnote{For simplicity, we slightly abuse the notation \(\eta\), which also denotes the learning rate. However, since the learning rate and the scaling factor share the same domain and are controllable, this has little impact on our discussion.}. 
We then take the total derivative of \cref{eq:stationary_point} at \(\epsilon=0\), leading to
\begin{align}
      \mathrm{d}_{\epsilon} {{{\boldsymbol{\theta}}}_{-j,0}^{*}} &= \left({\boldsymbol{I}}- \eta{\boldsymbol{J}}\left({{\boldsymbol{\theta}}}^{*}\right)\right) \mathrm{d}_{\epsilon} {{{\boldsymbol{\theta}}}^{*}} + \Delta\boldsymbol{v}_{-j}\left({{\boldsymbol{\theta}}}^{*}\right) \notag \\
        &=  \left({\boldsymbol{I}}- {\boldsymbol{Z}}\left({{\boldsymbol{\theta}}}^{*}\right)\right)^{-1}\Delta\boldsymbol{v}_{-j}\left({{\boldsymbol{\theta}}}^{*}\right) \label{eq:implicit}\\
         &\approx \sum_{m=0}^{M-1}{\boldsymbol{Z}}\left({{\boldsymbol{\theta}}}^{*}\right)^{m}\Delta\boldsymbol{v}_{-j}\left({{\boldsymbol{\theta}}}^{*}\right) , \label{eq:neumann}
\end{align}
where \cref{eq:neumann} uses \cref{lem:Z_bound} and its assumptions to allow truncated Neumann series approximation with \(M > 0\).

Replacing \({{\boldsymbol{\theta}}}^{*}\) in \cref{eq:neumann} with its early-stop version \({{\boldsymbol{\theta}}}^{(T)}\), the AID influence estimator \(  \widetilde{\Delta {{\boldsymbol{\theta}}}_{-j}}\approx   \Delta {{\boldsymbol{\theta}}}^{(T)}_{-j}  \) is obtained as
\begin{align}    \label{eq:estimator_aid}
    \widetilde{\Delta {{\boldsymbol{\theta}}}_{-j}} \coloneqq   \sum_{m=0}^{M-1} {\boldsymbol{Z}}\left({{\boldsymbol{\theta}}}^{(T)}\right)^{m} \Delta\boldsymbol{v}_{-j}\left({{\boldsymbol{\theta}}}^{(T)}\right).
\end{align}
\subsection{Role of the Jacobian of Gradients} \label{sec:jacobian_role}
\noindent The Jacobian of the concatenated gradients \(\boldsymbol{J}\left({{\boldsymbol{\theta}}}\right)\), incorporated in both estimators, plays an important role in representing the indirect effect between the generator and discriminator.
Specifically, its off-diagonal block \({\boldsymbol{J}}_{\boldsymbol{\varphi}\boldsymbol{\psi}}\coloneqq  \nabla^2_{{\boldsymbol{\varphi}}{\boldsymbol{\psi}}}V\left({{\boldsymbol{\varphi}^{(t)}}}, {\boldsymbol{\psi}^{(t)}}\right)\) represents how the absence of an instance in the discriminator's update at the \(t\)-th step affects the updated generator's parameter.

To see this role in ITD influence estimator, we break \cref{eq:recur_asgd} into block matrices as 
\begin{equation*}
\begin{pmatrix}
\Delta \boldsymbol{\varphi }_{-j}^{(t+1)}\\
\Delta \boldsymbol{\psi }_{-j}^{(t+1)}
\end{pmatrix} \approx \begin{pmatrix}
(\boldsymbol{I} -\eta \boldsymbol{H}_{\boldsymbol{\varphi }}) \Delta \boldsymbol{\varphi }_{-j}^{(t)} -\eta \boldsymbol{J}_{\boldsymbol{\varphi \psi }} \Delta \boldsymbol{\psi }_{-j}^{(t)}\\
(\boldsymbol{I} +\eta \boldsymbol{H}_{\boldsymbol{\psi }}) \Delta \boldsymbol{\psi }_{-j}^{(t)} +\eta \boldsymbol{J}_{\boldsymbol{\varphi\psi}}^{\intercal} \Delta \boldsymbol{\varphi}_{-j}^{(t)} +\Delta \boldsymbol{v}_{-j}^{\boldsymbol{\psi }}
\end{pmatrix}
\end{equation*} 
where \(\Delta\boldsymbol{v}_{-j}^{\boldsymbol{\psi}} \coloneqq  - \frac{\eta}{\left\vert {\mathcal{X}} \right\vert} {\nabla_{\boldsymbol{\psi}}} f\left(D\left({{\boldsymbol{\psi}}}^{(t)}, \boldsymbol{x}_j\right)\right) \), \({\boldsymbol{H}}_{\boldsymbol{\varphi}}\) and \({\boldsymbol{H}}_{\boldsymbol{\psi}}\) are Hessian matrices of \(V\left({{\boldsymbol{\varphi}^{(t)}}}, {\boldsymbol{\psi}^{(t)}}\right)\) with respect to \(\boldsymbol{\varphi}\) and \(\boldsymbol{\psi}\), respectively, and \(\Delta{{\boldsymbol{\varphi}}}_{-j}^{(t)} \) and \(\Delta{{\boldsymbol{\psi}}}_{-j}^{(t)} \) are the influence on \(\boldsymbol{\varphi}\) and \(\boldsymbol{\psi}\) at the \(t\)-th AGD step, respectively. 
At the \(t\)-th step, the absence of the \(j\)-th instance, denoted as \(\Delta \boldsymbol{v}_{-j}^{\boldsymbol{\psi }}\), affects only the influence on the discriminator parameter, i.e., \(\Delta \boldsymbol{\psi }_{-j}^{(t+1)}\).
Then, at the next step, \( \boldsymbol{J}_{\boldsymbol{\varphi \psi }} \Delta \boldsymbol{\psi }_{-j}^{(t+1)}\) determines how \(\Delta \boldsymbol{\psi }_{-j}^{(t+1)}\) changes the influence on the generator parameter, i.e., \(\Delta \boldsymbol{\varphi }_{-j}^{(t+2)}\).
Therefore, the off-diagonal block \({\boldsymbol{J}}_{\boldsymbol{\varphi}\boldsymbol{\psi}}\) can be regarded as transferring the indirect effect from the discriminator to the generator.

Previous influence estimation methods for supervised learning \cite{Hara2019,Koh2017} cannot handle this indirect effect between two different models because they assume the learning problem of a single combination of parameters and loss function.

\subsection{Estimation Errors} \label{sec:est_error}
\noindent This section shows the theoretical error bound of the ITD influence estimator and AID influence estimator, introducing an additional assumption.
\begin{restatable}{assumption}{smoothZ}\label{ass:smooth_Z}
\({\boldsymbol{J}}(\boldsymbol{\theta})\) is Lipschitz continuous with a constant \(L_{{\boldsymbol{J}}} \in {\mathbb{R}}^+\).
\end{restatable} 

\subsubsection{ITD Influence Estimator}
\noindent The following theorem provides the upper bound of the estimation error of the ITD influence estimator given by \cref{eq:estimator_itd}.
\begin{restatable}{theorem}{errorITD} \label{th:error_itd}
When \cref{ass:smooth_Z} holds true and \({\sigma} \coloneqq  \max_{{\boldsymbol{\theta}}\in \mathbb{R}^{d_{\boldsymbol{\theta}}}} \left\Vert{\boldsymbol{Z}}\left({{\boldsymbol{\theta}}}\right)\right\Vert \neq 1\), for any \(T\geq0\),
\begin{dmath*}
\left\Vert \widehat{\Delta \boldsymbol{\theta }_{-j}} -\Delta \boldsymbol{\theta }_{-j}^{(T)}\right\Vert \leq \frac{\eta ^{2} L_{f} L_{f^{\prime }}}{({\sigma} -1)^{2}}\left( T{\sigma}^{T-1} ({\sigma} -1)-{\sigma}^{T} +1\right) +\frac{\eta ^{3} L_{f}^{2} L_{\boldsymbol{J}}}{({\sigma} -1)^{3}}\left( {\sigma}^{2T-1} -(2T-1)({\sigma} -1){\sigma}^{T-1} -1\right) ,
\end{dmath*}
where \(L_f \coloneqq  \frac{1}{\left\vert {\mathcal{X}} \right\vert} \max_{\boldsymbol{\psi }}\Vert \nabla _{\boldsymbol{\psi }} f( D(\boldsymbol{\psi } ,x_{j}))\Vert\) and \(L_{f^\prime} \coloneqq   \frac{1}{\left\vert {\mathcal{X}} \right\vert} \max_{\boldsymbol{\psi  }}\Vert \nabla _{\boldsymbol{\psi } \boldsymbol{\theta }} f( D(\boldsymbol{\psi } ,x_{j}))\Vert \).
\end{restatable}

\noindent The convergence of \cref{th:error_itd} depends on whether  \({\sigma} > 1\) or  \({\sigma} < 1\). Given \({\sigma} > 1\), the estimation error of \(\widehat{\Delta{{\boldsymbol{\theta}}}_{-j}}\) grows at the rate of \(\exp(O(T))\), as shown below.
\begin{restatable}{corollary}{errorITDNonConvex} \label{col:error_itd_non_convex}
When \cref{ass:smooth_Z} holds true and \({\sigma} > 1\),
    \begin{equation*}
\left\Vert \widehat{\Delta \boldsymbol{\theta }_{-j}} -\Delta \boldsymbol{\theta }_{-j}^{(T)}\right\Vert \leq \frac{\eta ^{2} L_{f} L_{f^{\prime }}}{({\sigma} -1)^{2}} T{\sigma}^{T} +\frac{\eta ^{3} L_{f}^{2} L_{\boldsymbol{J}}}{({\sigma} -1)^{3}} {\sigma}^{2T-1} ,
    \end{equation*}
 for any \(T\geq0\),
\end{restatable}
\noindent However, when \({\sigma}<1\), which is guaranteed by the following assumption and its consequence, the error converges to a constant.
\begin{restatable}{assumption}{Neighbor}  \label{ass:neighbor}
\({\boldsymbol{\theta}}^{(0)}\) lies within the neighborhood of equilibrium, i.e., \({\boldsymbol{\theta}}^{(0)} \in \mathcal{B}({\boldsymbol{\theta}}^{*})\).
\end{restatable}
\begin{restatable}{lemma}{Converge}  \label{lem:converge}
When \cref{ass:J_bound}, \ref{ass:neighbor}, and \( \eta < \frac{2\mu}{\lambda^2}\) hold, then \({{\boldsymbol{\theta}}}^{(T)}\) converges to the unique equilibrium \({{{\boldsymbol{\theta}}}^{*}}\) within \(\mathcal{B}({\boldsymbol{\theta}}^{*})\) as \(T\to\infty\) for any \({{\boldsymbol{\theta}}}^{(0)} \in \mathcal{B}({\boldsymbol{\theta}}^{*})\).
\end{restatable}
\begin{restatable}{corollary}{errorITDConvex} \label{col:error_itd_convex}
When \cref{ass:J_bound}, \ref{ass:smooth_Z}, \ref{ass:neighbor}, and \( \eta < \frac{2\mu}{\lambda^2}\) hold,   \begin{dmath*}
\left\Vert \widehat{\Delta \boldsymbol{\theta }_{-j}} -\Delta \boldsymbol{\theta }_{-j}^{(T)}\right\Vert \leq \frac{\eta ^{2} L_{f} L_{f^{\prime }}}{(1-\sigma_{\mathcal{B}} )^{2}}\left( 1-\sigma_{\mathcal{B}}^{T}\right) +\frac{\eta ^{3} L_{f}^{2} L_{\boldsymbol{J}}}{(1-\sigma_{\mathcal{B}} )^{3}}\left( 1-\sigma_{\mathcal{B}}^{2T-1}\right),
    \end{dmath*}
    for any \(T\geq0\).
\end{restatable}
\noindent Note that as the learning rate \(\eta\) is controllable, the estimation error in \cref{col:error_itd_convex} converges to an arbitrarily small constant.

\subsubsection{AID Influence Estimator} \label{sec:est_err_aid}
\noindent The following theorem provides an upper bound of the approximation error of AID.
\begin{restatable}{theorem}{errorAID} \label{th:error_aid}
 When \cref{ass:J_bound,ass:smooth_Z,ass:invertible,ass:neighbor} hold true and \( \eta < \frac{2\mu}{\lambda^2}\), then for any  \(T\geq0\) and \(M>0\),
\begin{dmath*}
\left\Vert \widetilde{\Delta \boldsymbol{\theta }_{-j}} -\Delta \boldsymbol{\theta }_{-j}^{(T)}\right\Vert \leq \left(\frac{\eta L_{ f^{\prime }}}{1-\sigma_{\mathcal{B}}} +\frac{\eta ^{2} L_{f} L_{\boldsymbol{J}}}{( 1-\sigma_{\mathcal{B}})^{2}}\right) \rho\sigma_{\mathcal{B}}^{T}\left( 1-\sigma_{\mathcal{B}}^{M-1}\right) +\frac{\eta L_{f}}{1-\sigma_{\mathcal{B}}} \sigma_{\mathcal{B}}^{M} +2\rho\sigma_{\mathcal{B}}^{T} +\frac{\eta ^{3} L_{f}^{2} L_{\boldsymbol{J}}}{( 1-\sigma_{\mathcal{B}})^{3}} +\frac{\eta ^{2} L_{f} L_{f^{\prime }}}{( 1-\sigma_{\mathcal{B}})^{2}} .
\end{dmath*}
\end{restatable}
\noindent Since \(\sigma_{\mathcal{B}}<1\) from \cref{lem:Z_bound}, \cref{th:error_aid} suggests that a larger \(T\) is preferred for the smaller error. 
In addition, \cref{th:error_aid} also indicates that the optimal \(M\) may depend on \(T\); when \(T\) is small, the first term favors a small \(M\), yet when \(T\) is so large that the first term is negligible, \(M\) should be set large to suppress the second term.
This nature is actually observed in our experiment (\cref{sec:exp1}).

\section{Estimating Influence on GAN Evaluation Metrics} \label{sec:ganmetric}  
\noindent This section explains our evaluation scheme that judges the harmfulness of a given instance.
\cref{sec:def_igem} defines \textit{influence on GAN evaluation metric}, whose sign classifies whether the instance is harmful or not.
We then propose its estimators~(\cref{sec:est_eigem}) as well as their computation algorithms~(\cref{sec:alg_eigem}), incorporating the proposed ITD and AID influence estimators.

\subsection{Influence on GAN Evaluation Metric} \label{sec:def_igem}
\noindent This section defines our measure of harmfulness, which we call influence on GAN evaluation metric.

We begin with formulating the GAN evaluation metric.
Since the GAN evaluation metric typically takes a set of generated instances as its input~\cite{proscons}, we define it as a scalar function \( E : \mathbb{R}^{d_x \times N_{z^\prime}} \to \mathbb{R} \), where \( N_{z^\prime} \) is a number of generated instances.
With this definition, the GAN evaluation metric computed on a generator \(\boldsymbol{\varphi}\) is expressed as \( E (\mathcal{X}_{G}^{\prime}(\boldsymbol{\varphi}))\), where \(\mathcal{X}_{G}^{\prime}(\boldsymbol{\varphi})\coloneqq \{G(\boldsymbol{\varphi}, \boldsymbol{z}_i^\prime) \mid  \boldsymbol{z}_i^\prime 
\xleftarrow{\mathrm{iid}}\mathcal{N}\left(\boldsymbol{0},\boldsymbol{I}\right), i=1,2,\ldots,N_{z^\prime}\}\) denotes instances produced by the generator \(\boldsymbol{\varphi}\).
We suppose the latent variables \(\boldsymbol{z}_{i}^\prime\) are sampled independently from those used during the training.

We finally define the influence on GAN evaluation metric as follows:
\begin{definition}
Influence on GAN evaluation metric refers to \(\Delta E_{-j} \coloneqq  E\left(\mathcal{X}_{G}^{\prime}\left(\boldsymbol{\varphi}^{(T)}_{-j}\right)\right) -  E\left(\mathcal{X}_{G}^{\prime}\left(\boldsymbol{\varphi}^{(T)}\right)\right) \), which represents the change in the GAN evaluation metric caused by the retraining with the \(j\)-th instance removed.
\end{definition}

Our evaluation scheme judges whether a given instance is harmful or not based on the sign of its influence on GAN evaluation metric.
For instance, if larger \( E\) indicates better generative performance and \(\Delta E_{-j}\) is positive, the \(j\)-th instance is regarded as harmful.
This is because positive \(\Delta E_{-j}\) indicates that removing the \(j\)-th instance increases the GAN evaluation metric and, is thus interpreted as the presence of the \(j\)-th instance harming the generative performance.

\subsection{Estimators: ITD-EIGEM and AID-EIGEM} \label{sec:est_eigem}
\noindent This section introduces estimators of \( \Delta E_{-j} \), which we call  ITD-  and AID-based Estimator  of Influence on GAN Evaluation Metric (ITD-EIGEM and AID-EIGEM), which incorporate influence estimators \(\widehat{\Delta{{\boldsymbol{\theta}}}_{-j}}\) and \(\widetilde{\Delta{{\boldsymbol{\theta}}}_{-j}}\), respectively.

In the following, we assume that \(E\) is differentiable\footnote{E.g., IS has form of \(E (\mathcal{X}^\prime)  = \mathrm{exp}(\frac{1}{\vert \mathcal{X}^\prime \vert} \sum_{\boldsymbol{x} \in\mathcal{\mathcal{X}^\prime}}\mathbb{KL}(p_c(y \,\vert\,\boldsymbol{x})\,\Vert \,p_c(y))\), where \(p_c\) is a distribution of class label \(y\) drawn by a pre-trained learning classifier.
If \(p_c\) is differentiable, which holds in practical scenarios where classifiers are deep learning models, \(E\) is differentiable.}, which holds over common evaluation metrics, including Inception Score (IS) \cite{inception} and Fréchet inception distance (FID) \cite{fid}.
Using the differentiability of \(E\), the influence on GAN evaluation metric can be linearly approximated as 
\begin{equation} \label{eq:infl_approx}
    \Delta E_{-j} \approx  \nabla E^\intercal \Delta{{\boldsymbol{\theta}}}^{(T)}_{-j},
\end{equation}
where \(\nabla E : = ({\nabla_{\boldsymbol{\varphi}}} E(\mathcal{X}_{G}^{\prime}(\boldsymbol{\varphi}^{(T)}))^\intercal ~~ \boldsymbol{0}^\intercal)^\intercal \). 
We finally obtain our estimators ITD-EIGEM and AID-EIGEM by replacing \(\Delta{{\boldsymbol{\theta}}}^{(T)}_{-j}\) in \cref{eq:infl_approx} by its estimations \( \widehat{\Delta {{\boldsymbol{\theta}}}_{-j}}\) and \( \widetilde{\Delta {{\boldsymbol{\theta}}}_{-j}}\), respectively: 
\begin{subequations}  \label{eq:def_eigem}
\begin{empheq}[left={\Delta E_{-j} \approx\empheqlbrace}]{align}
 \nabla E^\intercal \widehat{\Delta {{\boldsymbol{\theta}}}_{-j}} & \eqqcolon   \widehat{\Delta E_{-j}} & (\text{ITD-EIGEM}), \label{eq:def_eigem_itd} \\
\nabla E^\intercal \widetilde{\Delta {{\boldsymbol{\theta}}}_{-j}} & \eqqcolon    \widetilde{\Delta E_{-j}} & (\text{AID-EIGEM}) . \label{eq:def_eigem_aid}
\end{empheq}
\end{subequations} 

\subsection{Algorithms} \label{sec:alg_eigem}
\noindent This section presents algorithms for computing \cref{eq:def_eigem}.

\subsubsection{ITD-EIGEM} \label{sec:itd_eigem}
\begin{algorithm}[t]
\caption{ITD-EIGEM}
\label{alg:lie_itd}
\begin{algorithmic}[1]
\REQUIRE{\(\boldsymbol{\theta}^{(0)},\ldots,\boldsymbol{\theta}^{(T)}\)}
\STATE{Initialize \(\boldsymbol{u} \leftarrow \nabla E  \) and \(\widehat{\Delta E_{-j}} \leftarrow 0\)}
\FOR{\(t = T-1, T-2, \ldots, 0\)}
\STATE{\(\widehat{\Delta E_{-j}} \leftarrow \widehat{\Delta E_{-j}} + \Delta\boldsymbol{v}_{-j}\left({{\boldsymbol{\theta}}}^{(t)}\right) ^\intercal\boldsymbol{u}\)} \label{line:update_infl_itd}
\STATE{\(\boldsymbol{u} \leftarrow \boldsymbol{u} - \eta \boldsymbol{J}\left(\boldsymbol{\theta}^{(t)}\right)^\intercal \boldsymbol{u}\) }  \label{line:update_u_itd} 
\ENDFOR
\RETURN{\(\widehat{\Delta E_{-j}}\)}
\end{algorithmic}
\end{algorithm}
\cref{alg:lie_itd} shows the algorithm for computing \( \widehat{\Delta E_{-j}}\), which is based on the recursive computation similar to \cite{Hara2019}.
From \cref{eq:estimator_itd,eq:def_eigem_itd}, we have
\begin{equation}
            \widehat{\Delta E_{-j}} = \sum_{t=0}^{T-1}\Delta\boldsymbol{v}_{-j}\left({{\boldsymbol{\theta}}}^{(t)}\right) ^\intercal \left(\prod_{s=t+1}^{T-1} {\boldsymbol{Z}}\left({{\boldsymbol{\theta}}}^{(s)}\right)\right)^\intercal
   \nabla E . \label{eq:infloneval_estimator_itd}
\end{equation}
\noindent When we introduce \(\boldsymbol{u}^{(t)}\coloneqq \left(\prod_{s=t+1}^{T-1} {\boldsymbol{Z}}\left({{\boldsymbol{\theta}}}^{(s)}\right)\right)^\intercal
   \nabla E \) and \(\widehat{\Delta E_{-j}^{(t)}}\coloneqq \sum_{t^\prime=t+1}^{T-1} \Delta\boldsymbol{v}_{-j}\left({{\boldsymbol{\theta}}}^{(t^\prime)}\right) ^\intercal\boldsymbol{u}^{(t^\prime)}\), \cref{eq:infloneval_estimator_itd} can be written as \(  \widehat{\Delta E_{-j}} =\widehat{\Delta E_{-j}^{(-1)}}\).
We use the fact that both \(\boldsymbol{u}^{(t)}\) and \(\widehat{\Delta E_{-j}^{(t)}}\) can be recursively computed for \(t=T-1,\ldots,0\):
\begin{equation*}
\begin{cases}
    \boldsymbol{u}^{(t-1)} ={\boldsymbol{Z}}\left({{\boldsymbol{\theta}}}^{(t)}\right)^\intercal \boldsymbol{u}^{(t)} ,  \\
    \widehat{\Delta E_{-j}^{(t-1)}}=  \widehat{\Delta E_{-j}^{(t)}} +  \Delta\boldsymbol{v}_{-j}\left({{\boldsymbol{\theta}}}^{(t)}\right) ^\intercal\boldsymbol{u}^{(t)}.
\end{cases}
\end{equation*}
With initializations \(\boldsymbol{u}^{(T-1)} = \nabla E  \) and \(\widehat{\Delta E_{-j}^{(T-1)}} = 0\), we obtain \cref{alg:lie_itd}.

\subsubsection{AID-EIGEM} \label{sec:aid_eigem}
\begin{algorithm}[t]
\caption{AID-EIGEM}
\label{alg:lie_aid}
\begin{algorithmic}[1]
\REQUIRE{\(\boldsymbol{\theta}^{(T)}\)}
\STATE{Initialize \(\boldsymbol{w} \leftarrow \nabla E \)}
\FOR{\(m = 0, 1, \ldots, M-1\)}
\STATE{\(\boldsymbol{w} \leftarrow \boldsymbol{w}   - \eta \boldsymbol{J}\left(\boldsymbol{\theta}^{(T)}\right)^\intercal \boldsymbol{w} + \nabla E\) }  \label{line:update_w_aid}
\ENDFOR
\STATE{\(\widetilde{\Delta E_{-j}} \leftarrow \Delta\boldsymbol{v}_{-j}\left({{\boldsymbol{\theta}}}^{(T)}\right) ^\intercal \boldsymbol{w} \)}  \label{line:update_infl_aid}
\RETURN{\(\widetilde{\Delta E_{-j}}\)}
\end{algorithmic}
\end{algorithm}
\noindent \cref{alg:lie_aid} for \( \widetilde{\Delta E_{-j}}\) utilizes a recursive computation that is analogous to \cite[Stochastic estimation]{Koh2017}.
Combining \cref{eq:estimator_aid,eq:def_eigem_aid}, we have
\begin{equation}
           \widetilde{\Delta E_{-j}} =  \Delta\boldsymbol{v}_{-j}\left({{\boldsymbol{\theta}}}^{(T)}\right)^\intercal 
 \sum_{m=0}^{M-1} \left({\boldsymbol{Z}}\left({{\boldsymbol{\theta}}^{(T)}}\right) ^{\intercal}\right)^{m} \label{eq:infloneval_estimator_aid}
   \nabla E.
\end{equation}
By introducing \(\boldsymbol{w}^{(m)}=\sum_{m^\prime=0}^{m-1}\left({\boldsymbol{Z}}\left({{\boldsymbol{\theta}}}^{(T)}\right)^{\intercal} \right)^{m^\prime}   \nabla E\), \cref{eq:infloneval_estimator_aid} can be rewritten as \(  \widetilde{\Delta E_{-j}} = \Delta\boldsymbol{v}_{-j}\left({{\boldsymbol{\theta}}}^{(t)}\right) ^\intercal\boldsymbol{w}^{(M)}\).
We obtain \cref{alg:lie_aid} by tracing the following recursive relation for \(m=0,\ldots,M-1\) from  \( \boldsymbol{w}^{(0)}= \nabla E \):
\begin{equation*}
    \boldsymbol{w}^{(m+1)}={\boldsymbol{Z}}\left({{\boldsymbol{\theta}}}^{(T)}\right)^\intercal \boldsymbol{w}^{(m)} + \nabla E.
\end{equation*}

\begin{table*}[t]
\caption{Comparison of Our Influence Estimation Methods\label{table:comparison}}
\centering
\begin{tabular}{lcccc}
\toprule
 \multirow{2}{*}{Method}  &  \multicolumn{2}{c}{Estimator for influence on parameters}    &  \multicolumn{2}{c}{Algorithm for influence on GAN evaluation metric}    \\ \cmidrule(lr){2-3} \cmidrule(lr){4-5}
& Assumption on the problem & Preferred AGD steps (\(T\)) for small error& Time complexity & Space complexity \\
\midrule
ITD  & None & Small &  \(O(Td_{{\theta}})\) & \(O(Td_{{\theta}})\) \\
AID & \cref{ass:J_bound,ass:invertible} & Large & \(O(Md_{{\theta}})\) & \(O(d_{{\theta}})\)  \\
\bottomrule
\end{tabular}
\end{table*}

\subsubsection{Time and Space Complexities} \label{sec:complexity}
\noindent This section compares the time and space complexity of \cref{alg:lie_itd,alg:lie_aid} with a strong emphasis on the dependency on the number of parameters \(d_{{\theta}}\) and AGD steps \(T\). 

The time complexities of \cref{alg:lie_itd,alg:lie_aid} are \(O(Td_{{\theta}})\) and \(O(Md_{{\theta}})\), respectively.
Notably, both algorithms can avoid \(O\left(d_{{\theta}}^2\right)\), which is required to explicitly compute \( \boldsymbol{J}\left(\boldsymbol{\theta}\right)\).
This is achieved by directly computing \( \boldsymbol{J}\left(\boldsymbol{\theta}\right)^\intercal \boldsymbol{u}\) and \( \boldsymbol{J}\left(\boldsymbol{\theta}\right)^\intercal \boldsymbol{w}\) utilizing the Jacobian-vector-product technique, also known as the forward mode automatic differentiation~\cite{baydin2018automatic}. 

The space complexities of \cref{alg:lie_itd,alg:lie_aid} are \(O(Td_{{\theta}})\) and \(O(d_{{\theta}})\), respectively.
\cref{alg:lie_aid} is more reasonable because it only requires maintaining vectors \(\boldsymbol{\theta}^{(T)}\), \(\boldsymbol{w}\), and \(\nabla E\), while \cref{alg:lie_itd} requires storing \(\boldsymbol{\theta}^{(0)},\ldots,\boldsymbol{\theta}^{(T)}\) during AGD steps, resulting in a space complexity being scaled by \(T\).

Considering these complexities and the number of AGD steps preferred for the ITD and AID influence estimators discussed in \cref{sec:est_error}, one can see that ITD-EIGEM and AID-EIGEM favor different training scenarios.
That is, ITD-EIGEM is effective in terms of both complexity and estimation error when \(T\) is small, while AID is advantageous in these aspects when \(T\) is large, as summarized in \cref{table:comparison}.

\subsubsection{Stability for Large Singular Values} \label{sec:stability}
\noindent The proposed ITD-EIGEM and AID-EIGEM are not guaranteed to converge when \(\max_{{\boldsymbol{\theta}}}\left\Vert \boldsymbol{I} - \eta\boldsymbol{J}(\boldsymbol{\theta})\right\Vert > 1\), as indicated in \cref{col:error_itd_non_convex,th:error_aid}, respectively.
When faced with the non-convergent, we suppress \(\sigma_{\mathcal{B}}\) using an alternative Jacobian \(\boldsymbol{J}(\boldsymbol{\theta})+\gamma\boldsymbol{I}\) when running \cref{alg:lie_itd} or \cref{alg:lie_aid}, without additional computational complexity.
Note that this modification corresponds to simply adding a regularization term \(\frac{1}{2} \gamma\left\Vert \boldsymbol{\theta} \right\Vert^2 \) to the cost function.

\subsubsection{Estimating Influence of Multiple Training Instances} \label{sec:group}
\noindent For data screening purposes, one may want to evaluate the influence on GAN evaluation metrics of all \(N_x\) instances in the dataset.
Fortunately, we have a more efficient way to perform such an evaluation than repeatedly applying \cref{alg:lie_itd} or \cref{alg:lie_aid} for every \(j=1,\ldots,N_x\).
Focusing on \cref{alg:lie_itd}, the line \ref{line:update_u_itd} is not dependent on \(j\), which means that, for each \(t\), the same \(\boldsymbol{u}\) applies to all instances.
Therefore, by modifying the line \ref{line:update_infl_itd} to update \(\widehat{\Delta E_{-j}}\) for every \(j=1,\ldots,N_x\), the influence of all instances can be estimated in one shot.
A similar implementation is also applicable to \cref{alg:lie_aid}.

\section{Related Work: Influence Estimation for Supervised Learning}  \label{sec:related_work}
\noindent This section compares our approach with the previous influence estimation methods for supervised learning.

\subsection{Base Methods of ITD- and AID-Influence Estimators}
We firstly compare our approaches with their base methods \cite{Hara2019,Koh2017}, highlighting how we tackle two assumptions in supervised learning that do not apply to GANs: the absence of a training instance directly changes the whole model parameters and the loss represents the task performance.

To see how the first assumption is used in the previous methods, we formulate previous influence estimators as special cases of ours.
When \(\mathbb{R}^{d_{\varphi}}=0\), \(g\left(\cdot\right) = 0\), and \(f\left(D\left(\boldsymbol{\psi},\boldsymbol{x}\right)\right)\) is the negated loss for supervised learning (e.g., the cross-entropy loss), \(\widehat{\Delta {{\boldsymbol{\theta}}}_{-j}}\) and \(\widetilde{\Delta {{\boldsymbol{\theta}}}_{-j}}\) are equivalent to the estimators proposed by \cite{Hara2019} and \cite{Koh2017}, respectively\footnote{We consider the full-batch version of SGD-Influence in \cite{Hara2019} and employ gradient \textit{ascent} since the problem for \(\boldsymbol{\psi}\) is the maximization problem, different from the minimization problem in \cite{Hara2019}.}. 
For instance, the recursive estimation of \cite{Hara2019} can be recovered as a case of \cref{eq:recur_asgd}:
\begin{equation}\label{eq:recur_hara}
\Delta \boldsymbol{\psi }_{-j}^{(t+1)} \approx (\boldsymbol{I} +\eta \boldsymbol{H}_{\boldsymbol{\psi }}) \Delta \boldsymbol{\psi }_{-j}^{(t)} +\Delta \boldsymbol{v}_{-j}^{\boldsymbol{\psi }} ,
\end{equation}
where we used notations introduced in \cref{sec:jacobian_role}.
\cref{eq:recur_hara} indicates that the effect of the absence of the \(j\)-th instance \( \Delta\boldsymbol{v}_{-j}^{\boldsymbol{\psi}}\) directly affects the whole model parameter \(   \Delta{{\boldsymbol{\psi}}}_{-j}^{(t+1)}\), which is not a case of ITD influence estimator for GANs.
Our ITD- and AID-influence estimators address this issue by incorporating the Jacobian of the gradient of the discriminator's loss regarding the generator's parameters (and vice versa), as explained in \cref{sec:jacobian_role}.

Regarding the second assumption, \cite{Hara2019} and \cite{Koh2017} compute the influence on the loss \(f\left(D\left(\boldsymbol{\psi},\boldsymbol{x}\right)\right)\) to evaluate the harmfulness of the training instance.
This is based on the assumption that the loss represents the task performance, which is not always true for the training of GANs.
We alleviate this problem by employing influence on GAN evaluation metrics and by using their differentiability.

\subsection{Hessian-free Influence Estimation Methods}
Another line of work is faster influence estimation methods, including \cite{chen2021hydra,pruthi2020estimating}, which have shown that rough approximations of influence are possible without considering second-order derivatives, i.e., Hessian matrices.
In contrast, the second-order derivative is essential in the influence estimation for GANs.
This is because the influence between the discriminator and the generator is measured only by the off-diagonal components of the Jacobian \(\boldsymbol{J}\left(\boldsymbol{\theta}\right)\), as explained in \cref{sec:jacobian_role}.
Thus, simple extensions of these methods would not be able to address influence estimation for GANs.

\section{Experiments} \label{sec:exps}
\noindent We evaluate the proposed method from two aspects: the accuracy of influence estimation on GAN evaluation metrics (\cref{sec:exp1}) and the performance of data cleansing using our estimation (\cref{sec:exp2}).
See our appendix for detailed settings and results.

\subsection{General Setup} 
\noindent To simulate three scenarios, one satisfies \cref{ass:J_bound,ass:invertible,ass:neighbor} and the others may not, we set up three generation tasks: Linear Quadratic GAN (LQGAN)\cite{nagarajan2017gradient} trained for 1-dimensional normal distribution (1D-Normal), Deep Convolutional GAN (DCGAN)\cite{dcgan} trained for MNIST~\cite{mnist}, and StyleGAN\cite{stylegan} fine-tuned for Animal Faces-HQ\cite{afhq}.
For each task, we chose suitable GAN evaluation metrics from the average log-likelihood (ALL), Inception score (IS)~\cite{inception}, and Fréchet inception distance (FID) \cite{fid}.

\subsubsection{LQGAN Trained for 1D-Normal with ALL Evaluation}
\noindent LQGAN is a simple GAN whose discriminator and generator are formulated as linear quadratic forms.
In our case, they are \(D(\boldsymbol{\psi }, \boldsymbol{x}) = \psi_1 x^2+\psi_2 x\) and  \(G(\boldsymbol{\varphi }, \boldsymbol{z}) =\varphi_1 z +\varphi_2\), in which both \(\boldsymbol{z}\) and \(\boldsymbol{x}\) are 1-dimensional.
LQGAN with the original minimax loss~\cite{goodfellow2014generative} ensures \cref{ass:J_bound,ass:neighbor} to hold \cite{nagarajan2017gradient}.
We also empirically verified that \({{\boldsymbol{\theta}}}^{(T)}\) converges to the analytical solution given in \cite[Theorem D.1]{nagarajan2017gradient}, which satisfies \cref{ass:invertible}.

We utilize ALL to evaluate and compute the influence on the generative performance of LQGAN.
ALL measures the likelihood of the true data under the distribution which is estimated from generated data using kernel density estimation.

\subsubsection{DCGAN Trained for MNIST with IS/FID Evaluation}
\noindent To test our methods in more practical settings where \cref{ass:J_bound} is not guaranteed, we employ DCGAN which consists of multiple convolution layers.
We train DCGAN to generate images of MNIST using the modified minimax loss proposed in \cite{goodfellow2014generative}.
We also make both AGD training and influence estimation for DCGAN more practical algorithms, namely, we used a stochastic version of AGD and influence estimation where minibatch indices and latent variables are sampled at every \(t\)-th step of AGD \cref{eq:update} and \cref{alg:lie_itd}, and \(m\)-th step of \cref{alg:lie_aid}.
Our estimator and algorithm derived on minibatch settings can be found in our appendix.

We employ IS and FID both for evaluation and influence estimation on GAN evaluation metrics.
IS utilizes the class probabilities produced by a pre-trained classifier to gauge the distinctness and variation in the classification of the generated images.
FID measures Fréchet distance between two sets of feature vectors of real images and those of generated images.
Since IS and FID require class distribution and feature vectors, respectively, we trained a CNN classifier using a validation MNIST dataset.

\subsubsection{StyleGAN Fine-tuned for Animal Faces-HQ with FID Evaluation}
\noindent We employ StyleGAN\cite{stylegan} to test our methods in a more complex setting.
StyleGAN incorporates a wide range of techniques, such as the style-based generator, mixing regularization, and noise inputs at different layers, allowing for more flexible and high-quality image generation.

We consider evaluating the influence of instances used for the fine-tuning, that updates StyleGAN pre-trained on Flickr-Faces-HQ\cite{stylegan} to generate images of the cat category from Animal Faces-HG~dataset\cite{afhq}, which we call AFHQ-CAT.
Recent studies have demonstrated that fine-tuning the generator can be achieved effectively by training a small set of parameters using Low-Rank-Adaptation (LoRA)\cite{hu2021lora}. 
In this study, we train LoRA parameters for both the generator and discriminator, and we treat these LoRA parameters as \(\boldsymbol{\varphi}\) and \(\boldsymbol{\psi }\) in our formulation.
To perform influence estimation and evaluate the performance of the StyleGAN, we employ FID by extracting features from InceptionV3 \cite{inceptionnet} following the original definition\cite{fid}.
    
Apart from the architectural complexity, our approach must address a more complicated training setup.
Recent GANs commonly employ various optimization techniques, including the moving average of the generator \cite{yaz2018unusual} and momentum-based optimizers, such as RMSProp and Adam\cite{kingma2014adam}.
Our ITD influence estimator is based on the assumption that GANs are trained using the vanilla gradient descent, requiring an adjustment to align with these optimization techniques. 
Hence, we have introduced a more practical ITD-influence estimator derived from the training iterations with the above techniques. 
Detailed implementation of the introduced estimator is provided in our appendix.

\begin{figure*}[t]
    \centering
        \begin{minipage}{0.90\linewidth}
        \centering
    \includegraphics[clip,width=0.9\linewidth]{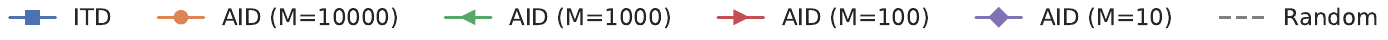}
    \end{minipage}
    
    \begin{minipage}{0.49\linewidth}
    \centering
            \subfloat[][Linear Quadratic GAN]{
        \includegraphics[trim={.0cm .7cm .5cm 1cm},clip,width=\linewidth]{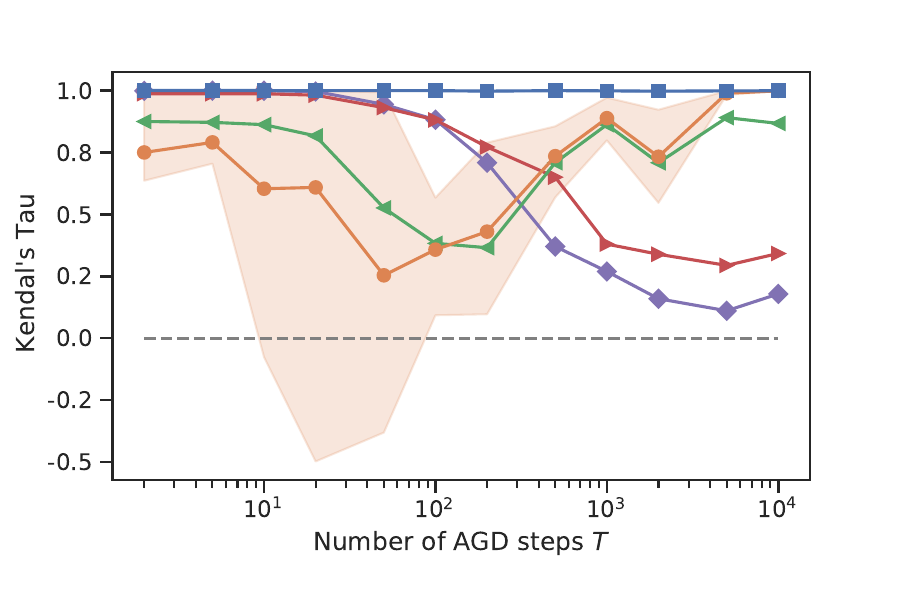}
    \label{sub:toy_valid}}
    \end{minipage}
    \begin{minipage}{0.49\linewidth}
    \centering
    \subfloat[][DCGAN]{
               \includegraphics[trim={.0cm .7cm .5cm 1cm},clip,width=\linewidth]{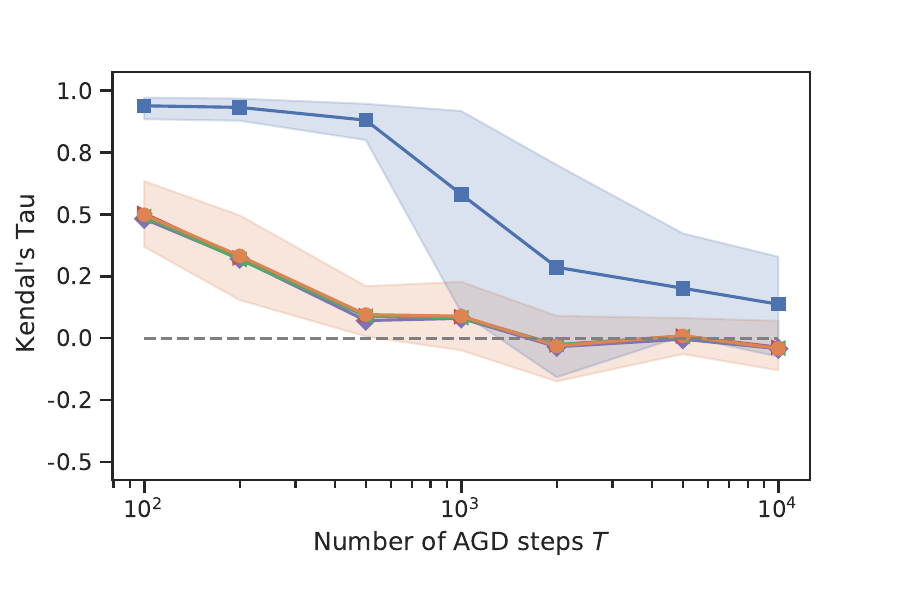}
    \label{sub:mnist_valid_is}} 
    \end{minipage}

    \caption[]{
The average Kendall's Tau calculated from the true and estimated influence values on ALL for LQGAN~\subref{sub:toy_valid} and IS for DCGAN~\subref{sub:mnist_valid_is} of 100 instances.
The error bars show the 10\% and 90\% percentiles of Kendall's Tau obtained from iterative experiments.
To enhance visibility, we excluded the error bars of AID (M=1,10,100,1000).
}
    \label{fig:valid}
\end{figure*}

\subsection{Experiment 1: Estimation Accuracy} \label{sec:exp1}
\noindent This section empirically evaluates how accurately our ITD-EIGEM and AID-EIGEM can estimate the influence on GAN evaluation metrics.
Moreover, we evaluate how selections of AGD steps \(T\) and the depth of Neuman series approximation \(M\) affect the estimation.

    \begin{figure*}[t!]
        \centering
        \begin{minipage}{0.45\linewidth}
            \subfloat[][Linear Quadratic GAN]{
            \includegraphics[trim={.0cm .8cm 1.5cm 1cm},clip,width=\linewidth]{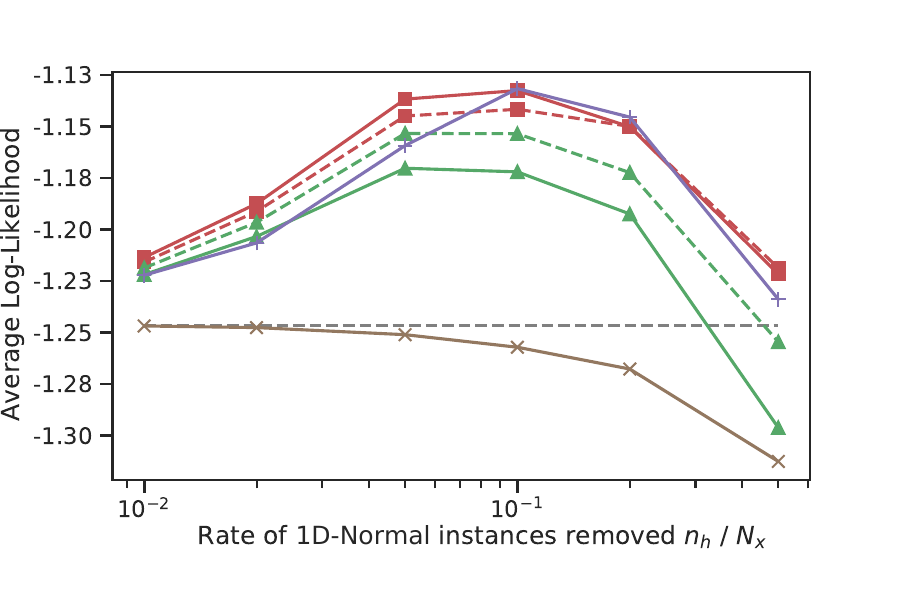}
            \label{sub:clean_ll}
        }
        \end{minipage}
         \hspace{0.5cm}
        \begin{minipage}{0.45\linewidth}
            \centering
            \includegraphics[trim={.0cm -1.5cm .0cm .0cm},clip,width=0.6\linewidth]{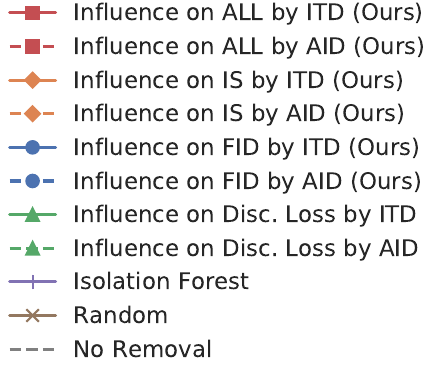}
        \end{minipage}
        \newline
        \begin{minipage}{0.45\linewidth}
        \centering
        \subfloat[][DCGAN (tested for IS \& full-epoch retraining)]{
            \includegraphics[trim={.0cm .8cm 1.5cm 1cm},clip,width=\linewidth]{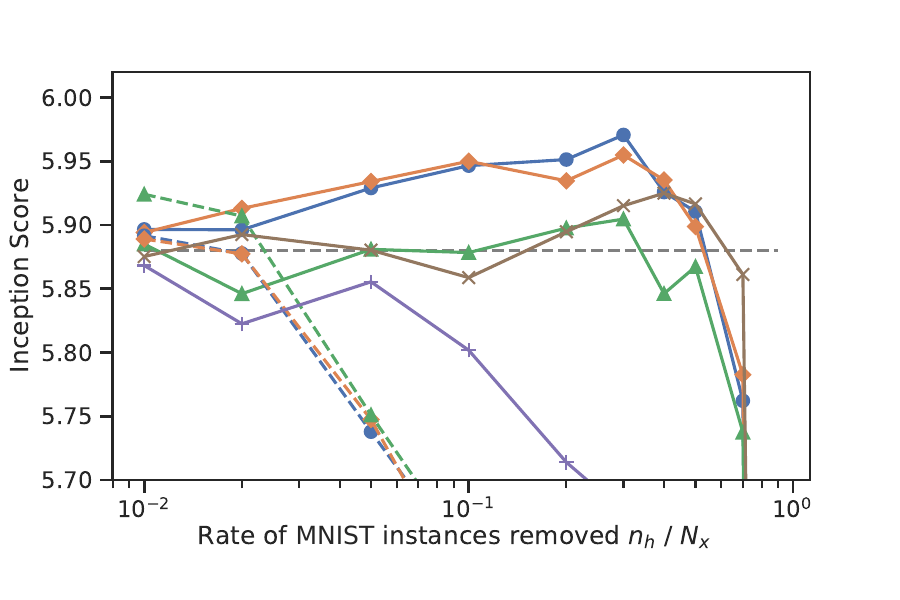}
            \label{sub:clean_is}
        }
        \end{minipage}
        \hspace{0.5cm}
        \begin{minipage}{0.45\linewidth}
        \centering
        \subfloat[][DCGAN (tested for IS \& one-epoch retraining)]{
            \includegraphics[trim={.0cm .8cm 1.5cm 1cm},clip,width=\linewidth]{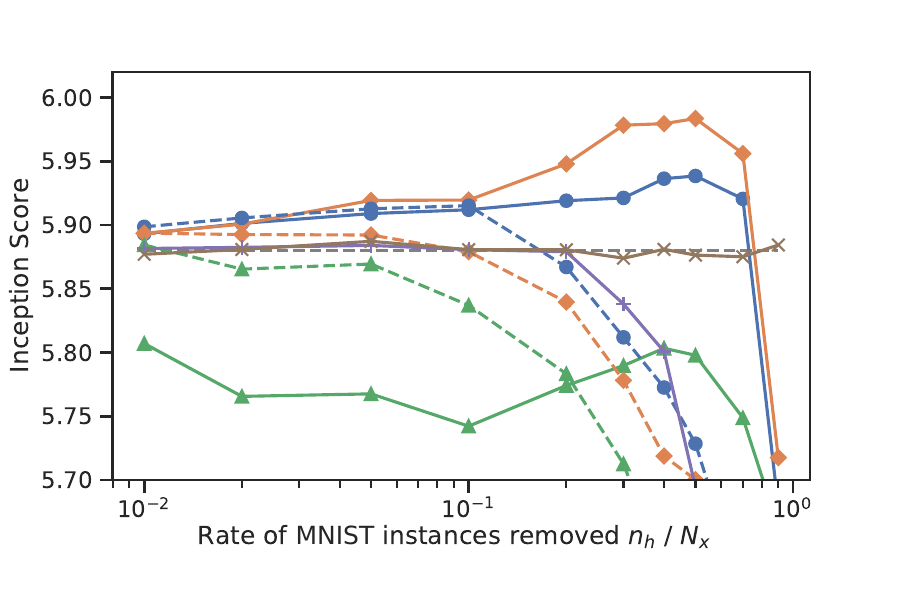}
            \label{sub:clean_is_1epoch}
        }
        \end{minipage}
        \newline
        \begin{minipage}{0.45\linewidth}
        \centering
        \subfloat[][DCGAN (tested for FID \& full-epoch retraining)]{
            \includegraphics[trim={.0cm .8cm 1.5cm 1cm},clip,width=\linewidth]{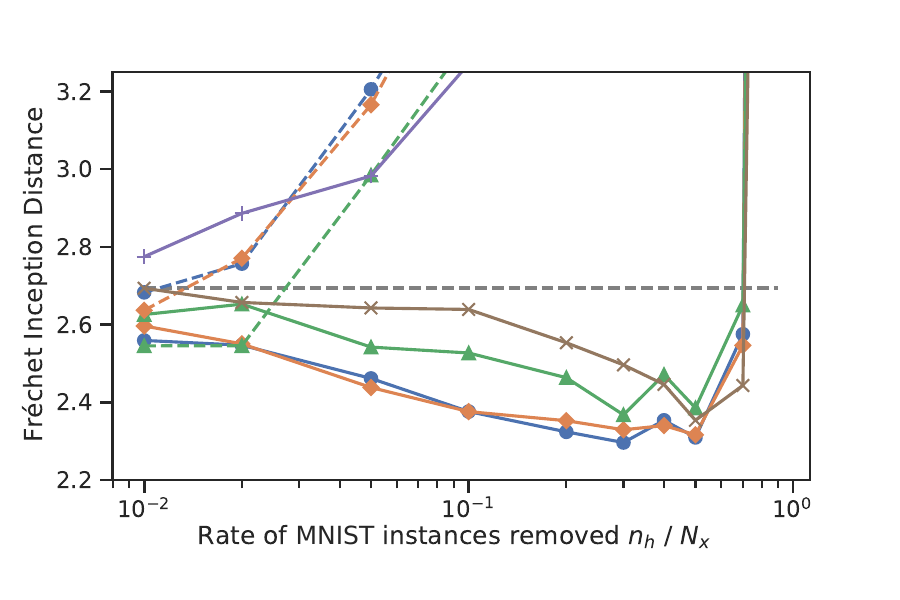}
            \label{sub:clean_fid}
        }
        \end{minipage}
        \hspace{0.52cm}
        \begin{minipage}{0.45\linewidth}
        \centering
        \subfloat[][DCGAN (tested for FID \& one-epoch retraining)]{
            \includegraphics[trim={.0cm .8cm 1.5cm 1cm},clip,width=\linewidth]{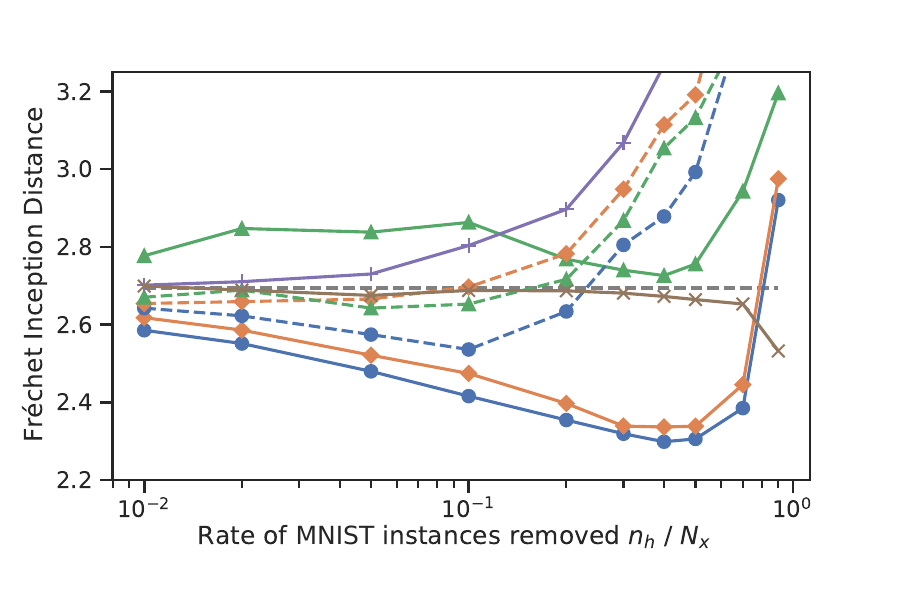}
            \label{sub:clean_fid_1epoch}
    }
        \end{minipage} 
           \hspace{0.2cm}
        \begin{minipage}{0.45\linewidth}
        \centering
        \subfloat[][StyleGAN (tested for FID \& full-epoch retraining)]{
            \includegraphics[trim={.0cm .8cm 1.5cm 1cm},clip,width=\linewidth]{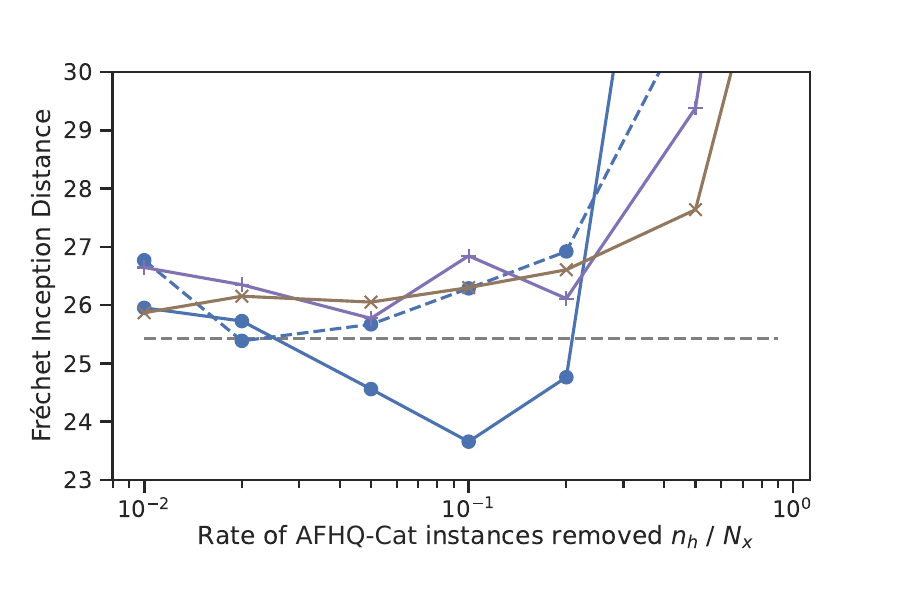}
            \label{sub:clean_stylegan}
        }
        \end{minipage}
        \hspace{0.52cm}
        \begin{minipage}{0.45\linewidth}
        \centering
        \subfloat[][StyleGAN (tested for FID \& one-epoch retraining)]{
            \includegraphics[trim={.0cm .8cm 1.5cm 1cm},clip,width=\linewidth]{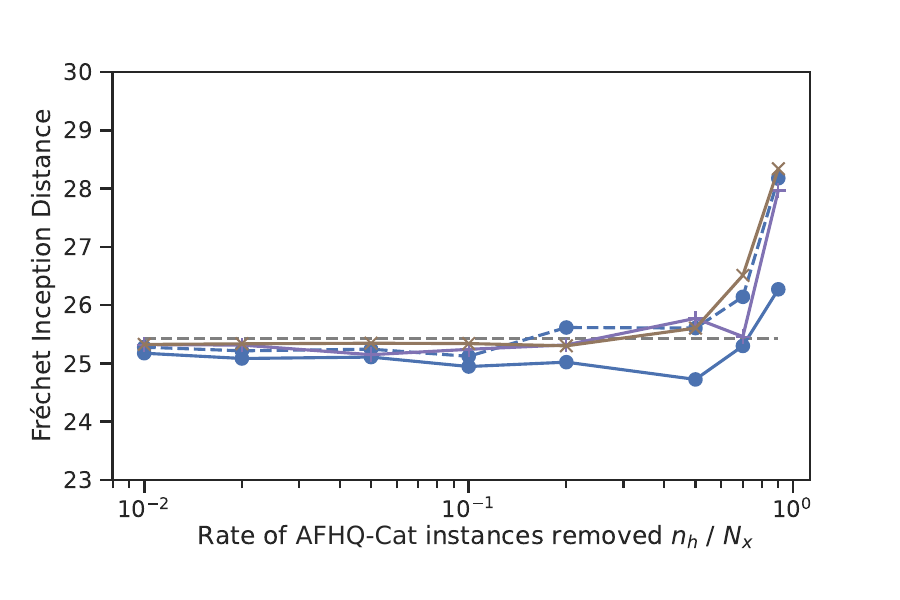}
            \label{sub:clean_stylegan_1epoch}
    }
    \end{minipage}
    \caption[]{
The average test GAN evaluation metrics after the data cleansing.
A higher value in \subref{sub:clean_ll}-\subref{sub:clean_is_1epoch} and a lower value in other plots \subref{sub:clean_fid}-\subref{sub:clean_stylegan_1epoch} indicate better generative performance, respectively.
We left out error bars and extreme values for clarity.
}
    \label{fig:clean}
\end{figure*}

\subsubsection{Setup}
\noindent We ran \cref{alg:lie_itd,alg:lie_aid} to estimate the influence on ALL  for LQGAN and the influence on IS for DCGAN to compare them with their true values.
We excluded the StyleGAN setting from this experiment due to its large computation in computing the true influence.

For both LQGAN and DCGAN, we performed the same procedure below unless otherwise noted.
\begin{enumerate}[label=\roman*)]
\item     \underline{Dataset preparation}: We used \(x\sim\mathcal{N}(1,1)\) to construct the 1D-Normal training dataset with 1,000 instances for AGD training and the validation dataset with 1,000 instances for computing ALL. For MNIST, we randomly selected 10,000 instances for AGD training and 10,000 validation instances for computing IS. 
\item \underline{Training}: LQGAN and DCGAN were trained through \(T\) steps of AGD. The MNIST classifier used for computing IS was also trained using the validation dataset.
\item \underline{Selection of removed instances}: We randomly selected 100 target instances from the training dataset. 
    \item \underline{Estimating \(\Delta E_{-j}\)}: To estimate the influence on GAN evaluation metric, we performed \cref{alg:lie_itd} and \cref{alg:lie_aid} for the target instances.
\item \underline{Computing true \(\Delta E_{-j}\)}: The true influence on GAN evaluation metric of each target instance was computed by running the counterfactual AGD.
    \item \underline{Evaluation}: Estimation accuracy was evaluated by Kendall's tau, which measures the ordinal correlation between estimated and true influence on GAN evaluation metrics, as adopted in the previous study~\cite{Hara2019}. This is because, for data cleansing purposes, the ranking of the harmfulness is considered more important than the estimation error of individual instances.
\end{enumerate}
We ran the procedure above ten times by changing the random seeds and excluded two cases where the AGD training of DCGAN did not converge.
We also varied the number of AGD steps \(T\) and the depth of Neuman series approximation \(M\) to see their effects on estimation errors.

\subsubsection{Estimation of Influence on ALL with LQGAN}
\noindent \cref{fig:valid} shows the average Kendal's Tau of the repeated experiments.

\cref{fig:valid}\subref{sub:toy_valid} illustrates that the ITD-EIGEM was able to provide accurate estimates for all \(T\).
This suggests the constant error remaining after infinite \(T\), mentioned in \cref{col:error_itd_convex}, has little effect on
the ranking of the instance's harmfulness.

AID-EIGEM provided a precise estimation when both \(T\) and \(M\) are sufficiently large, aligning with \cref{th:error_aid}.
Moreover, the results of AID-EIGEM with different \(M\) are consistent with our observation in \cref{sec:est_err_aid}; smaller \(M\) yielded a smaller error when \(T\) is small, while a large \(M\) achieved better result when \(T\) is large.

\subsubsection{Estimation of Influence on IS with DCGAN} 
\noindent The result in \cref{fig:valid}\subref{sub:mnist_valid_is} is noisier than the LQGAN setting since precise estimations are more challenging in this case, where \cref{ass:J_bound} is not guaranteed.
We thus evaluated the errors of our estimates by checking whether Kendal's tau is statistically significantly larger than that of random ranking with p-values \(<0.05\).

Despite the difficulties in problem setting, ITD-EIGEM demonstrated statistically significantly better than the random ranking~(\cref{fig:valid}\subref{sub:mnist_valid_is}), albeit with the exception of the result at \(T=1000\).
Likewise, AID-EIGEM demonstrated statistically significantly superior performance compared to the random approach at \(T=100, 200, 500\).
The results of AID-EIGEM with different \(M\) also suggest that large $M$ does not contribute to better estimation in this setting.
Although ITD-EIGEM outperforms AID-EIGEM in estimation accuracy, AID-EIGEM remains a promising approach because of its significantly smaller memory cost as remarked in \cref{table:comparison}.

\subsection{Experiment 2: Data Cleansing}\label{sec:exp2}
\noindent We finally investigate whether the identified harmful instances are helpful in \textit{data cleansing}.
We define data cleansing as a task to improve GAN evaluation metrics by identifying a set of harmful instances and retraining without using them.

\subsubsection{Setup}
\noindent We will begin by outlining the configuration of the datasets used in the experiments.
 involved in preparing the datasets utilized in our experiments.
For 1D-Normal used to train LQGAN, we prepared a mixture Gaussian distribution consisting of two 1-dimensional normal distributions to simulate the situation where the training dataset includes harmful instances.
We used \(x \sim b \mathcal{N}(1, 0.5) + (1-b) \mathcal{N}(-2, 0.5) \) with \(b \sim \mathrm{Bernoulli}(0.95)\), considering samples from  \(\mathcal{N}(-2,0.5)\) to be harmful instances.
When computing the influence on ALL, we used a validation dataset generated only from \(\mathcal{N}(1,0.5)\), simulating the situation where a developer can create a small validation dataset with no harmful instances by their inspection.
For MNIST and AFHQ-CAT, we simply split the original training dataset into our training and validation datasets, considering the original dataset already includes some harmful instances, as suggested in \cite{khanna2019interpreting,Hara2019}.

Next, we detail our approaches and baselines used to identify harmful instances and the criteria for determining their harmfulness.
We identified harmful instances in the 1D-Normal training dataset using estimated influence on ALL, IS, and FID for every applicable setup.
We regarded a training instance harmful when it had a negative (positive) influence on FID (ALL or IS).
We also selected instances using baseline approaches for both setups: anomaly detection method, influence on the discriminator's loss, and random values.
For anomaly detection, we adopted Isolation Forest \cite{isolation}.
Isolation Forest fitted its model using the raw training data points in the 1D-Normal setting and feature vectors of the training datasets drawn by the pre-trained classifier in the MNIST setting.
We also tested an instance selection using the influence on the discriminator's loss to verify our claim that the influence on the loss does not represent the harmfulness of the instances.
Influence on the discriminator's loss was calculated on \(V\left({{\boldsymbol{\varphi}}}^{(T)},{\boldsymbol{\psi}}^{(T)}\right)\) using validation instances and newly sampled latent variables.
We considered instances with a negative influence on the discriminator's loss to be harmful.

Our experiments consist of the following five steps:
\begin{enumerate}[label=\roman*)]
    \item \underline{Preparing datasets}: We sampled 1D-Normal instances to construct the training dataset with 1,000 instances for AGD training and the validation dataset with 1,000 instances for computing ALL. For MNIST, we randomly selected 50,000 instances for AGD training and 10,000 validation instances for computing IS and FID and training the classifier for IS and FID.
    AFHQ-CAT dataset was randomly split into 3,336 instances for AGD training and 1,111 instances for generating InceptionV3 features used in the instance selection.
    \item \underline{Scoring harmfulness}: After the AGD training, we scored the harmfulness of all training instances using our approaches and baselines.
    \item \underline{Selecting instances to be removed}: We selected the top \(n_h < N\) harmful instances according to the computed harmful scores, testing with various \(n_h\).
    \item \underline{Retraining}: We then retrained the model with the selected harmful instances excluded. When retraining, we tested two strategies: 
    \begin{itemize}
        \item \textit{Full-epoch retraining} runs complete \(T\) steps of counterfactual AGD from the initial parameter \(\boldsymbol{\theta}^{(0)}\).
        \item \textit{One-epoch retraining} runs counterfactual AGD starting from the trained parameter at the one-epoch behind the final step, similarly adopted by \cite{Hara2019}. When using this strategy, in step (ii), \cref{alg:lie_itd} runs iterations only for the last epoch.
    \end{itemize}
    \item \underline{Evaluation}: Finally, we evaluated the performance of retrained models by ALL for LQGAN, IS/FID for DCGAN, and FID for StyleGAN using the test dataset and newly sampled test latent variables. 
    The test datasets of 1D-Normal, MNIST, and AFHQ-CAT contain 1,000, 10,000, and 1,111 instances, respectively.
\end{enumerate}
\noindent We ran the experiments 20 times using different random seeds for LQGAN and DCGAN, excluding three trials in which DCGAN failed to converge.
As for StyleGAN, we only ran the experiment once because of its significant computational expenses involved in both influence estimation and training.

To thoroughly evaluate the data cleansing, we examine the efficacy of multiple instance selection approaches (\cref{sec:compare_methods}), the impact on different retraining strategies (\cref{sec:compare_retrain}), the enhancements of general generative performance (\cref{sec:coverage}), and visual analysis of harmful instances and generated samples (\cref{sec:quality}).
 
\subsubsection{Overall Performance}
\noindent \cref{fig:clean} shows the average test GAN evaluation metrics of the repeated experiments for each instance selection approach.

\cref{fig:clean}\subref{sub:clean_ll} indicates that the data cleansing by the influence on ALL by ITD-EIGEM and the Isolation Forest resulted in the best improvements across all methods.

For the MNIST with DCGAN setup, our selection approach with ITD-EIGEM showed the best FID and IS improvements, regardless of the choice of GAN evaluation metric used to judge harmfulness and the retraining strategy, i.e., full-epoch or one-epoch retraining~(\cref{fig:clean}\subref{sub:clean_is}-\ref{fig:clean}\subref{sub:clean_fid_1epoch}).

For the AFHQ-CAT with StyleGAN setup, our selection approach with ITD-EIGEM showed the best FID improvements both in the full- or one-epoch retraining~(\cref{fig:clean}\subref{sub:clean_stylegan}-\ref{fig:clean}\subref{sub:clean_stylegan_1epoch}).

\subsubsection{Comparison of Instance Selection Approaches} \label{sec:compare_methods}
\noindent Overall, ITD-EIGEM outperformed AID-EIGEM, especially in the DCGAN and StyleGAN settings.
This is likely because AID-EIGEM relies on a \cref{ass:J_bound} that is not applicable in deep learning settings.
Nevertheless, AID-EIGEM remains a valuable option because of its memory efficiency and improvements of the test IS and FID in the one-epoch retraining settings~(\cref{fig:clean}\subref{sub:clean_is_1epoch} and \ref{fig:clean}\subref{sub:clean_fid_1epoch}).

Regarding the baselines, Isolation Forest was effective for data cleansing in the simple setting with LQGAN, yet this worsened the performance in the other cases (\cref{fig:clean}\subref{sub:clean_is}-\subref{sub:clean_stylegan_1epoch})
Data cleansing based on the influence on the discriminator's loss failed to improve the GAN evaluation metrics, although small improvements were observed in \ref{fig:clean}\subref{sub:clean_fid}.
This result supports our hypothesis that the discriminator's loss is not a reliable metric of generative performance, and thus, the influence on the discriminator's loss cannot accurately measure the harmfulness of instances.
Randomly removing instances unexpectedly enhanced the test FID and IS in the full-epoch setting~(\cref{fig:clean}\subref{sub:clean_is} and \ref{fig:clean}\subref{sub:clean_fid}).
We hypothesize that this is because the random removal, which scales down the gradient, worked similarly to the learning rate tuning.
However, it should be noted that the improvements in our approaches do not solely stem from this ``pseudo'' learning rate tuning; the t-test with p-values \(<0.05\) confirmed that the improvements achieved with ITD-EIGEM were statistically significantly better than those attained through random selection. 

\subsubsection{Full-epoch v.s. One-epoch Retraining} \label{sec:compare_retrain}
\noindent Surprisingly, the ITD-EIGEM data cleansing with one-epoch retraining demonstrated a competitive performance compared to the full-epoch retraining\footnote{
The small improvements observed in the one-epoch retraining of StyleGAN (\cref{fig:clean}\subref{sub:clean_stylegan_1epoch}) appear to stem from the nature of the moving averaged generator; the data cleansing performed to a single epoch only partially changes the final averaged generator.
}.
This suggests that considering the influence during the last epoch is informative enough for data cleansing.
The superior performance of the one-epoch retraining demonstrated its practical effectiveness, namely, the significantly smaller computational cost of ITD-EIGEM and retraining compared to the full-epoch retraining.

\subsubsection{Can Enhancing One Metric Lead to Overall Improvements?} \label{sec:coverage}
\noindent Because all the current GAN evaluation metrics have their own weaknesses~\cite{proscons}, our data cleansing may ``overfit'' that metric, sacrificing some aspects of generative performance. 
We checked if such an overfit occurs by running the data cleansing using the influence estimation on a given metric and by evaluating the cleansed model using different metrics.

In the MNIST case, \cref{fig:clean}\subref{sub:clean_is}-\ref{fig:clean}\subref{sub:clean_fid_1epoch} indicates that data cleansing based on the influence on a specific GAN evaluation metric improves another metric that is not used for the selection; removing harmful instances based on the influence on IS improved test FID (and vice versa). 

For AFHQ-Cat, we evaluated the cleansed model using density and coverage metrics \cite{naeem2020reliable}, which correspond to the quality and diversity of the generated images, respectively.
\cref{table:density_coverage} presents the test density and coverage for the cleansed models obtained from each method, with the removal rate chosen based on their best validation FID. 

It is generally expected that instance removal, which inherently reduces dataset diversity, would enhance the quality of generation since the model can focus on a limited set of instance patterns. Such a quality improvement is actually confirmed by the improved density of all approaches (\cref{table:density_coverage}). 
However, our results also demonstrate a counterintuitive finding: the ITD-EIGEM-based data cleansing significantly improved the coverage of the generated samples without compromising density (\cref{table:density_coverage}). 
In \cref{sec:quality}, we will investigate the underlying mechanism of this phenomenon.

\begin{table}[t]
\caption{Test Density and Coverage of StyleGAN after the Data Cleansing\label{table:density_coverage}}
\centering
\begin{tabular}{lcc}
\toprule
 & Density & Coverage \\
\midrule
No Removal & 0.738 (+0.000) & 0.696 (+0.000) \\
\textbf{Influence on FID by ITD} & 0.778 (+0.040) & \textbf{0.717 (+0.021)}\\
\textbf{Influence on FID by AID} & 0.822 (+0.084) & 0.702 (+0.006) \\
Isolation Forest & 0.790 (+0.052) & 0.693 (-0.003) \\
Random & 0.772 (+0.034) & 0.691 (-0.005) \\
\bottomrule
\end{tabular}
\end{table}

\subsubsection{Qualitative Study of Harmful Instances and Generated Samples} \label{sec:quality}

\begin{figure*}[t]
\begin{minipage}{0.48\linewidth}
\centering
\begin{minipage}{\linewidth}
\centering
\subfloat[Harmful training instances]{
\includegraphics[width=0.78\linewidth,trim=-0.5cm 0.5cm 0.4cm 0.4cm, clip]{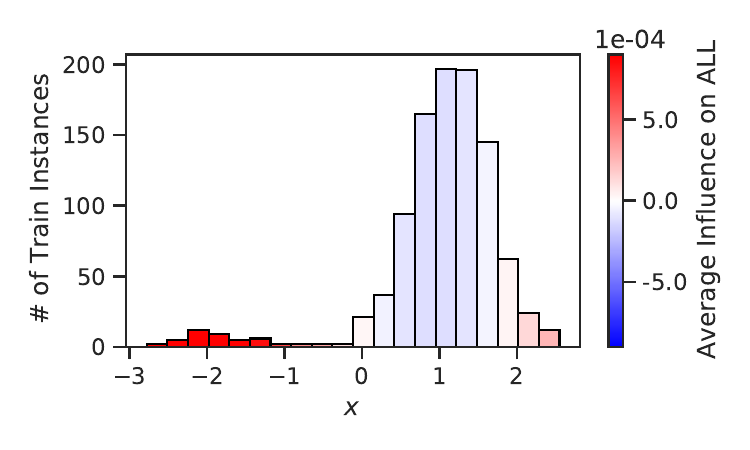}
    \label{sub:visual_gauss_harmful}}
\end{minipage} 
\begin{minipage}{\linewidth}
    \centering
    \subfloat[Change in generator distribution by data cleansing]{
    \includegraphics[width=\linewidth,trim=-1.5cm 0.45cm -2.5cm 0.1cm, clip]{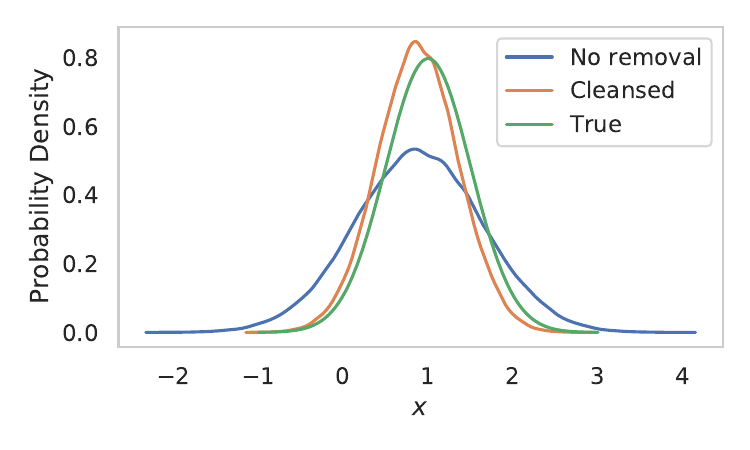}
    \label{sub:visual_gauss_cleansed}}
\end{minipage}  
\caption[]{
\noindent Influence on ALL representing harmfulness of 1D-Normal training instances~\subref{sub:visual_gauss_harmful} and generator's distributions before and after the data cleansing~\subref{sub:visual_gauss_cleansed}.
\subref{sub:visual_gauss_harmful} presents the histogram of the training instances, with each segment colored according to the average influence on ALL calculated over the instances within the belonging bin.
\subref{sub:visual_gauss_cleansed} shows the kernel density estimates of the true distribution (``True'') and generator's distributions before (``No removal'') and after (``Cleansed'') the data cleansing.
}
\label{fig:visual_2d_main}
\end{minipage}
\hfill
\begin{minipage}{0.48\linewidth}
\centering
\begin{minipage}[c]{0.98\linewidth}
\centering
\subfloat[Total influence on FID of harmful instances]{
\includegraphics[width=\linewidth]{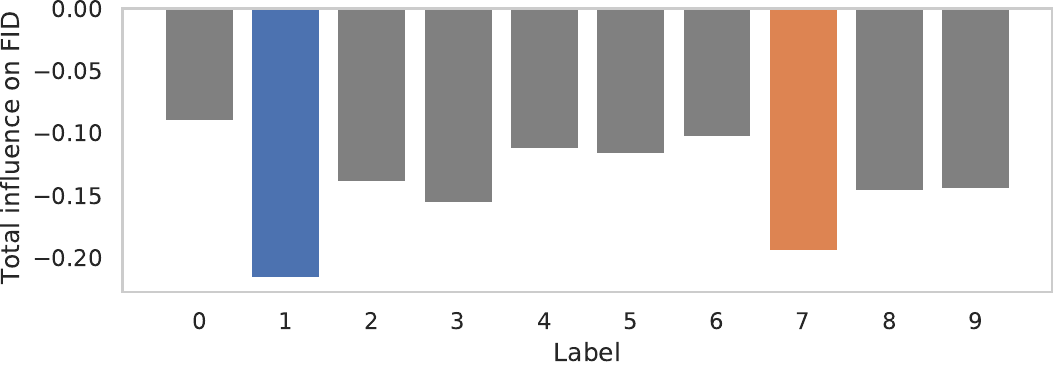}
\label{sub:harmful_label}}
\end{minipage} 
\begin{minipage}[c]{0.48\linewidth}
\centering
\subfloat[No removal]{
\includegraphics[width=\linewidth]{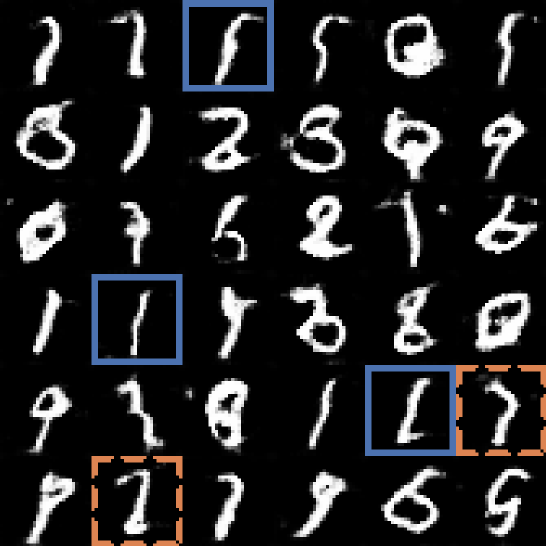}
\label{sub:visual_mnist_no_removal}}
\end{minipage} 
\begin{minipage}[c]{0.48\linewidth}
\centering
\subfloat[Cleansed]{
\includegraphics[width=\linewidth]{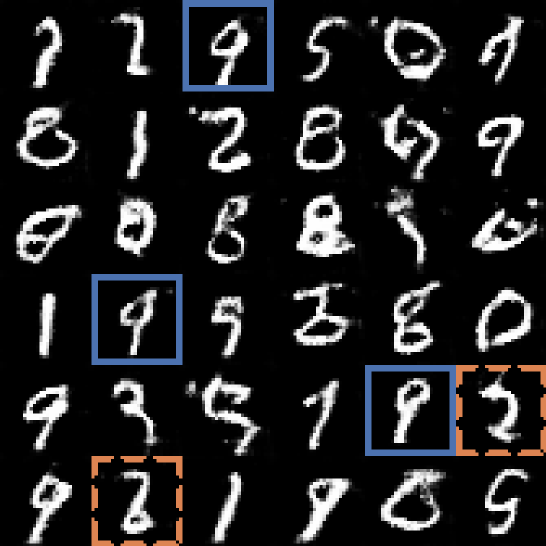}
\label{sub:visual_mnist_cleansed}}
\end{minipage} 

\caption[]{
Label-wise total influence on FID of estimated harmful instances~\subref{sub:harmful_label} and generated instances before~\subref{sub:visual_mnist_no_removal} and after~\subref{sub:visual_mnist_cleansed} the data cleansing.
Both \subref{sub:visual_mnist_no_removal} and \subref{sub:visual_mnist_cleansed} use the same series of test latent variables.
As seen in \subref{sub:harmful_label}, instances labeled as digits 1 and 7 were suggested to be the most harmful.
\subref{sub:visual_mnist_no_removal} and \subref{sub:visual_mnist_cleansed} indicate that their exclusion increased the diversity of generated instances by assigning latent variables that had been associated with the digits 1 (blue solid line) and 7 (orange dotted line) to other digits.
}
\label{fig:visual_mnist_main}
\end{minipage}
\end{figure*}
\begin{figure*}[t!]
\begin{minipage}{\linewidth}
\centering
\begin{minipage}{\linewidth}
\centering
\subfloat[Harmful instances]{
\includegraphics[width=0.98\linewidth,clip]{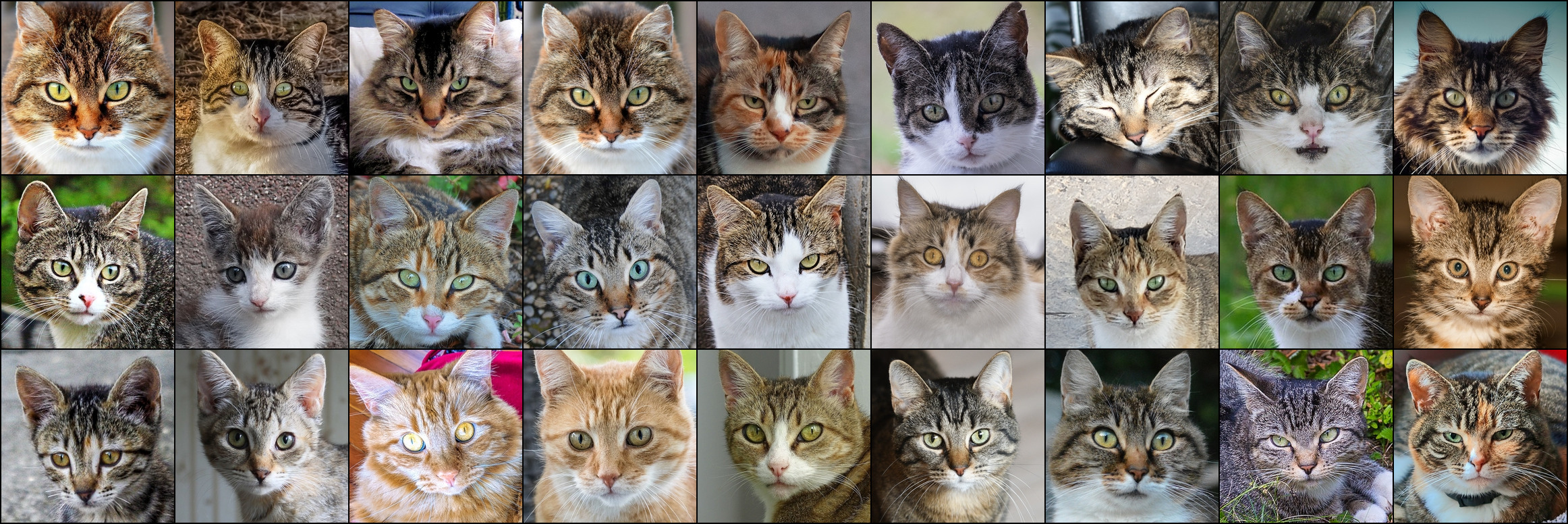}
    \label{sub:harmful_stylegan}}
\end{minipage} 

\vspace{4pt}

\begin{minipage}{\linewidth}
\centering
\subfloat[Helpful instances]{
\includegraphics[width=0.98\linewidth,clip]{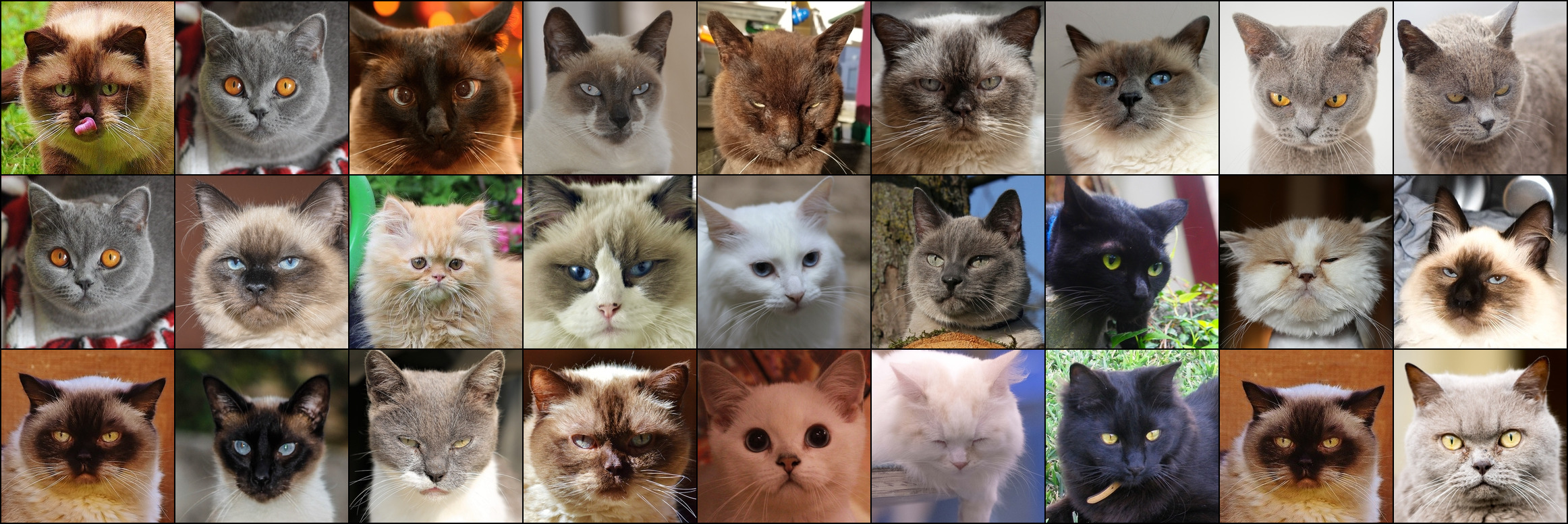}
    \label{sub:helpful_stylegan}}
\end{minipage} 

\vspace{4pt}

\begin{minipage}{\linewidth}
\centering
\subfloat[Randomly selected instances]{
\includegraphics[width=0.98\linewidth,clip]{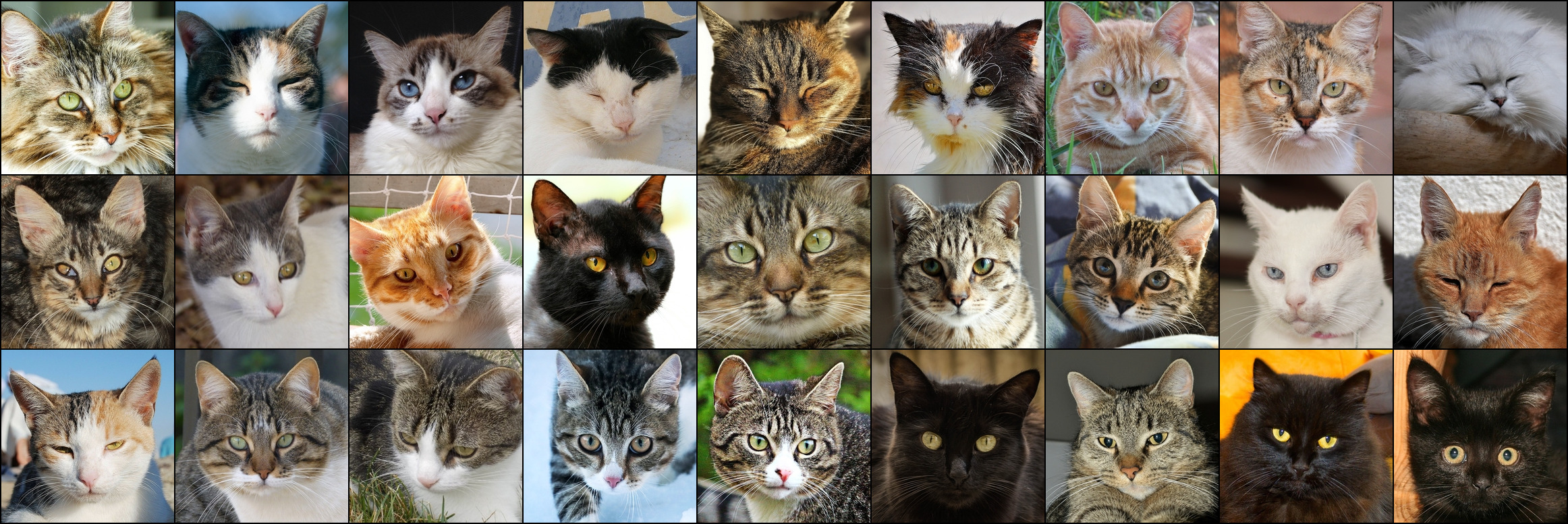}
    \label{sub:random_stylegan}}
\end{minipage} 
\caption[]{
\noindent Top-27 harmful~\subref{sub:harmful_stylegan} and helpful training~\subref{sub:helpful_stylegan} instances suggested by our ITD-EIGEM performed over entire training steps, and randomly selected instances from the dataset~\subref{sub:random_stylegan}.
}
\label{fig:influencial_cats}
\end{minipage}
\end{figure*}
\def\widthcat{0.195}
\begin{figure*}[t!]
\begin{minipage}{\linewidth}
\centering
\begin{minipage}{\widthcat\linewidth}
\centering
\subfloat[No removal]{
\includegraphics[width=\linewidth,clip]{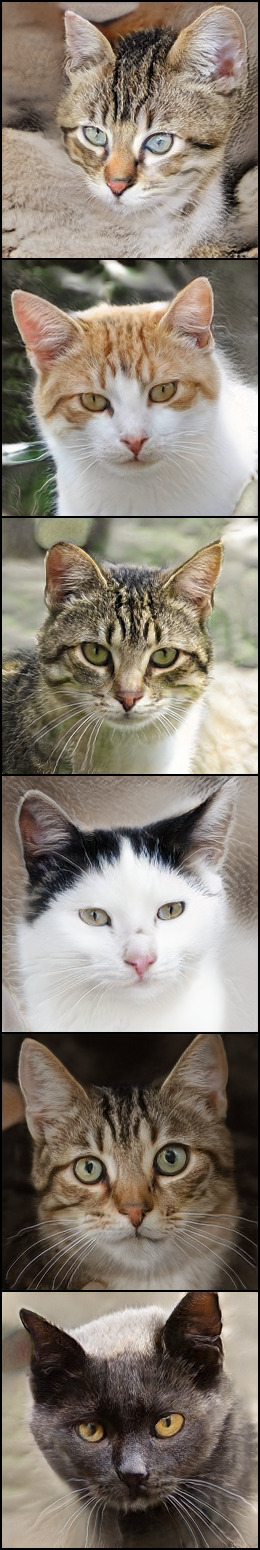}
    \label{sub:visual_cat_no_removal}}
\end{minipage} 
\hfill
\begin{minipage}{\widthcat\linewidth}
\centering
\subfloat[Infl. on FID by ITD]{
\includegraphics[width=\linewidth,clip]{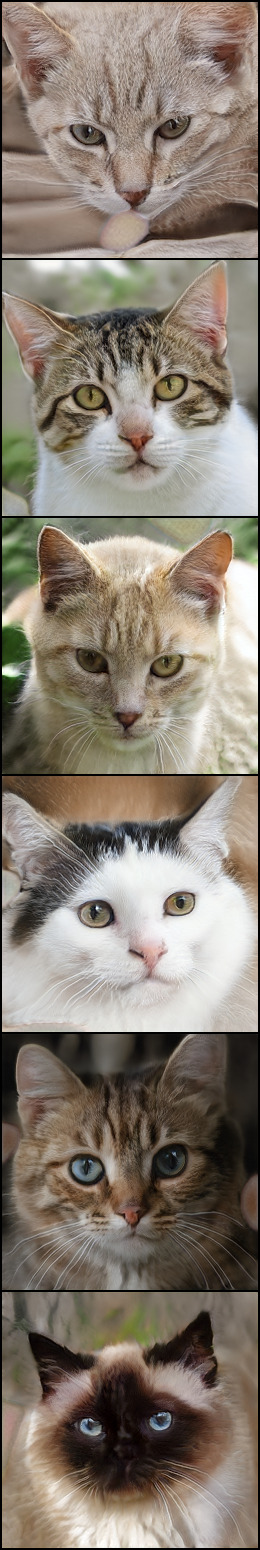}
    \label{sub:visual_cat_fid_itd}}
\end{minipage} 
\hfill
\begin{minipage}{\widthcat\linewidth}
\centering
\subfloat[Infl. on FID by AID]{
\includegraphics[width=\linewidth,clip]{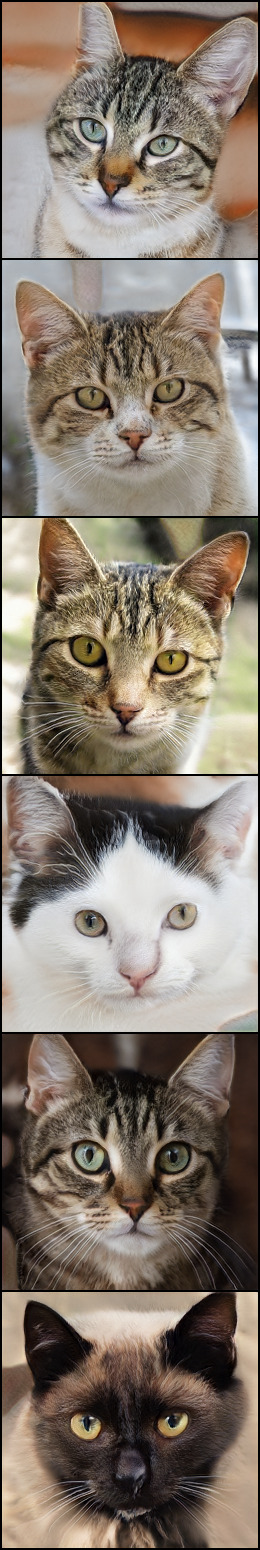}
    \label{sub:visual_cat_fid_aid}}
\end{minipage} 
\hfill
\begin{minipage}{\widthcat\linewidth}
\centering
\subfloat[Isolation Forest]{
\includegraphics[width=\linewidth,clip]{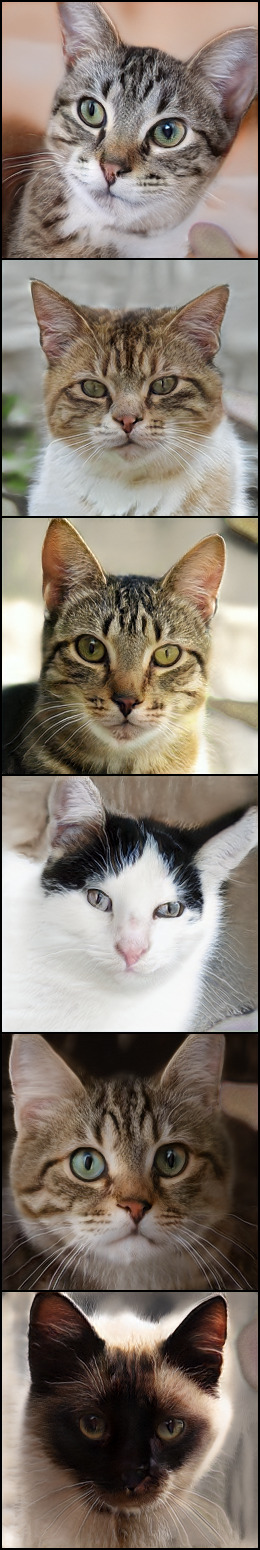}
    \label{sub:visual_cat_if}}
\end{minipage} 
\hfill
\begin{minipage}{\widthcat\linewidth}
\centering
\subfloat[Random]{
\includegraphics[width=\linewidth,clip]{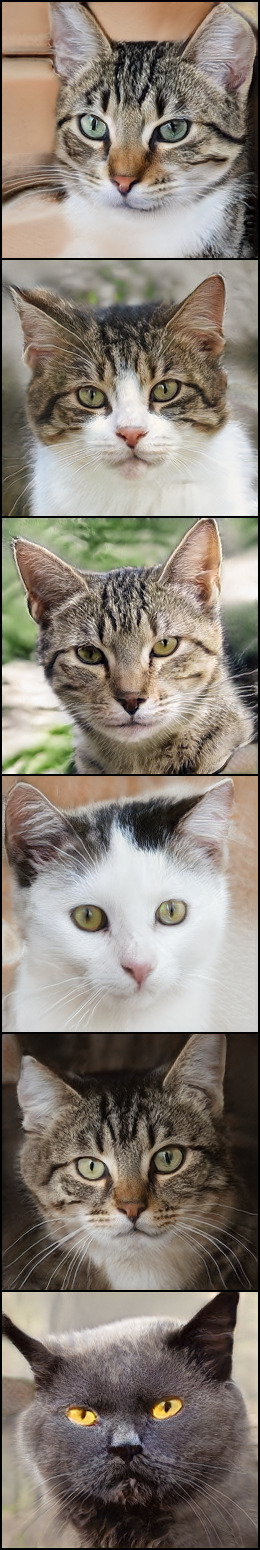}
    \label{sub:visual_cat_random}}
\end{minipage} 
\caption[]{Generated images before and after the data cleansing. For every method, we chose the model that yielded the best validation FID. All the images in the same row use the same test latent variable.}
\label{fig:visual_cat_main}
\end{minipage}
\end{figure*}

Our data cleansing presented visual improvements in generated samples.
Furthermore, we found that suggested harmful instances tend to belong to the oversampling region of the generator's distribution.

\cref{fig:visual_2d_main}\subref{sub:visual_gauss_harmful} shows that the 1D-Normal training instances drawn from \(\mathcal{N}(-2,0.5)\), which is added to simulate harmful instances, were correctly identified as harmful.
As a result, the distribution of the cleansed generator~(``Cleansed' in \cref{fig:visual_2d_main}\subref{sub:visual_gauss_cleansed}) moved closer to the desired \(\mathcal{N}(1,0.5)\)~(``True'' in \cref{fig:visual_2d_main}\subref{sub:visual_gauss_cleansed}). 

In the case of MNIST, \cref{fig:visual_mnist_main}\subref{sub:harmful_label} indicates that a large part of the harmful instances were labeled as digits 1 and 7, likely because the original generator produced digits 1 and 7 too frequently than the other digits~(\cref{fig:visual_mnist_main}\subref{sub:visual_mnist_no_removal}).
By removing them, the samples generated from the same latent variables changed from the images of digits 1 and 7 to those of other digits~ (highlighted samples in \cref{fig:visual_mnist_main}\subref{sub:visual_mnist_no_removal} and \ref{fig:visual_mnist_main}\subref{sub:visual_mnist_cleansed}).
This implies that a certain amount of density that had been over-allocated to the digit 1 shifted to the regions of other digits.
We suppose this effect improved the diversity in the generated samples, resulting in better FID and IS.

In the case of AFHQ-CAT, we observed similar results. 
\cref{fig:influencial_cats} categorizes training instances into three groups: harmful instances \subref{sub:harmful_stylegan}, predicted to negatively influence FID via ITD-EIGEM; helpful instances \subref{sub:helpful_stylegan}, predicted to have positive influences; and randomly sampled instances \subref{sub:random_stylegan}.
\cref{fig:visual_cat_main} shows images generated by the original and cleansed models, using the same latent variables in each column.

From \cref{fig:influencial_cats}, harmful and helpful instances are distinguished by common and rare patterns, respectively.
A significant portion of the harmful instances (\cref{fig:influencial_cats}\subref{sub:harmful_stylegan}) were yellow cats with stripes, indicating that they share common patterns frequently appearing in the dataset (\cref{fig:influencial_cats}\subref{sub:random_stylegan}). 
Conversely, many of the helpful instances (\cref{fig:influencial_cats}\subref{sub:helpful_stylegan}) consisted of cats without stripes and seal point cats\footnote{Cats having a light-colored body with dark brown points on their ears and face}, featuring rare patterns in the dataset (\cref{fig:influencial_cats}\subref{sub:random_stylegan}). 

Removing those harmful instances seemed to lead the model to generate samples with rare patterns, as evident in \cref{fig:visual_cat_main}.
The cats in the first and third rows indicate that latent variables initially linked to a common pattern (i.e., yellow cats with stripes) were re-assigned to a rare pattern (i.e., cats without stripes) in the cleansed model (\cref{fig:visual_cat_main}\subref{sub:visual_cat_fid_itd}).
Similarly, a grey cat in the bottom row of the original model became a seal point cat after our data cleansing ((\cref{fig:visual_cat_main}\subref{sub:visual_cat_fid_itd}).
These re-assignments of latent variables to rare patterns were not observed or only partially observed in other approaches (\cref{fig:visual_cat_main}\subref{sub:visual_cat_fid_aid}-\subref{sub:visual_cat_random}).

\section{Current Limitations and Future Directions}
Our method does not guarantee that instances identified as harmful for one GAN evaluation metric are equally harmful for other metrics.
This limitation stems from the nature of current GAN evaluation metrics, namely, they can only evaluate limited aspects of generative performance\cite{proscons}. 
For example, FID considers sample diversity but only partially addresses visual quality, focusing more on textures than shapes of objects \cite{karras2020analyzing}. 
\cref{sec:exp2} demonstrated that removing instances harmful to FID improved other metrics such as IS and coverage.
However, these metrics may share the same focus, and thus improving them can compromise some aspects of the generative performance, such as the accuracy of object shapes.
Moreover, since excluding instances can reduce diversity in datasets, it potentially compromises the diversity of generated samples in certain settings.

Note that this limitation does not stem from our influence estimation method itself.
Our method can be combined with any evaluation metric that is differentiable. 
This flexibility allows our approach to be integrated with future evaluation metrics, which possibly provide a more comprehensive evaluation capability.
Future work will focus on incorporating advanced GAN evaluation metrics to better understand the relationship between training instances and generative performance.

\section{Conclusion}
\noindent We proposed influence estimators for GANs that use the Jacobian of the gradient of the discriminator's loss with respect to the generator’s parameters (and vice versa), which traces how the absence of an instance in the discriminator’s training affects the generator’s parameters.
We also proposed an evaluation scheme to judge if an instance is harmful or not based on the influence on GAN evaluation metrics rather than that on the loss value.
The proposed estimators and differentiability of GAN evaluation metrics allow efficient estimation of the influence on GAN evaluation.
Empirical results showcased that the estimated influence on GAN evaluation metric well agreed with the true value.
We finally demonstrated removing identified harmful instances effectively improved the generative performance with respect to various GAN evaluation metrics.

\bibliography{main}
\bibliographystyle{IEEEtran}

\begin{bibunit}
\newpage

\onecolumn
\appendices
\section{Theoretical Results}  \label{app:proof}
\noindent Recalling the following assumptions, we will prove our \cref{th:error_itd,th:error_aid}.
\Jbound*
\invertible*
\smoothZ*
\Neighbor*

\subsection{Proof of \texorpdfstring{\cref{th:error_itd}}{Theorem 1}}
\noindent We begin providing useful lemmas with proofs.
\Zbound*
\begin{proof}
 Let \(\sigma_{\mathrm{max}}\left(\boldsymbol{A}\right)\) denote the largest singular value of a matrix \(\boldsymbol{A}\).
 From \cref{ass:J_bound}, for every \({{\boldsymbol{\theta}}} \in \mathcal{B}({\boldsymbol{\theta}}^{*})\),
 \begin{align*}
           \left\Vert \boldsymbol{Z}\left(\boldsymbol{\theta}\right)\right\Vert^2
             &= \sigma_{\mathrm{max}} \left(\left(\boldsymbol{I}- \eta  \boldsymbol{J}\left(\boldsymbol{\theta} \right)\right)^{\intercal}\left(\boldsymbol{I}- \eta  \boldsymbol{J}\left(\boldsymbol{\theta} \right)\right) \right)\\
             &= \sigma_{\mathrm{max}} \left(\boldsymbol{I}- \eta  \left(\boldsymbol{J}\left(\boldsymbol{\theta} \right) + \boldsymbol{J}\left(\boldsymbol{\theta}\right)^{\intercal}\right) + \eta^2 \boldsymbol{J}\left(\boldsymbol{\theta}\right)^{\intercal}\boldsymbol{J}\left(\boldsymbol{\theta}\right)\right)  \\
              &\leq   1 - 2 \eta \mu + \eta^2 \lambda^2 .
 \end{align*}
Since \(\eta > 0\), the sufficient condition of \(\eta\) that ensures \(\sigma_{\mathcal{B}} < 1 \) is \(\eta < \frac{2\mu}{\lambda^2}\).
\end{proof}
\Converge*
\begin{proof}
We begin with showing the uniqueness of the Nash equilibrium within \(\mathcal{B}({\boldsymbol{\theta}}^{*})\).
Nash equilibrium \({\boldsymbol{\theta}}^{*}=\left({{\boldsymbol{\varphi}}^{*}}{}^\intercal ~ {\boldsymbol{\psi}}^{*}{}^\intercal\right){}^\intercal\) needs to satisfy
\begin{align*}
V(\boldsymbol{\varphi } ,{\boldsymbol{\psi}}^{*}) &\geq V({\boldsymbol{\varphi}}^{*},{\boldsymbol{\psi}}^{*}) ,\ \forall \boldsymbol{\varphi }  \text{ s.t. } \left({{\boldsymbol{\varphi}}}{}^\intercal ~ {\boldsymbol{\psi}}^{*}{}^\intercal\right){}^\intercal \in  \mathcal{B}({\boldsymbol{\theta}}^{*}),\\
V({\boldsymbol{\varphi}}^{*} ,\boldsymbol{\psi }) &\leq V({\boldsymbol{\varphi}}^{*},{\boldsymbol{\psi}}^{*}), \ \forall \boldsymbol{\psi }  \text{ s.t. }   \left({{\boldsymbol{\varphi}}^{*}}{}^\intercal ~ {\boldsymbol{\psi}}{}^\intercal\right){}^\intercal \in  \mathcal{B}({\boldsymbol{\theta}}^{*}),
\end{align*}
or, equivalently
\begin{align}
\nabla _{\boldsymbol{\varphi }}^{2} V({\boldsymbol{\varphi }} ,{\boldsymbol{\psi}}^{*}) &\succeq 0, \ \forall  \boldsymbol{\varphi } \text{ s.t. } \left({{\boldsymbol{\varphi}}}{}^\intercal ~ {\boldsymbol{\psi}}^{*}{}^\intercal\right){}^\intercal \in  \mathcal{B}({\boldsymbol{\theta}}^{*}), \label{eq:definite_varphi}\\
-\nabla _{\boldsymbol{\psi }}^{2} V({\boldsymbol{\varphi}}^{*},{\boldsymbol{\psi }}) &\succeq 0, \ \forall \boldsymbol{\psi }   \text{ s.t. }   \left({{\boldsymbol{\varphi}}^{*}}{}^\intercal ~ {\boldsymbol{\psi}}{}^\intercal\right){}^\intercal \in  \mathcal{B}({\boldsymbol{\theta}}^{*}), \label{eq:definite_psi} \\
\nabla _{\boldsymbol{\varphi }} V({\boldsymbol{\varphi}}^{*},{\boldsymbol{\psi}}^{*}) &=\boldsymbol{0}, \label{eq:zero_grad_varphi}\\
\nabla _{\boldsymbol{\psi }} V({\boldsymbol{\varphi}}^{*},{\boldsymbol{\psi}}^{*}) &=\boldsymbol{0}. \label{eq:zero_grad_psi}
\end{align}
Recall \cref{ass:J_bound} which ensures
\begin{equation*}
    \frac{1}{2}\left(\boldsymbol{J} (\boldsymbol{\theta } )+\boldsymbol{J} (\boldsymbol{\theta } )^{\intercal }\right) =\begin{pmatrix}
\nabla _{\boldsymbol{\varphi }}^{2} V(\boldsymbol{\varphi } ,\boldsymbol{\psi }) & 0\\
0 & -\nabla _{\boldsymbol{\psi }}^{2} V(\boldsymbol{\varphi } ,\boldsymbol{\psi }) 
\end{pmatrix} \succ 0, ~\forall \left({{\boldsymbol{\varphi}}}{}^\intercal ~ {\boldsymbol{\psi}}{}^\intercal\right){}^\intercal \in  \mathcal{B}({\boldsymbol{\theta}}^{*}).
\end{equation*}
Given the strong concavity of \( V(\boldsymbol{\varphi } ,\boldsymbol{\psi }) \) with respect to \(\boldsymbol{\psi}\), the aforementioned relation implies that
\begin{equation*}
     \begin{array}{l}
\nabla _{\boldsymbol{\varphi }}^{2} V(\boldsymbol{\varphi } ,\boldsymbol{\psi }) \succ 0~\text{and} ~\nabla _{\boldsymbol{\psi }}^{2} V(\boldsymbol{\varphi } ,\boldsymbol{\psi }) \prec 0,~\forall \left({{\boldsymbol{\varphi}}}{}^\intercal ~ {\boldsymbol{\psi}}{}^\intercal\right){}^\intercal \in  \mathcal{B}({\boldsymbol{\theta}}^{*}).
\end{array}
\end{equation*}
Therefore, there exists the unique point in \(\mathcal{B}({\boldsymbol{\theta}}^{*})\) that satisfies  \cref{eq:definite_varphi,eq:definite_psi,eq:zero_grad_varphi,eq:zero_grad_psi}; $\min_{\boldsymbol{\varphi }} V(\boldsymbol{\varphi }, \boldsymbol{\psi })$ has a unique solution where \(\nabla _{\boldsymbol{\varphi }} V({\boldsymbol{\varphi}},{\boldsymbol{\psi}}) =\boldsymbol{0}\) for any $ \boldsymbol{\psi }$ lying with the neighborhood, and $\max_{\boldsymbol{\psi }} V(\boldsymbol{\varphi },\boldsymbol{\psi })$ also has a unique solution where \(\nabla _{\boldsymbol{\psi }} V({\boldsymbol{\varphi}},{\boldsymbol{\psi}}) =\boldsymbol{0}\) for any $\boldsymbol{\varphi }$ lying with the neighborhood.
Therefore, \(\boldsymbol{\theta}^{*}\in \mathcal{B}({\boldsymbol{\theta}}^{*})\) is the unique Nash equilibrium in \( \mathcal{B}({\boldsymbol{\theta}}^{*})\).

We then show that \({{\boldsymbol{\theta}}}^{(T)}\) converges to \({\boldsymbol{\theta}}^{*}\) when \(T\to\infty\).
Consider a mapping \(U(\boldsymbol{\theta }) \coloneqq \boldsymbol{\theta } -\eta \boldsymbol{v}(\boldsymbol{\theta })\) with \( \eta < \frac{2\mu}{\lambda^2}\) that is defined on \(\boldsymbol{\theta } \in \mathcal{B}({\boldsymbol{\theta}}^{*})\).
Since \(\partial _{\boldsymbol{\theta }} U(\boldsymbol{\theta }) =\boldsymbol{Z}(\boldsymbol{\theta })\), we have \(\max_{\boldsymbol{\theta }\in \mathcal{B}({\boldsymbol{\theta}}^{*})}\left\Vert \partial _{\boldsymbol{\theta }} U(\boldsymbol{\theta }) \right\Vert<1\), and thus \(U(\boldsymbol{\theta }) \) is a contraction mapping.
It is also trivially true that the Nash equilibrium in  \(\mathcal{B}({\boldsymbol{\theta}}^{*})\) is the fixed point of \(U(\boldsymbol{\theta }) \). 
Therefore, by recalling \(\boldsymbol{\theta}^{(t+1)}=U\left(\boldsymbol{\theta }^{(t)}\right) \) and \cref{ass:neighbor}, \({{\boldsymbol{\theta}}}^{(T)}\) converges to the unique stationary point \({\boldsymbol{\theta}}^{*}\) as \(T\to\infty\) for any \({{\boldsymbol{\theta}}}^{(0)} \in \mathcal{B}({\boldsymbol{\theta}}^{*})\).

\end{proof}
\begin{lemma} \label{lem:error_itd_true}
When \({\sigma} \coloneqq  \max_{{\boldsymbol{\theta}}\in \mathbb{R}^{d_{\boldsymbol{\theta}}} }\left\Vert{\boldsymbol{Z}}\left({{\boldsymbol{\theta}}}\right) \right\Vert \neq 1\), then for every \(t\geq 0\),
\begin{equation*}
     \left\Vert {{\boldsymbol{\theta}}}^{(t)}_{-j} - {{\boldsymbol{\theta}}}^{(t)} \right\Vert \leq   \frac{L_F\left(1-{\sigma}^{t}\right)}{1- {\sigma}},
\end{equation*}
where \(L_F \coloneqq  \max_{{\boldsymbol{\theta}}} \left\Vert \Delta\boldsymbol{v}_{-j}\left({{\boldsymbol{\theta}}}\right) \right\Vert\).
\end{lemma}
\begin{proof}
From Lagrange's mean value theorem, there exists \(r(s) \in \left[0,1\right]\) for every \(s\geq 0\) such that for \({{\boldsymbol{\theta}}}^{*(s)}_{-j} \coloneqq  r(s) {{\boldsymbol{\theta}}}^{(s)}_{-j} + (1-r(s)){{\boldsymbol{\theta}}}^{(s)}\),
\begin{equation*}
 {{\boldsymbol{\theta}}}^{(s+1)}_{-j}- {{\boldsymbol{\theta}}}^{(s+1)}  =  \Delta\boldsymbol{v}_{-j}\left({{\boldsymbol{\theta}}}^{*(s)}_{-j}\right) + {\boldsymbol{Z}}\left( {{\boldsymbol{\theta}}}^{*(s)}_{-j} \right)\left({{\boldsymbol{\theta}}}^{(s)}_{-j}- {{\boldsymbol{\theta}}}^{(s)}\right).
\end{equation*}
By recursively applying this equation from \({{\boldsymbol{\theta}}}^{(0)}_{-j}- {{\boldsymbol{\theta}}}^{(0)}=\boldsymbol{0}\), 
\begin{equation*}
    {{\boldsymbol{\theta}}}^{(t)}_{-j}- {{\boldsymbol{\theta}}}^{(t)} 
    = \sum_{s=0}^{t-1}\left(\prod_{k=s+1}^{t-1}  {\boldsymbol{Z}}\left( {{\boldsymbol{\theta}}}^{*(k)}_{-j} \right) \right)\Delta\boldsymbol{v}_{-j}\left({{\boldsymbol{\theta}}}^{*(s)}_{-j}\right).
\end{equation*}
Recalling \( \left\Vert {{\boldsymbol{\theta}}}^{(0)}_{-j} - {{\boldsymbol{\theta}}}^{(0)} \right\Vert=0\), and since \({\sigma}\neq 1\), we obtain the desired inequality as
\begin{align*}
     \left\Vert {{\boldsymbol{\theta}}}^{(t)}_{-j} - {{\boldsymbol{\theta}}}^{(t)} \right\Vert 
     &\leq   L_F  \sum_{s=0}^{t-1}  {\sigma}^{t-1-s} \\
     &=  L_F \frac{1-{\sigma}^{t}}{1- {\sigma}} .
\end{align*}
\end{proof}

\begin{lemma}  \label{lem:sum_msigma}
When \(a \neq 1\), then for every \(M>0\),
\begin{equation}
        \sum_{m=0}^{M-1} m a^{m} = \frac{a\left(1-a^{M-1}\right)}{\left(1-a\right)^2}  - \frac{\left(M-1\right)a^M}{1-a}
\end{equation}
\end{lemma}
\begin{proof}
Since \({\sigma} \neq 1\),
     \begin{align*}
             \left(1-{\sigma}\right)\sum_{m=0}^{M-1} m {\sigma}^{m} 
             &=  \sum_{m=1}^{M-1} {\sigma}^{m} - \left(M  -1\right){\sigma}^{M} \\
              &= \frac{{\sigma}\left(1-{\sigma}^{M-1}\right)}{1-{\sigma}}  - \left(M  -1\right){\sigma}^{M} .
 \end{align*}
By dividing  both sides of this equation by \((1-{\sigma})\) we obtain, 
 \begin{equation*}
        \sum_{m=0}^{M-1} m {\sigma}^{m} = \frac{{\sigma}\left(1-{\sigma}^{M-1}\right)}{\left(1-{\sigma}\right)^2}  - \frac{\left(M-1\right){\sigma}^M}{1-{\sigma}} 
\end{equation*}
obtaining the desired result.
\end{proof}

Here, we restate our result on the iterative differentiation with the proof.
\errorITD*
\begin{proof}
By using \({{\boldsymbol{\theta}}}^{*(s)}_{-j}\) defined in \cref{lem:error_itd_true} and recalling the definition of \(\widehat{\Delta {{\boldsymbol{\theta}}}_{-j}}\) in \cref{eq:estimator_itd}, we obtain
\begin{align} \label{eq:error_itd_decompose}
    \left\Vert \widehat{\Delta {{\boldsymbol{\theta}}}_{-j}}-\Delta{{\boldsymbol{\theta}}}^{(T)}_{-j} \right\Vert
    &=  \left\Vert \sum_{s=0}^{T-1} \left(\prod_{k=s+1}^{T-1} {\boldsymbol{Z}}\left({{\boldsymbol{\theta}}}^{(k)}\right)\right)\Delta\boldsymbol{v}_{-j}\left({{\boldsymbol{\theta}}}^{(s)}\right) - \sum_{s=0}^{T-1} \left(\prod_{k=s+1}^{T-1} {\boldsymbol{Z}}\left({{\boldsymbol{\theta}}}^{*(k)}_{-j}\right) \right) \Delta\boldsymbol{v}_{-j}\left({{\boldsymbol{\theta}}}^{*(s)}_{-j}\right) \right\Vert \notag \\
    &=\left\Vert \sum _{s=0}^{T-1}\left(\prod _{k=s+1}^{T-1}\boldsymbol{Z}\left(\boldsymbol{\theta }^{(k)}\right)\right) \Delta \boldsymbol{v}_{-j}\left(\boldsymbol{\theta }^{(s)}\right) -\left(\sum _{s=0}^{T-1}\left(\prod _{k=s+1}^{T-1}\boldsymbol{Z}\left(\boldsymbol{\theta }^{(k)}\right)\right) \Delta \boldsymbol{v}_{-j}\left(\boldsymbol{\theta }_{-j}^{*(s)}\right)\right) \right. \notag \\
    &\qquad \left. +\left(\sum _{s=0}^{T-1}\left(\prod _{k=s+1}^{T-1}\boldsymbol{Z}\left(\boldsymbol{\theta }^{(k)}\right)\right) \Delta \boldsymbol{v}_{-j}\left(\boldsymbol{\theta }_{-j}^{*(s)}\right)\right) -\sum _{s=0}^{T-1}\left(\prod _{k=s+1}^{T-1}\boldsymbol{Z}\left(\boldsymbol{\theta }_{-j}^{*(k)}\right)\right) \Delta \boldsymbol{v}_{-j}\left(\boldsymbol{\theta }_{-j}^{*(s)}\right)\right\Vert  \notag \\
    &\leq \left\Vert \sum_{s=0}^{T-1} \left( \prod_{k=s+1}^{T-1} {\boldsymbol{Z}}\left({{\boldsymbol{\theta}}}^{(k)}\right)\right) \left(\Delta\boldsymbol{v}_{-j}\left({{\boldsymbol{\theta}}}^{(s)}\right) - \Delta\boldsymbol{v}_{-j}\left({{\boldsymbol{\theta}}}^{*(s)}_{-j}\right)\right) \right\Vert  \notag \\
    &\qquad + \left\Vert \sum_{s=0}^{T-1}\left(\prod_{k=s+1}^{T-1} {\boldsymbol{Z}}\left({{\boldsymbol{\theta}}}^{(k)}\right) - \prod_{k=s+1}^{T-1} {\boldsymbol{Z}}\left({{\boldsymbol{\theta}}}^{*(k)}_{-j}\right)\right) \Delta\boldsymbol{v}_{-j}\left({{\boldsymbol{\theta}}}^{*(s)}_{j}\right)  \right\Vert.
\end{align}
From \cref{lem:error_itd_true} and since \({\sigma} \neq 1\), the first term in the right hand of \cref{eq:error_itd_decompose} can be bounded by
\begin{align} \label{eq:itd_error_first}
    \left\Vert \sum_{s=0}^{T-1} \left( \prod_{k=s+1}^{T-1} {\boldsymbol{Z}}\left({{\boldsymbol{\theta}}}^{(k)}\right)\right) \left(\Delta\boldsymbol{v}_{-j}\left({{\boldsymbol{\theta}}}^{(s)}\right) - \Delta\boldsymbol{v}_{-j}\left({{\boldsymbol{\theta}}}^{*(s)}_{-j}\right)\right) \right\Vert 
     &\leq L_{F^\prime}\sum_{s=0}^{T-1}  \left\Vert \prod_{k=s+1}^{T-1} {\boldsymbol{Z}}\left({{\boldsymbol{\theta}}}^{(k)}\right) \right\Vert \left\Vert {{\boldsymbol{\theta}}}^{(s)}-{{\boldsymbol{\theta}}}^{(s)}_{-j} \right\Vert \notag\\
    & \leq L_{F} L_{F^{\prime }}\sum _{s=0}^{T-1} {\sigma}^{T-1-s}\left(\frac{1-{\sigma}^{s}}{1-{\sigma}}\right)    \notag\\
    & =\frac{L_{F} L_{F^{\prime }}}{1-{\sigma}}\left(\sum _{s=0}^{T-1}\left( {\sigma}^{T-1-s} -{\sigma}^{T-1}\right)\right) \notag\\
    & =\frac{L_{F} L_{F^{\prime }}}{1-{\sigma}}\left(\frac{\left( 1-{\sigma}^{T}\right)}{1-{\sigma}} -T{\sigma}^{T-1}\right) \notag \\
    & =\frac{L_{F} L_{F^{\prime }}}{( 1-{\sigma})^{2}}\left( 1-{\sigma}^{T}\right) -\frac{L_{F} L_{F^{\prime }} T}{1-{\sigma}} {\sigma}^{T-1} \notag \\ 
   &=\frac{L_{F} L_{F^{\prime }}}{(1-{\sigma} )^{2}}\left( 1-{\sigma}^{T} -T{\sigma}^{T-1} (1-{\sigma} )\right) .
\end{align}
Regarding the second term in the right hand of \cref{eq:error_itd_decompose}, \cref{lem:error_itd_true,lem:sum_msigma}, \cref{ass:smooth_Z}, and \({\sigma} \neq 1\) ensure
\begin{align} \label{eq:itd_error_second}
    &\left\Vert \sum_{s=0}^{T-1}\left(\prod_{k=s+1}^{T-1} {\boldsymbol{Z}}\left({{\boldsymbol{\theta}}}^{(k)}\right) - \prod_{k=s+1}^{T-1} {\boldsymbol{Z}}\left({{\boldsymbol{\theta}}}^{*(k)}_{-j}\right)\right) \Delta\boldsymbol{v}_{-j}\left({{\boldsymbol{\theta}}}^{*(s)}_{j}\right)  \right\Vert \notag\\
    &\begin{aligned}[b]
    \quad &= \left\Vert \sum_{s=0}^{T-1} \left\{\sum_{k=s+1}^{T-1} \left(\prod_{t=k+1}^{T-1} {\boldsymbol{Z}}\left({{\boldsymbol{\theta}}}^{(t)}\right)\right)\left({\boldsymbol{Z}}\left({{\boldsymbol{\theta}}}^{*(k)}_{-j}\right)-{\boldsymbol{Z}}\left({{\boldsymbol{\theta}}}^{(k)}\right)\right) \left(\prod_{t=s+1}^{k-1} {\boldsymbol{Z}}\left({{\boldsymbol{\theta}}}^{*(t)}_{-j}\right)\right) \right\} \Delta\boldsymbol{v}_{-j}\left({{\boldsymbol{\theta}}}^{*(s)}_{j}\right) \right\Vert \\
    &\leq L_{F}\sum _{s=0}^{T-1} {\sigma}^{T-2-s}\sum _{k=s+1}^{T-1} L_{\boldsymbol{Z}}\left\Vert \boldsymbol{\theta }^{(k)} -\boldsymbol{\theta }_{-j}^{(k)}\right\Vert \\
    &\leq \frac{L_{F}^{2} L_{\boldsymbol{Z}}}{1-{\sigma}}\sum _{s=0}^{T-1} {\sigma}^{T-2-s}\sum _{k=s+1}^{T-1}\left( 1-{\sigma}^{k}\right)\\
    &=\frac{L_{F}^{2} L_{\boldsymbol{Z}}}{1-{\sigma}}\sum _{s=0}^{T-1} {\sigma}^{T-2-s}\left( (T-1-s)-\frac{{\sigma}^{s+1}\left( 1-{\sigma}^{T-1-s}\right)}{1-{\sigma}}\right)\\
    &=\frac{L_{F}^{2} L_{\boldsymbol{Z}}}{1-{\sigma}}\sum _{s=0}^{T-1}\left( (T-1-s){\sigma}^{T-2-s} -\frac{{\sigma}^{T-1}\left( 1-{\sigma}^{T-1-s}\right)}{1-{\sigma}}\right)\\
    &=\frac{L_{F}^{2} L_{\boldsymbol{Z}}}{1-{\sigma}}\left(\frac{1}{{\sigma}}\sum _{s=0}^{T-1} (T-1-s){\sigma}^{T-1-s} -\frac{1}{1-{\sigma}}\left(\sum _{s=0}^{T-1} {\sigma}^{T-1} -\sum _{s=0}^{T-1} {\sigma}^{2T-2-s}\right)\right)\\
    &=\frac{L_{F}^{2} L_{\boldsymbol{Z}}}{1-{\sigma}}\left(\frac{1}{{\sigma}}\left(\frac{{\sigma}\left( 1-{\sigma}^{T-1}\right)}{(1-{\sigma} )^{2}} -\frac{(T-1){\sigma}^{T}}{1-{\sigma}}\right) -\frac{1}{1-{\sigma}}\left( T{\sigma}^{T-1} -\frac{{\sigma}^{T-1}\left( 1-{\sigma}^{T}\right)}{1-{\sigma}}\right)\right)\\
    &=\frac{L_{F}^{2} L_{\boldsymbol{Z}}}{( 1-{\sigma})^{3}}\left(\frac{1}{{\sigma}}\left( {\sigma}\left( 1-{\sigma}^{T-1}\right) -(T-1)( 1-{\sigma}) {\sigma}^{T}\right) -( 1-{\sigma}) T{\sigma}^{T-1} +\left( 1-{\sigma}^{T}\right) {\sigma}^{T-1}\right)\\
    &=\frac{L_{F}^{2} L_{\boldsymbol{Z}}}{( 1-{\sigma})^{3}}\left( 1-{\sigma}^{T-1} -(T-1)( 1-{\sigma}) {\sigma}^{T-1} -( 1-{\sigma}) T{\sigma}^{T-1} +\left( 1-{\sigma}^{T}\right) {\sigma}^{T-1}\right)\\
    &=\frac{L_{F}^{2} L_{\boldsymbol{Z}}}{( 1-{\sigma})^{3}}\left( 1-{\sigma}^{T-1} -( 1-{\sigma}) T{\sigma}^{T-1} +( 1-{\sigma}) {\sigma}^{T-1} -( 1-{\sigma}) T{\sigma}^{T-1} +\left( 1-{\sigma}^{T}\right) {\sigma}^{T-1}\right)\\
    &=\frac{L_{F}^{2} L_{\boldsymbol{Z}}}{( 1-{\sigma})^{3}}\left( 1-{\sigma}^{T-1} +( 1-{\sigma}) {\sigma}^{T-1} -( 1-{\sigma}) 2T{\sigma}^{T-1} +{\sigma}^{T-1} -{\sigma}^{2T-1}\right)\\
    &=\frac{L_{F}^{2} L_{\boldsymbol{Z}}}{( 1-{\sigma})^{3}}\left( 1+( 1-{\sigma} -2T+2T{\sigma}) {\sigma}^{T-1} -{\sigma}^{2T-1}\right)\\
    &=\frac{L_{F}^{2} L_{\boldsymbol{Z}}}{( 1-{\sigma})^{3}}\left( 1-( 2T-1)( 1-{\sigma}) {\sigma}^{T-1} -{\sigma}^{2T-1}\right),
    \end{aligned}
\end{align}
From \cref{eq:error_itd_decompose,eq:itd_error_first,eq:itd_error_second}, we obtain the desired bound as
\begin{align*} 
&\left\Vert \widehat{\Delta {{\boldsymbol{\theta}}}_{-j}}-\Delta{{\boldsymbol{\theta}}}^{(T)}_{-j} \right\Vert  \\
&\leq\frac{L_{F} L_{F^{\prime }}}{(1-{\sigma} )^{2}}\left( 1-{\sigma}^{T} -T{\sigma}^{T-1} (1-{\sigma} )\right) + \frac{L_{F}^{2} L_{\boldsymbol{Z}}}{( 1-{\sigma})^{3}}\left( 1-( 2T-1)( 1-{\sigma}) {\sigma}^{T-1} -{\sigma}^{2T-1}\right). 
\end{align*}
\end{proof}

We finally show \cref{col:error_itd_non_convex,col:error_itd_convex} using \cref{th:error_itd}.
\errorITDNonConvex*
\begin{proof}
From \cref{th:error_itd} and  \({\sigma} > 1 \) we obtain
\begin{align*}
\left\Vert \widehat{\Delta \boldsymbol{\theta }_{-j}} -\Delta \boldsymbol{\theta }_{-j}^{(T)}\right\Vert  & \leq \frac{L_{F} L_{F^{\prime }}}{({\sigma} -1)^{2}}\left( T{\sigma}^{T-1} ({\sigma} -1)-{\sigma}^{T} +1\right) +\frac{L_{F}^{2} L_{\boldsymbol{Z}}}{( {\sigma} -1)^{3}}\left( {\sigma}^{2T-1} -( 2T-1)( {\sigma} -1) {\sigma}^{T-1} -1\right)\\
 & \leq \frac{L_{F} L_{F^{\prime }}}{({\sigma} -1)^{2}} T{\sigma}^{T-1} ({\sigma} -1)+\frac{L_{F}^{2} L_{\boldsymbol{Z}}}{( {\sigma} -1)^{3}} {\sigma}^{2T-1}\\
 & \leq \ \frac{L_{F} L_{F^{\prime }}}{({\sigma} -1)^{2}} T{\sigma}^{T} +\frac{L_{F}^{2} L_{\boldsymbol{Z}}}{( {\sigma} -1)^{3}} {\sigma}^{2T-1}.
\end{align*}
\end{proof}
We then show \cref{col:error_itd_convex} as the consequence of \cref{th:error_itd,lem:Z_bound,lem:converge}.
\errorITDConvex*
\begin{proof}
From \cref{th:error_itd,lem:Z_bound,lem:converge}, we obtain 
\begin{align*}
\left\Vert \widehat{\Delta \boldsymbol{\theta }_{-j}} -\Delta \boldsymbol{\theta }_{-j}^{(T)}\right\Vert  & \leq \frac{L_{F} L_{F^{\prime }}}{(\sigma_{\mathcal{B}} -1)^{2}}\left( T\sigma_{\mathcal{B}}^{T-1} (\sigma_{\mathcal{B}} -1)-\sigma_{\mathcal{B}}^{T} +1\right) +\frac{L_{F}^{2} L_{\boldsymbol{Z}}}{( \sigma_{\mathcal{B}} -1)^{3}}\left( \sigma_{\mathcal{B}}^{2T-1} -( 2T-1)( \sigma_{\mathcal{B}} -1) \sigma_{\mathcal{B}}^{T-1} -1\right)\\
 & =\frac{L_{F} L_{F^{\prime }}}{(1-\sigma_{\mathcal{B}} )^{2}}\left( -T\sigma_{\mathcal{B}}^{T-1} (1-\sigma_{\mathcal{B}} )-\sigma_{\mathcal{B}}^{T} +1\right) +\frac{L_{F}^{2} L_{\boldsymbol{Z}}}{( 1-\sigma_{\mathcal{B}})^{3}}\left( -\sigma_{\mathcal{B}}^{2T-1} -( 2T-1)( 1-\sigma_{\mathcal{B}}) \sigma_{\mathcal{B}}^{T-1} +1\right)\\
 & \leq \ \frac{L_{F} L_{F^{\prime }}}{(1-\sigma_{\mathcal{B}} )^{2}}\left( 1-\sigma_{\mathcal{B}}^{T}\right) +\frac{L_{F}^{2} L_{\boldsymbol{Z}}}{( 1-\sigma_{\mathcal{B}})^{3}}\left( 1-\sigma_{\mathcal{B}}^{2T-1}\right).
\end{align*}
\end{proof}

\subsection{Proof of \texorpdfstring{\cref{th:error_aid}}{Theorem 2}} \label{app:proof_aid}
\begin{lemma} \label{lem:game_convernge}
Under \cref{ass:J_bound,ass:neighbor}, for every \(\boldsymbol{\theta}^{(0)}\in\mathcal{B}({\boldsymbol{\theta}}^{*}\)) and \(T\geq0\),  \(\boldsymbol{\theta}^{(T)}\) given by \cref{eq:update} satisfies
\begin{align*}
     \left\Vert \boldsymbol{\theta }^{(T)}- {{\boldsymbol{\theta}}}^{*} \right\Vert \leq \rho \sigma_{\mathcal{B}}^T .
\end{align*}
\end{lemma}
\begin{proof}
Since \(U(\boldsymbol{\theta })\) is contraction mapping and its Lipschitz constant is at most \(\sigma_{\mathcal{B}}\) as shown in the proof of \cref{lem:converge}, we have \(\left\Vert \boldsymbol{\theta }^{(t+1)}- {{\boldsymbol{\theta}}}^{*}\right\Vert \leq \sigma_{\mathcal{B}} \left\Vert \boldsymbol{\theta }^{(t)}- {{\boldsymbol{\theta}}}^{*}\right\Vert\).
By recursively applying this from \(\left\Vert \boldsymbol{\theta }^{(0)}- {{\boldsymbol{\theta}}}^{*}\right\Vert \leq \rho\) for \(t=0,\ldots,T\), we obtain the desired result.
\end{proof}
\begin{lemma} \label{lem:influence}
Under \cref{ass:J_bound}, we have
\begin{align*}
     \left\Vert {{\boldsymbol{\theta}}}_{-j}^{*} - {{\boldsymbol{\theta}}}^{*} \right\Vert = \Delta{{\boldsymbol{\theta}}}_{-j}^{*} \leq  \frac{L_{F}}{1-\sigma_{\mathcal{B}}},
\end{align*}
where \({{\boldsymbol{\theta}}}_{-j}^{*}\coloneqq  {{\boldsymbol{\theta}}}_{-j,0}^{*}\). 
\end{lemma}
\begin{proof}
From the definitions \(\Delta{{\boldsymbol{\theta}}}_{-j}^{*} ={{{\boldsymbol{\theta}}}_{-j,1}^{*}} - {{{\boldsymbol{\theta}}}_{-j,0}^{*}}\), there exists \(r\) such that
\begin{equation*} \label{eq:true_reminder}
    \Delta{{\boldsymbol{\theta}}}_{-j}^{*} = \left({\boldsymbol{I}}-{\boldsymbol{Z}}\left({{\boldsymbol{\theta}}}_{-j,r}^{*}\right)\right)^{-1} \Delta\boldsymbol{v}_{-j}\left({{\boldsymbol{\theta}}}_{-j,r}^{*}\right) .
\end{equation*}
Since \cref{ass:J_bound} ensures \(0 \prec {\boldsymbol{Z}}\left({{\boldsymbol{\theta}}}\right)^\intercal {\boldsymbol{Z}}\left({{\boldsymbol{\theta}}}\right)\preceq \sigma_{\mathcal{B}}^2\), 
\begin{equation*}
    \left\Vert \Delta{{\boldsymbol{\theta}}}_{-j}^{*} \right\Vert \leq  \frac{L_{F}}{1-\sigma_{\mathcal{B}}},
\end{equation*}
obtaining the desired result.
\end{proof}

We restate our result on the influence estimator using the approximate implicit differentiation.
\errorAID*
\begin{proof}
We first decompose the approximation error by
\begin{align} \label{eq:error_decomposed} 
    \left\Vert \widetilde{\Delta{{\boldsymbol{\theta}}}_{-j}}-\Delta{{\boldsymbol{\theta}}}^{(T)}_{-j} \right\Vert &\leq
    \underbrace{\left\Vert \widetilde{\Delta{{\boldsymbol{\theta}}}_{-j}} - \sum_{m=0}^{M-1} {\boldsymbol{Z}}\left({{\boldsymbol{\theta}}}^{*}\right)^{m}\Delta\boldsymbol{v}_{-j}\left({{\boldsymbol{\theta}}}^{*}\right) \right\Vert}_{\clubsuit}
    + \underbrace{\left\Vert \sum_{m=0}^{M-1} {\boldsymbol{Z}}\left({{\boldsymbol{\theta}}}^{*}\right)^{m}  \Delta\boldsymbol{v}_{-j}\left({{\boldsymbol{\theta}}}^{*}\right) - \frac{\mathrm{d} {{{\boldsymbol{\theta}}}_{-j,0}^{*}}}{\mathrm{d} {\epsilon}} \right\Vert}_{\spadesuit} \notag\\
    &\qquad + \underbrace{\left\Vert \frac{\mathrm{d} {{{\boldsymbol{\theta}}}_{-j,0}^{*}}}{\mathrm{d} {\epsilon}} - \Delta{{\boldsymbol{\theta}}}_{-j}^{*} \right\Vert}_{\heartsuit}
    + \underbrace{\left\Vert \Delta{{\boldsymbol{\theta}}}_{-j}^{*}-\Delta{{\boldsymbol{\theta}}}^{(T)}_{-j} \right\Vert}_{\diamondsuit}.
\end{align}
In \cref{eq:error_decomposed}, 
\begin{itemize}
    \item \(\clubsuit\) expresses the error norm between AID estimations on \(\boldsymbol{\theta}^{(T)}\) and \(\boldsymbol{\theta}^{*}\).
    \item \(\spadesuit\) expresses the error norm between AID estimations on  \(\boldsymbol{\theta}^{*}\) that use finite \(M\) and infinite \(M\) for the inverse matrix approximation.
    \item \(\heartsuit\) expresses the error norm yielded by the linear approximation using \(\frac{\mathrm{d} {{{\boldsymbol{\theta}}}_{-j,0}^{*}}}{\mathrm{d} {\epsilon}} \).
    \item \(\diamondsuit\) expresses the error norm between true influence on \(\boldsymbol{\theta}^{*}\) and \(\boldsymbol{\theta}^{(T)}\).
\end{itemize}
\paragraph{Bound of \texorpdfstring{\(\clubsuit\)}{♣︎}}
\begin{align} \label{eq:bound_decompose}
   & \left\Vert \widetilde{\Delta{{\boldsymbol{\theta}}}_{-j}} -  \sum_{m=0}^{M-1} {\boldsymbol{Z}}\left({{\boldsymbol{\theta}}}^{*}\right)^{m}\Delta\boldsymbol{v}_{-j}\left({{\boldsymbol{\theta}}}^{*}\right) \right\Vert \notag\\
   &\qquad= \left\Vert \sum_{m=0}^{M-1} {\boldsymbol{Z}}\left({{\boldsymbol{\theta}}}^{(T)}\right)^{m}\Delta\boldsymbol{v}_{-j}\left({{\boldsymbol{\theta}}}^{(T)}\right) - \sum_{m=0}^{M-1} {\boldsymbol{Z}}\left({{\boldsymbol{\theta}}}^{*}\right)^{m} \Delta\boldsymbol{v}_{-j}\left({{\boldsymbol{\theta}}}^{*}\right) \right\Vert \notag \\
    &\qquad\leq \left\Vert \sum_{m=0}^{M-1}{\boldsymbol{Z}}\left({{\boldsymbol{\theta}}}^{(T)}\right)^{m} \left(\Delta\boldsymbol{v}_{-j}\left({{\boldsymbol{\theta}}}^{(T)}\right) - \Delta\boldsymbol{v}_{-j}\left({{\boldsymbol{\theta}}}^{*}\right)\right)\right\Vert + \left\Vert \sum_{m=0}^{M-1}\left({\boldsymbol{Z}}\left({{\boldsymbol{\theta}}}^{(T)}\right)^{m} - {\boldsymbol{Z}}\left({{\boldsymbol{\theta}}}^{*}\right)^{m}\right) \Delta\boldsymbol{v}_{-j}\left({{\boldsymbol{\theta}}}^{*}\right) \right\Vert 
\end{align}
From \cref{ass:J_bound}, the first term in the right hand of \cref{eq:bound_decompose} can be bounded by 
\begin{align} \label{eq:bound_club_first}
    \left\Vert \sum_{m=0}^{M-1}{\boldsymbol{Z}}\left({{\boldsymbol{\theta}}}^{(T)}\right)^{m}  \left(\Delta\boldsymbol{v}_{-j}\left({{\boldsymbol{\theta}}}^{(T)}\right) - \Delta\boldsymbol{v}_{-j}\left({{\boldsymbol{\theta}}}^{*}\right)\right)\right\Vert
    &\leq L_{F^\prime} \left\Vert \sum_{m=0}^{M-1}{\boldsymbol{Z}}\left({{\boldsymbol{\theta}}}^{(T)}\right)^{m} \right\Vert \left\Vert {{\boldsymbol{\theta}}}^{(T)} - {{\boldsymbol{\theta}}}^{*} \right\Vert \notag \\
    &=L_{F^\prime} \frac{ 1-\sigma_{\mathcal{B}}^M}{1-\sigma_{\mathcal{B}}}  \rho \sigma_{\mathcal{B}}^T \\
    &= \frac{\rho L_{F^\prime} \sigma_{\mathcal{B}}^T\left(1-\sigma_{\mathcal{B}}^M\right)}{1-\sigma_{\mathcal{B}}}.
\end{align} 
From \cref{ass:J_bound,ass:smooth_Z} and \cref{lem:sum_msigma}, the second term of the right hand in \cref{eq:bound_decompose} can be bounded as
\begin{align} \label{eq:bound_club_second}
    \left\Vert \sum_{m=0}^{M-1}\left({\boldsymbol{Z}}\left({{\boldsymbol{\theta}}}^{(T)}\right)^{m} - {\boldsymbol{Z}}\left({{\boldsymbol{\theta}}}^{*}\right)^{m}\right)  \Delta\boldsymbol{v}_{-j}\left({{\boldsymbol{\theta}}}^{*}\right) \right\Vert
  &= \left\Vert \sum_{m=0}^{M-1} \left\{ \sum_{s=0}^{m-1} {\boldsymbol{Z}}\left({{\boldsymbol{\theta}}}^{(T)}\right)^{m-1- s} \left({\boldsymbol{Z}}\left({{\boldsymbol{\theta}}}^{(T)}\right)-{\boldsymbol{Z}}\left({{\boldsymbol{\theta}}}^{*}\right)\right) {\boldsymbol{Z}}\left({{\boldsymbol{\theta}}}^{*}\right)^s \right\} \Delta\boldsymbol{v}_{-j}\left({{\boldsymbol{\theta}}}^{*}\right)  \right\Vert \notag \\
    &\leq L_{F}L_{{\boldsymbol{Z}}}\left\Vert {{\boldsymbol{\theta}}}^{(T)}-{{\boldsymbol{\theta}}}^{*} \right\Vert\sum_{m=0}^{M-1} m \sigma_{\mathcal{B}}^{m-1} \notag \\
    &\leq L_{F}L_{{\boldsymbol{Z}}}\rho \sigma_{\mathcal{B}}^{T} \frac{\left(1-\sigma_{\mathcal{B}}^{M-1}\right)}{\left(1-\sigma_{\mathcal{B}}\right)^2} \notag \\
    &= \frac{\rho L_{F} L_{{\boldsymbol{Z}}}\sigma_{\mathcal{B}}^{T} \left(1-\sigma_{\mathcal{B}}^{M-1}\right)}{\left(1-\sigma_{\mathcal{B}}\right)^2} .
\end{align} 
From \cref{eq:bound_club_first,eq:bound_club_second} and \(\left(1-\sigma_{\mathcal{B}}^{M-1}\right) < \left(1-\sigma_{\mathcal{B}}^{M}\right)\), we obtain
\begin{align} \label{eq:bound_club}
   \left\Vert \widetilde{\Delta{{\boldsymbol{\theta}}}_{-j}} - \sum_{m=0}^{M-1} {\boldsymbol{Z}}\left({{\boldsymbol{\theta}}}^{*}\right)^{m} + \Delta\boldsymbol{v}_{-j}\left({{\boldsymbol{\theta}}}^{*}\right) \right\Vert
   & \leq \frac{\rho L_{F^\prime} \sigma_{\mathcal{B}}^T\left(1-\sigma_{\mathcal{B}}^M\right)}{1-\sigma_{\mathcal{B}}}+ \frac{\rho L_{F} L_{{\boldsymbol{Z}}}\sigma_{\mathcal{B}}^{T} \left(1-\sigma_{\mathcal{B}}^{M-1}\right)}{\left(1-\sigma_{\mathcal{B}}\right)^2}  \notag \\
    & \leq \left( \frac{\rho L_{F^\prime} }{1-\sigma_{\mathcal{B}}}+ \frac{\rho L_{F} L_{{\boldsymbol{Z}}}}{\left(1-\sigma_{\mathcal{B}}\right)^2}\right) \sigma_{\mathcal{B}}^{T} \left(1-\sigma_{\mathcal{B}}^{M-1}\right). 
\end{align}

\paragraph{Bound of \texorpdfstring{\(\spadesuit\)}{♠︎}}
\noindent From \cref{ass:J_bound}, 
\begin{align} \label{eq:bound_spade}
    \left\Vert  \sum_{m=0}^{M-1} {\boldsymbol{Z}}\left({{\boldsymbol{\theta}}}^{*}\right)^{m}\Delta\boldsymbol{v}_{-j}\left({{\boldsymbol{\theta}}}^{*}\right) - \frac{\mathrm{d} {{{\boldsymbol{\theta}}}_{-j,0}^{*}}}{\mathrm{d} {\epsilon}} \right\Vert
    &= \left\Vert \sum_{m=0}^{M-1} {\boldsymbol{Z}}\left({{\boldsymbol{\theta}}}^{*}\right)^{m} \Delta\boldsymbol{v}_{-j}\left({{\boldsymbol{\theta}}}^{*}\right)-  \sum_{m=0}^{\infty} {\boldsymbol{Z}}\left({{\boldsymbol{\theta}}}^{*}\right)^{m}\Delta\boldsymbol{v}_{-j}\left({{\boldsymbol{\theta}}}^{*}\right) \right\Vert \notag \\
    &= \left\Vert \sum_{m=M}^{\infty} {\boldsymbol{Z}}\left({{\boldsymbol{\theta}}}^{*}\right)^{m}\Delta\boldsymbol{v}_{-j}\left({{\boldsymbol{\theta}}}^{*}\right) \right\Vert \notag  \\
    &\leq L_{F} \sum_{m=M}^{\infty}\sigma_{\mathcal{B}}^{m}  \notag  \\
    &= \frac{ L_{F}\sigma_{\mathcal{B}}^M}{1-\sigma_{\mathcal{B}}} .
\end{align}

\paragraph{Bound of \texorpdfstring{\(\heartsuit\)}{♡}}
\noindent Using \cref{eq:true_reminder}, we can bound \(\heartsuit\) by
\begin{align} \label{eq:bound_heart_tmp}
    &\left\Vert \frac{\mathrm{d} {{{\boldsymbol{\theta}}}_{-j,0}^{*}}}{\mathrm{d} {\epsilon}} - \Delta{{\boldsymbol{\theta}}}_{-j}^{*} \right\Vert \notag\\
    &\qquad =  \left\Vert \left({\boldsymbol{I}}- {\boldsymbol{Z}}\left({{\boldsymbol{\theta}}}^{*}\right)\right)^{-1} \Delta\boldsymbol{v}_{-j}\left({{\boldsymbol{\theta}}}^{*}\right) - \left({\boldsymbol{I}}-{\boldsymbol{Z}}\left({{\boldsymbol{\theta}}}_{-j,r}^{*}\right)\right)^{-1} \Delta\boldsymbol{v}_{-j}\left({{\boldsymbol{\theta}}}_{-j,r}^{*}\right) \right\Vert\notag\\
    &\qquad\leq   \left\Vert\left({\boldsymbol{I}}- {\boldsymbol{Z}}\left({{\boldsymbol{\theta}}}^{*}\right)\right)^{-1} \left(\Delta\boldsymbol{v}_{-j}\left({{\boldsymbol{\theta}}}^{*}\right)- \Delta\boldsymbol{v}_{-j}\left({{\boldsymbol{\theta}}}_{-j,r}^{*}\right)\right) \right\Vert +  \left\Vert \left( \left({\boldsymbol{I}}-{\boldsymbol{Z}}\left({{\boldsymbol{\theta}}}_{-j,r}^{*}\right)\right)^{-1} - \left({\boldsymbol{I}}-{\boldsymbol{Z}}\left({{\boldsymbol{\theta}}}^{*}\right)\right)^{-1}\right) \Delta\boldsymbol{v}_{-j}\left({{\boldsymbol{\theta}}}_{-j,r}^{*}\right) \right\Vert
\end{align}
From \cref{lem:influence} and \cref{ass:J_bound}, the first term of the right hand of \cref{eq:bound_heart_tmp} can be bounded as
\begin{align} \label{eq:bound_heart_first}
    \left\Vert\left({\boldsymbol{I}}- {\boldsymbol{Z}}\left({{\boldsymbol{\theta}}}^{*}\right)\right)^{-1} \left(\Delta\boldsymbol{v}_{-j}\left({{\boldsymbol{\theta}}}^{*}\right)- \Delta\boldsymbol{v}_{-j}\left({{\boldsymbol{\theta}}}_{-j,r}^{*}\right)\right) \right\Vert
       & \leq \frac{1}{1-\sigma_{\mathcal{B}}}  L_{F^\prime} \left\Vert {{\boldsymbol{\theta}}}_{-j}^{*} -{{\boldsymbol{\theta}}}^{*} \right\Vert \notag\\
       & \leq \frac{L_{F} L_{F^\prime}}{\left(1-\sigma_{\mathcal{B}}\right)^2} .
\end{align}
By using the Neumann series expression of the inverse matrix and recalling \cref{lem:sum_msigma} with \(M\to\infty\), the the second term of \cref{eq:bound_heart_tmp} can be bounded similarly to \cref{eq:bound_club_second}:
\begin{align} \label{eq:bound_heart_second}
&\left\Vert \left( \left({\boldsymbol{I}}-{\boldsymbol{Z}}\left({{\boldsymbol{\theta}}}_{-j,r}^{*}\right)\right)^{-1} - \left({\boldsymbol{I}}-{\boldsymbol{Z}}\left({{\boldsymbol{\theta}}}^{*}\right)\right)^{-1}\right) \Delta\boldsymbol{v}_{-j}\left({{\boldsymbol{\theta}}}_{-j,r}^{*}\right) \right\Vert \notag\\
  & \qquad=\left\Vert \left( \sum_{m=0}^{\infty}{\boldsymbol{Z}}\left({{\boldsymbol{\theta}}}_{-j,r}^{*}\right)^{m} -  \sum_{m=0}^{\infty} {\boldsymbol{Z}}\left({{\boldsymbol{\theta}}}^{*}\right)^{m}\right)  \Delta\boldsymbol{v}_{-j}\left({{\boldsymbol{\theta}}}_{-j,r}^{*}\right) \right\Vert \notag\\
    & \qquad=   \left\Vert  \left( \sum_{m=0}^{\infty}\sum_{s=0}^{m-1}{\boldsymbol{Z}}\left({{\boldsymbol{\theta}}}_{-j,r}^{*}\right)^{m-s-1} \left({\boldsymbol{Z}}\left({{\boldsymbol{\theta}}}_{-j,r}^{*}\right) - {\boldsymbol{Z}}\left({{\boldsymbol{\theta}}}^{*}\right)\right) {\boldsymbol{Z}}\left({{\boldsymbol{\theta}}}^{*}\right)^{s}\right)  \Delta\boldsymbol{v}_{-j}\left({{\boldsymbol{\theta}}}_{-j,r}^{*}\right) \right\Vert  \notag\\
      &\qquad \leq L_{F}L_{{\boldsymbol{Z}}} \left\Vert {{\boldsymbol{\theta}}}_{-j}^{*} -{{\boldsymbol{\theta}}}^{*} \right\Vert \sum_{m=0}^{\infty} m \sigma_{\mathcal{B}}^{m-1} \notag\\
   &   \qquad \leq L_{F}L_{{\boldsymbol{Z}}} \frac{L_{F}}{1-\sigma_{\mathcal{B}}} \frac{\sigma_{\mathcal{B}}}{\sigma_{\mathcal{B}}\left(1-\sigma_{\mathcal{B}}\right)^{2}} \notag\\
   &   \qquad = \frac{L_{F}^2 L_{{\boldsymbol{Z}}}}{\left(1-\sigma_{\mathcal{B}}\right)^{3}}     .
\end{align}
From \cref{eq:bound_heart_first} and \cref{eq:bound_heart_second}, 
\begin{equation} \label{eq:bound_heart}
      \left\Vert \frac{\mathrm{d} {{{\boldsymbol{\theta}}}_{-j,0}^{*}}}{\mathrm{d} {\epsilon}} - \Delta{{\boldsymbol{\theta}}}_{-j}^{*} \right\Vert 
      \leq \frac{L_{F} L_{F^\prime}}{\left(1-\sigma_{\mathcal{B}}\right)^2} +  \frac{L_{F}^2 L_{{\boldsymbol{Z}}}}{\left(1-\sigma_{\mathcal{B}}\right)^{3}}  .
\end{equation}  
  
\paragraph{Bound of \texorpdfstring{\(\diamondsuit\)}{♢}}
\noindent We rewrite \(\diamondsuit\) by
\begin{align} \label{eq:bound_diamond}
    \left\Vert \Delta{{\boldsymbol{\theta}}}_{-j}^{*}-\Delta{{\boldsymbol{\theta}}}^{(T)}_{-j} \right\Vert
    &= \left\Vert \left({{\boldsymbol{\theta}}}_{-j}^{*} - {{\boldsymbol{\theta}}}^{*}\right) - \left({{\boldsymbol{\theta}}}^{(T)}_{-j} - {{\boldsymbol{\theta}}}^{(T)}\right) \right\Vert \notag \\
        &\leq \left\Vert \left({{\boldsymbol{\theta}}}_{-j}^{*} - {{\boldsymbol{\theta}}}^{(T)}_{-j}\right)\right\Vert + \left\Vert \left({{\boldsymbol{\theta}}}^{*}- {{\boldsymbol{\theta}}}^{(T)}\right) \right\Vert \notag \\
        &\leq 2\rho\sigma_{\mathcal{B}}^T.
\end{align}

By combining \cref{eq:error_decomposed,eq:bound_club,eq:bound_spade,eq:bound_heart,eq:bound_diamond}, we obtain the desired bound as
\begin{align*}
     \left\Vert \widetilde{\Delta{{\boldsymbol{\theta}}}_{-j}}-\Delta{{\boldsymbol{\theta}}}^{(T)}_{-j} \right\Vert 
     &\leq \left( \frac{\rho L_{F^\prime} }{1-\sigma_{\mathcal{B}}}+ \frac{\rho L_{F} L_{{\boldsymbol{Z}}}}{\left(1-\sigma_{\mathcal{B}}\right)^2}\right) \sigma_{\mathcal{B}}^{T} \left(1-\sigma_{\mathcal{B}}^{M-1}\right) +  \frac{ L_{F}}{1-\sigma_{\mathcal{B}}}\sigma_{\mathcal{B}}^M  +2\rho\sigma_{\mathcal{B}}^T  +  \frac{L_{F}^2 L_{{\boldsymbol{Z}}}}{\left(1-\sigma_{\mathcal{B}}\right)^{3}}  +  \frac{L_{F} L_{F^\prime}}{\left(1-\sigma_{\mathcal{B}}\right)^2} .
\end{align*}
\end{proof}

\section{Extending ITD-EIGEM to Common GAN Optimization Techniques} \label{app:itd_ext}

\subsection{Minibatch Training} \label{app:minibatch}
\noindent This section explains the extension of our method to minibatch settings.

To extend our method to the minibatch setting, we define the adversarial stochastic gradient descent (ASGD) for GANs. 
Let \(\mathcal{X}_{t} \subset \mathcal{X}_{\boldsymbol{x}}\) be the set of minibatch instances at the \(t\)-th step.
We redefine the loss \(V\) to take a minibatch \(\mathcal{X}_t\) as its input:
\begin{equation}
    V\left({{\boldsymbol{\varphi}}},{\boldsymbol{\psi}}; \mathcal{X}_t\right) 
    \coloneqq \overline{f}(\boldsymbol{\psi}, \mathcal{X}_t) + \overline{g}(\boldsymbol{\varphi}, \boldsymbol{\psi}, \mathcal{Z}_t),
\end{equation}
where \(\mathcal{Z}_t \subset \mathcal{Z}\) is the corresponding set of sampled latent variables for the minibatch.
The ASGD updates the concatenated parameters \({{\boldsymbol{\theta}}}\coloneqq \left({{\boldsymbol{\varphi}}}{}^\intercal ~ {\boldsymbol{\psi}}{}^\intercal\right){}^\intercal \in \mathbb{R}^{d_{{{\boldsymbol{\theta}}}}=d_{\varphi}+d_{\psi}}\) by:
\begin{equation}
    {{\boldsymbol{\theta}}}^{(t+1)} = {{\boldsymbol{\theta}}}^{(t)} - \eta {\boldsymbol{v}}\left({{\boldsymbol{\theta}}}^{(t)}; \mathcal{X}_t\right), \label{eq:update_mini}
\end{equation}
where \(\eta \in \mathbb{R}^+\) denotes the learning rate, and \({\boldsymbol{v}}\left({{\boldsymbol{\theta}}}; \mathcal{X}_t\right)\) denotes the concatenated gradient for the minibatch:
\begin{equation}
  {\boldsymbol{v}}\left({{\boldsymbol{\theta}}}; \mathcal{X}_t\right) \coloneqq \begin{pmatrix}
      {\nabla_{\boldsymbol{\varphi}}} V\left({{\boldsymbol{\varphi}}}, {\boldsymbol{\psi}}; \mathcal{X}_t\right) \\ -{\nabla_{\boldsymbol{\psi}}} V\left({{\boldsymbol{\varphi}}}, {\boldsymbol{\psi}}; \mathcal{X}_t\right)
  \end{pmatrix}.
 \end{equation}
We define the counterfactual ASGD to represent the parameter updates when a training instance indexed as \(j\) is removed. Let \(V_{-j}\left({{\boldsymbol{\varphi}}},{\boldsymbol{\psi}}; \mathcal{X}_t\right)\) denote the modified loss function, which takes the removal into account:
\begin{equation}
    V_{-j}\left({{\boldsymbol{\varphi}}},{\boldsymbol{\psi}}; \mathcal{X}_t\right) \coloneqq V\left({{\boldsymbol{\varphi}}},{\boldsymbol{\psi}}; \mathcal{X}_t\right) - \delta_{\boldsymbol{x}_j \in \mathcal{X}_t} \frac{\epsilon}{\left\vert {\mathcal{X}_t} \right\vert} f\left(D\left(\boldsymbol{\psi},\boldsymbol{x}_j\right)\right) ,
\end{equation}
where \(\delta_{\boldsymbol{x}_j \in \mathcal{X}_t}\) is the Kronecker delta, which is 1 if \(\boldsymbol{x}_j \in \mathcal{X}_t\) and 0 otherwise.
The counterfactual ASGD starts from \({{\boldsymbol{\theta}}}^{(0)}_{-j} = {{\boldsymbol{\theta}}}^{(0)}\) and updates the parameters at each step \(t\) as follows:
\begin{gather}
    {{\boldsymbol{\theta}}}^{(t+1)}_{-j} = {{\boldsymbol{\theta}}}^{(t)}_{-j} - \eta {\boldsymbol{v}}_{-j}\left({{\boldsymbol{\theta}}}^{(t)}_{-j}; \mathcal{X}_t\right), \label{eq:update_cf_mini}  \\
    \text{where}\quad {\boldsymbol{v}}_{-j}\left({{\boldsymbol{\theta}}}; \mathcal{X}_t\right) \coloneqq \begin{pmatrix}
        {\nabla_{\boldsymbol{\varphi}}} V_{-j}\left({{\boldsymbol{\varphi}}}, {\boldsymbol{\psi}}; \mathcal{X}_t\right) \\ -{\nabla_{\boldsymbol{\psi}}} V_{-j}\left({{\boldsymbol{\varphi}}}, {\boldsymbol{\psi}}; \mathcal{X}_t\right)
    \end{pmatrix}.
\end{gather}
We define the influence on parameters in the minibatch setting similarly to the full-batch setting. Let \(\Delta {{\boldsymbol{\theta}}}^{(t)}_{-j} \coloneqq  {{\boldsymbol{\theta}}}^{(t)}_{-j}-{{\boldsymbol{\theta}}}^{(t)}\) be the changes in the concatenated parameter at the \(t\)-th step of ASGD. We aim to estimate \(\Delta {{\boldsymbol{\theta}}}^{(T)}_{-j}\) at the final step \(T\).
To apply the linear approximation, we introduce an interpolated gradient between \({\boldsymbol{v}}\left({{\boldsymbol{\theta}}}; \mathcal{X}_t\right)\) and \({\boldsymbol{v}}_{-j}\left({{\boldsymbol{\theta}}}; \mathcal{X}_t\right)\) using \(\epsilon \in [0,1]\):
\begin{align*}
    {\boldsymbol{v}}_{-j,\epsilon}\left({{\boldsymbol{\theta}}}; \mathcal{X}_t\right) &= (1-\epsilon){\boldsymbol{v}}\left({{\boldsymbol{\theta}}}; \mathcal{X}_t\right) + \epsilon {\boldsymbol{v}}_{-j}\left({{\boldsymbol{\theta}}}; \mathcal{X}_t\right) \\
    &= {\boldsymbol{v}}\left({{\boldsymbol{\theta}}}; \mathcal{X}_t\right) +\delta_{\boldsymbol{x}_j \in \mathcal{X}_t} \frac{\epsilon}{\left\vert {\mathcal{X}_t} \right\vert} {\nabla_{\boldsymbol{\theta}}} f\left(D\left(\boldsymbol{\psi},\boldsymbol{x}_j\right)\right) .
\end{align*}
The linear approximation of \({\boldsymbol{v}}_{-j,1}\left({{\boldsymbol{\theta}}}^{(t)}_{-j}; \mathcal{X}_t\right)\) around \(\epsilon=0\) and \({{\boldsymbol{\theta}}}={{\boldsymbol{\theta}}}^{(t)}\) gives the following relation:
\begin{dmath*}
    {\boldsymbol{v}}_{-j}\left({{\boldsymbol{\theta}}}^{(t)}_{-j}; \mathcal{X}_t\right) -{\boldsymbol{v}}\left({{\boldsymbol{\theta}}}^{(t)}; \mathcal{X}_t\right) \approx {\boldsymbol{J}}\left({{\boldsymbol{\theta}}}^{(t)}; \mathcal{X}_t\right) \Delta {{\boldsymbol{\theta}}}^{(t)}_{-j} + \delta_{\boldsymbol{x}_j \in \mathcal{X}_t} \frac{1}{\left\vert {\mathcal{X}_t} \right\vert} {\nabla_{\boldsymbol{\theta}}} f\left(D\left(\boldsymbol{\psi},\boldsymbol{x}_j\right)\right) ,
\end{dmath*}
where \({\boldsymbol{J}}\left({{\boldsymbol{\theta}}}; \mathcal{X}_t\right) \coloneqq {\partial_{\boldsymbol{\theta}}}{\boldsymbol{v}}\left({{\boldsymbol{\theta}}}; \mathcal{X}_t\right)\).
By using this relation and subtracting \eqref{eq:update_mini} from \eqref{eq:update_cf_mini}, we have
\begin{align}
\label{eq:recur_asgd_mini}
\Delta {{\boldsymbol{\theta}}}^{(t+1)}_{-j} &= \Delta {{\boldsymbol{\theta}}}^{(t)}_{-j} - \eta \left({\boldsymbol{v}}_{-j}\left({{\boldsymbol{\theta}}}^{(t)}_{-j}; \mathcal{X}_t\right) - {\boldsymbol{v}}\left({{\boldsymbol{\theta}}}^{(t)}; \mathcal{X}_t\right)\right) \notag \\
&\approx \left( {\boldsymbol{I}} - \eta {\boldsymbol{J}}\left({{\boldsymbol{\theta}}}^{(t)}; \mathcal{X}_t\right) \right) \Delta {{\boldsymbol{\theta}}}^{(t)}_{-j} + \Delta {\boldsymbol{v}}_{-j}\left({{\boldsymbol{\theta}}}^{(t)}; \mathcal{X}_t\right),
\end{align}
where \(\Delta {\boldsymbol{v}}_{-j}\left({{\boldsymbol{\theta}}}; \mathcal{X}_t\right) \coloneqq - \delta_{\boldsymbol{x}_j \in \mathcal{X}_t} \frac{\eta}{\left\vert {\mathcal{X}_t} \right\vert} {\nabla_{\boldsymbol{\theta}}} f\left(D\left(\boldsymbol{\psi}, \boldsymbol{x}_j\right)\right)\).
By recursively applying \eqref{eq:recur_asgd_mini} from \(\Delta {{\boldsymbol{\theta}}}^{(0)}_{-j} = \boldsymbol{0}\), we obtain the influence estimator \(\widehat{\Delta{{\boldsymbol{\theta}}}_{-j}} \approx \Delta{{\boldsymbol{\theta}}}^{(T)}_{-j}\) as:
\begin{equation}
\label{eq:estimator_itd_mini}
\widehat{\Delta{{\boldsymbol{\theta}}}_{-j}} \coloneqq \sum_{t=0}^{T-1} \left(\prod_{s=t+1}^{T-1} {\boldsymbol{Z}}\left({{\boldsymbol{\theta}}}^{(s)}; \mathcal{X}_s\right)\right) \Delta {\boldsymbol{v}}_{-j}\left({{\boldsymbol{\theta}}}^{(t)}; \mathcal{X}_t\right),
\end{equation}
where \({\boldsymbol{Z}}\left({{\boldsymbol{\theta}}}; \mathcal{X}_t\right) \coloneqq {\boldsymbol{I}} - \eta {\boldsymbol{J}}\left({{\boldsymbol{\theta}}}; \mathcal{X}_t\right)\) and \(\prod\) denotes the product operation with the multiplication order \(\prod_{t=0}^{T-1} \boldsymbol{A}_t = \boldsymbol{A}_{T-1} \cdots \boldsymbol{A}_0\).

\subsection{Momentum-based Optimizer} 
\noindent We redefine the loss function and gradient where the weight of the \(j\)-th instance is scaled by \(\epsilon \in [0,1]\) as
\begin{equation*}
\ \ V_{-j,\epsilon }(\boldsymbol{\varphi } ,\boldsymbol{\psi }) \coloneq V(\boldsymbol{\varphi } ,\boldsymbol{\psi }) -\frac{\epsilon }{| \mathcal{X}| } f( D(\boldsymbol{\psi } ,\boldsymbol{x}_{j})),
\end{equation*}
defining the concatenated gradient:
\begin{equation*}
\boldsymbol{v}_{-j,\epsilon }(\boldsymbol{\theta }) \coloneq \begin{pmatrix}
\nabla _{\boldsymbol{\varphi }} V_{-j,\epsilon }(\boldsymbol{\varphi } ,\boldsymbol{\psi })\\
-\nabla _{\boldsymbol{\psi }} V_{-j,\epsilon }(\boldsymbol{\varphi } ,\boldsymbol{\psi })
\end{pmatrix}
\end{equation*}
Now, we consider the following update rule of Adversarial-RMSProp
\begin{equation*}
\boldsymbol{\theta }^{(t+1)} =\boldsymbol{\theta }^{(t)} -\eta \tilde{\boldsymbol{v}}_{-j,0}^{( t)}\left(\boldsymbol{\theta }^{(t)}\right)
\end{equation*}
where, 
\begin{align*}
\tilde{\boldsymbol{v}}_{-j,\epsilon }^{( t)}(\boldsymbol{\theta }) &\coloneq \frac{\boldsymbol{v}_{-j,\epsilon }(\boldsymbol{\theta })}{\sqrt{\boldsymbol{q}_{-j,\epsilon }^{( t)}(\boldsymbol{\theta })} +\delta }, \\
\boldsymbol{q}_{-j,\epsilon }^{( t)}(\boldsymbol{\theta }) &\coloneq \alpha \boldsymbol{q}_{t}^{( t)} +( 1-\alpha )\boldsymbol{v}_{-j,\epsilon }(\boldsymbol{\theta })^{2},
\end{align*}
Then we can define the counterfactual Adversarial-RMSProp as
\begin{equation*}
\boldsymbol{\theta }_{-j}^{(t+1)} =\boldsymbol{\theta }_{-j}^{(t)} -\eta \tilde{\boldsymbol{v}}_{-j,1}^{( t)}\left(\boldsymbol{\theta }_{-j}^{(t)}\right).
\end{equation*}
Then we can approximate the influence on parameters of a single step by
\begin{align*}
\tilde{\boldsymbol{v}}_{-j,1}^{( t)}\left(\boldsymbol{\theta }_{-j}^{(t)}\right) -\tilde{\boldsymbol{v}}_{-j,0}^{( t)}\left(\boldsymbol{\theta }^{(t)}\right) \approx \tilde{\boldsymbol{J}}^{( t)} \Delta \boldsymbol{\theta }_{-j}^{(t)} +\Delta \tilde{\boldsymbol{v}}_{-j}^{( t)} ,
\end{align*}
where \(\tilde{\boldsymbol{J}}^{( t)} \coloneq \partial _{\boldsymbol{\theta }}\tilde{\boldsymbol{v}}_{-j,0}^{( t)}\left(\boldsymbol{\theta }^{(t)}\right)\) and \(\Delta \tilde{\boldsymbol{v}}_{-j}^{( t)} :=\partial_\epsilon \tilde{\boldsymbol{v}}_{-j,0}^{( t)}\left(\boldsymbol{\theta }^{(t)}\right)\).
By using the expressions above and letting \(\tilde{\boldsymbol{Z}}^{( t)} \coloneq \boldsymbol{I} -\eta \tilde{\boldsymbol{J}}^{( t)}\), we finally obtain the ITD influence estimator as
\begin{equation*}
\widehat{\Delta \boldsymbol{\theta }_{-j}} \coloneq \sum _{t=0}^{T-1}\left(\prod _{s=t+1}^{T-1}\tilde{\boldsymbol{Z}}^{( s)}\right) \Delta \tilde{\boldsymbol{v}}_{-j}^{( t)}.
\end{equation*}
Here, \(\partial_\epsilon \tilde{\boldsymbol{v}}_{-j,0}^{( t)}\left(\boldsymbol{\theta }^{(t)}\right)\) requires its recursive derivation through \(\boldsymbol{q}_{t}^{( t)} \).
Such a derivative is important to trace how the removal of \(j\)-th instance changes the momentum of RMSProp at future steps.
Although it is possible to trace such an effect, it requires additional computational overhead.
However, we found that setting \(\partial_\epsilon \boldsymbol{q}_{t}^{( t)}\) to be zero still yields sufficiently accurate influence in practice, which we did in our experiment on StyleGAN in \cref{sec:exp2}.

\subsection{Moving Averaged Generator}
\noindent The moving average technique for parameter averaging in GAN training computes the time-average of the parameters, providing a more stable convergence by smoothing out fluctuations over time. 
As a common practice, \cite{stylegan} utilizes the exponential moving average, which computes an exponentially discounted sum of the parameters using the following update rule:
\begin{equation*}
\overline{\boldsymbol{\varphi }}^{(t+1)} =(1-\beta )\overline{\boldsymbol{\varphi }}^{(t)} +\beta \boldsymbol{\varphi }^{(t+1)},
\end{equation*}
where  \(\overline{\boldsymbol{\theta }}^{(0)} = \boldsymbol{\theta }^{(0)}\) and \(\beta\) is the smoothing factor \(0 < \beta < 1\).
To apply the ITD influence estimator to the averaged generator, we need to consider how the removal of the \(i\)-th instance affects the final averaged generator.
The ITD influence estimator accounts for the influence of removing a training instance by approximating the changes in the parameters over time. 
It starts by expressing the averaged parameters with and without the \(i\)-th instance as:
\begin{align*}
\overline{\boldsymbol{\theta }}^{(t+1)} &=(\boldsymbol{I} -\boldsymbol{B} )\overline{\boldsymbol{\theta }}^{(t)} +\boldsymbol{B\theta }^{(t+1)}, \\
\overline{\boldsymbol{\theta }}_{-j}^{(t+1)} &=(\boldsymbol{I} -\boldsymbol{B} )\overline{\boldsymbol{\theta }}_{-j}^{(t)} +\boldsymbol{B\theta }_{-j}^{(t+1)}
\end{align*}
where \(\boldsymbol{B} =\begin{pmatrix} \boldsymbol{I} & \boldsymbol{O}\\ \boldsymbol{O} & \beta\boldsymbol{I} \end{pmatrix}\) and $\boldsymbol{\theta }^{(t)}\coloneq \left(\overline{\boldsymbol{\varphi }}{^{(t)\intercal} }, \boldsymbol{\psi}^{(t)\intercal }\right)$.
The matrix \(\boldsymbol{B}\) is introduced to account for the fact that the discriminator's parameter, represented by \(\boldsymbol{\psi}^{(t)}\), is not updated using the moving average. 
Thus, the update rule for the discriminator's parameter follows the original definition of the AGD and Counterfactual AGD.
The difference between the parameters with and without the \(i\)-th instance at each step \(t\) can be defined as:
\begin{align*}
\Delta \overline{\boldsymbol{\theta }}_{-j}^{(t)} \coloneq \overline{\boldsymbol{\theta }}_{-j}^{(t)} -\overline{\boldsymbol{\theta }}^{(t)} & =(\boldsymbol{I} -\boldsymbol{B} )\Delta \overline{\boldsymbol{\theta }}_{-j}^{(t-1)} +\boldsymbol{B} \Delta \boldsymbol{\theta }_{-j}^{(t)}\\
 & \approx (\boldsymbol{I} -\boldsymbol{B} )\Delta \overline{\boldsymbol{\theta }}_{-j}^{(t-1)} +\boldsymbol{B}\widehat{\Delta \boldsymbol{\theta }_{-j}^{(t)}} ,
\end{align*}
where \(\widehat{\Delta \boldsymbol{\theta }_{-j}^{(t)}}\) is the slightly modified notation of the orignal ITD influence estimator for AGD \cref{eq:estimator_itd} that only approximate the parameter changes before the \(t\)-th step.
By summing over all time steps, we obtain:
\begin{align*}
\Delta\overline{ \boldsymbol{\theta }}_{-j}^{(T)} & =\sum _{\tau =0}^{T-1} (\boldsymbol{I} -\boldsymbol{B} )\boldsymbol{B}^{(T-(\tau +1))} \Delta \boldsymbol{\theta}_{-j}^{(k)}\\
 & \approx \sum _{\tau =0}^{T-1} (\boldsymbol{I} -\boldsymbol{B} )\boldsymbol{B}^{(T-(\tau +1))}\widehat{\Delta \boldsymbol{\theta }_{-j}^{(\tau )}}\eqqcolon \widehat{\Delta \overline{\boldsymbol{\theta }}_{-j}^{(t)}}.
\end{align*}
This results in the following expression for the change in the loss function due to the removal of the \(i\)-th instance:
\begin{align*}
E\left(\mathcal{X}_{G}^{\prime}\left(\boldsymbol{\varphi}^{(T)}_{-j}\right)\right) -  E\left(\mathcal{X}_{G}^{\prime}\left(\boldsymbol{\varphi}^{(T)}\right)\right) &\approx \widehat{\Delta \overline{\boldsymbol{\theta }}_{-j}^{(t)\intercal}} \nabla \overline{E}\\
&\approx \sum _{\tau =0}^{T-1} (\boldsymbol{I} -\boldsymbol{B} )\boldsymbol{B}^{(T-(\tau +1))}\sum _{k=0}^{\tau } \Delta \boldsymbol{v}_{-j}\left(\boldsymbol{\theta }^{(k)}\right)^{\intercal }\left(\prod _{s=k+1}^{\tau }\boldsymbol{Z}\left(\boldsymbol{\theta }^{(s)}\right)\right)^{\intercal } \nabla \overline{E} \\
&\eqqcolon \widehat{\Delta \overline{E}_{-j}},
\end{align*}
where ${\displaystyle \nabla \overline{E} \coloneq (\nabla _{\boldsymbol{\varphi }} E(\mathcal{X}_{G}^{\prime } (\overline{\boldsymbol{\varphi }}^{(T)} ))^{\intercal } ,\ \ \boldsymbol{0}^{\intercal } )^{\intercal }}$.
To further simplify the recursive computation of the influence estimator, we introduce $\overline{\boldsymbol{u}}^{(t)}$:
\begin{align*}
\overline{\boldsymbol{u}}^{(t)} \coloneq \left(\sum _{k=t}^{T-1} (\boldsymbol{I -B} )\boldsymbol{B}^{(T-(k+1))}\prod _{s=t+1}^{T-1}\boldsymbol{Z}\left(\boldsymbol{\theta }^{(s)}\right)\right)^{\intercal } \nabla \overline{E}
\end{align*}
Using this recursive computation, we can express the influence on the parameter as follows:
\begin{align*}
\overline{\boldsymbol{u}}^{(t-1)} & =\left(\sum _{\tau =t-1}^{T-1} (\boldsymbol{I -B} )\boldsymbol{B}^{(T-(\tau +1))}\prod _{s=t}^{\tau }\boldsymbol{Z}\left(\boldsymbol{\theta }^{(s)}\right)\right)^{\intercal } \nabla \overline{E}\\
 & =\left(\sum _{\tau =t}^{T-1} (\boldsymbol{I -B} )\boldsymbol{B}^{(T-(\tau +1))}\prod _{s=t}^{\tau }\boldsymbol{Z}\left(\boldsymbol{\theta }^{(s)}\right)\right) +(\boldsymbol{I -B} )\boldsymbol{B}^{(T-t))}\prod _{s=t}^{t-1}\boldsymbol{Z}\left(\boldsymbol{\theta }^{(s)}\right)^{\intercal } \nabla \overline{E}\\
 & =\boldsymbol{Z}\left(\boldsymbol{\theta }^{(t)}\right)\left(\sum _{\tau =t}^{T-1} (\boldsymbol{I -B} )\boldsymbol{B}^{(T-(\tau +1))}\prod _{s=t+1}^{\tau }\boldsymbol{Z}\left(\boldsymbol{\theta }^{(s)}\right)\right) +(\boldsymbol{I -B} )\boldsymbol{B}^{(T-t)} \nabla \overline{E}\\
 & =\boldsymbol{Z}\left(\boldsymbol{\theta }^{(t)}\right)\overline{\boldsymbol{u}}^{(t)}  +(\boldsymbol{I -B} )\boldsymbol{B}^{(T-t))} \nabla \overline{E}
\end{align*}
Then, the estimation of influence on evaluation metric at the \(t\)-th step can also be computed recursively as 
\begin{align*}
\widehat{\Delta\overline{E}_{-j}^{(t-1)}} & \coloneq \sum _{\tau =t}^{T-1} (\boldsymbol{I -B} )\boldsymbol{B}^{(T-(\tau +1))}\sum _{k=t}^{\tau } \Delta \boldsymbol{v}_{-j}\left(\boldsymbol{\theta }^{(k)}\right)^{\intercal }\left(\prod _{s=k+1}^{\tau }\boldsymbol{Z}\left(\boldsymbol{\theta }^{(s)}\right)\right)^{\intercal } \nabla \overline{E}\\
 & =\sum _{k=t}^{T-1} \Delta \boldsymbol{v}_{-j}\left(\boldsymbol{\theta }^{(k)}\right)^{\intercal }\sum _{\tau =k}^{T-1} (\boldsymbol{I -B} )\boldsymbol{B}^{(T-(\tau +1))}\left(\prod _{s=k+1}^{\tau }\boldsymbol{Z}\left(\boldsymbol{\theta }^{(s)}\right)\right)^{\intercal } \nabla \overline{E}\\
 & =\sum _{k=t}^{T-1} \Delta \boldsymbol{v}_{-j}\left(\boldsymbol{\theta }^{(k)}\right)^{\intercal }\boldsymbol{u}^{(k)}\\
 & =\widehat{\Delta E_{-j}^{(t)}} +\Delta \boldsymbol{v}_{-j}\left(\boldsymbol{\theta }^{(t)}\right)^{\intercal }\boldsymbol{u}^{(t)}.
\end{align*}
The overall influence can then be written as $\widehat{\Delta \overline{E}_{-j}}=\widehat{\Delta\overline{E}_{-j}^{(-1)}} $.
This recursive approach allows us to efficiently compute the influence estimator $\widehat{\Delta E_{-j}^{(t)}}$ for all $t$ from $T-1$ to $0$ using the derived recursive relations and initial conditions.

\section{Detailed Settings and Results of Experiments} \label{app:exp}

\subsection{GAN Evaluation Metrics}
\noindent In our experiments, we used three GAN evaluation metrics: average log-likelihood (ALL), inception score (IS), and Fréchet inception distance (FID).
    
ALL is the de-facto standard for evaluating generative models \cite{adagan}.
ALL measures the likelihood of the true data under the distribution that is estimated from generated data using kernel density estimation.
We calculated ALL using the validation dataset under the distribution estimated from the generated instances.
We adopted Gaussian kernel with the bandwidth 1 for kernel density estimation used in ALL.

The empirical version of IS has a form of  \(E (\mathcal{X}^\prime)  = \mathrm{exp}(\frac{1}{\vert \mathcal{X}^\prime \vert} \sum_{\boldsymbol{x} \in\mathcal{\mathcal{X}^\prime}}\mathbb{KL}(p_c(y \,\vert\,\boldsymbol{x})\,\Vert \,p_c(y))\), where \(p_c\) is a distribution of class label \(y\) drawn by a pre-trained learning classifier
      
FID measures Fréchet distance between two sets of feature vectors of real images and those of generated images.
The feature vectors are calculated on the basis of a pre-trained classifier.

Larger values of ALL and IS and a smaller value of FID indicate better generative performance.

To compute IS and FID, we trained a CNN classifier of MNIST with a validation dataset, whose architecture can be found in \cref{table:cnn}.
We selected the output of the 4th layer for the feature vectors for FID.

\subsection{Experiment 1: Estimation Accuracy} \label{app:exp1}

\subsubsection{LQGAN Trained for 1D-Normal}
We used \(x\sim\mathcal{N}(1,1)\) to construct the 1D-Normal training dataset \(\mathcal{X}\) with 1,000 instances for AGD training and the validation dataset with 1,000 instances for computing ALL. 
We also sampled 1,000 latent variables from \(z\sim\mathcal{N}(0,1)\) to construct \(\mathcal{Z}\).
Both the learning rate in \cref{eq:update} and the scaling coefficient for AID-EIGEM in \cref{eq:stationary_point} were set to \(0.01\).

\subsubsection{DCGAN Trained for MNIST}
For MNIST, we randomly selected 10,000 instances for AGD training and 10,000 validation instances for computing IS. 
DCGAN consists of transposed convolution (or deconvolution) layers and convolution layers (Table~\ref{table:dcgan}). 
We used Layer Normalization \cite{ba2016layer} for the layers shown in Table~\ref{table:dcgan} for the stability of the training.
In this experiment, we set \(8\) as the number of channels \(h_G\) and \(h_D\) in \cref{table:cnn}.
We also introduced the L2-norm regularization with a rate \(10^{-3}\) for all the layers.
We used the non-zero-sum game objective of the original paper \cite{goodfellow2014generative} for training stability.
In addition, both gradient descent \cref{eq:update} and \cref{alg:lie_itd} were performed in stochastic manner using the minibatch with 100 samples.
The learning rate was set to \(0.001\).
We also used the regularization with \(\gamma=0.1\) for AID-EIGEM for the stability of the recursive computation~\cref{sec:stability}.

\subsubsection{Results}
\cref{table:tau_toy,table:tau_is} show the complete result of \cref{fig:valid}\subref{sub:toy_valid} and \cref{fig:valid}\subref{sub:mnist_valid_is}, respectively.

\subsection{Experiment 2: Data Cleansing}\label{app:exp2}

\paragraph{LQGAN Trained for 1D-Normal}
We used \(x \sim b \mathcal{N}(1, 0.5) + (1-b) \mathcal{N}(-2, 0.5) \) with \(b \sim \mathrm{Bernoulli}(0.95)\) to construct the 1D-Normal training dataset \(\mathcal{X}\) with 1,000 instances.
We separately sampled 1,000 instances to construct the validation and test dataset from \(x \sim \mathcal{N}(1, 0.5)\).
The validation dataset is used to compute the influence on ALL and the test dataset is used to evaluate the test ALL after the data cleansing.
The AGD training adopted \(T=10000\).
The scaling coefficient for AID-EIGEM and the learning rate follow the setting of \cref{sec:exp1}.
We adopted the same architecture as \cref{sec:exp1}.

\paragraph{DCGAN Trained for MNIST}
For MNIST, we randomly selected 50,000 instances for AGD training and 10,000 validation instances for computing IS and FID.
The test dataset consists of 10,000 instances which are exclusive from the training and validation dataset.
The architecture of DCGAN followed (\cref{table:dcgan}) in which \(h_G=32\) and \(h_D=32\).
We also introduced the L2-norm regularization with a rate \(10^{-3}\) for all the layers.
DCGAN was trained by 10000 steps of the stochastic gradient descent with a learning rate \(0.001\).
The other settings followed those of \cref{sec:exp2} except for introducing the regularization with \(\gamma=0.1\) for ITD-EIGEM in the full-epoch retraining setting.

\paragraph{StyleGAN Fine-tuned for Animal Faces-HQ}
We conducted experiments on StyleGAN \cite{stylegan} using 5,558 cat images from the Animal Faces-HQ \cite{afhq} dataset, split into 80
Images were resized to 256x256 pixels.

We adopted a PyTorch implementation\footnote{\url{https://github.com/rosinality/style-based-gan-pytorch}} that replicates the original StyleGAN architecture. 
The pre-trained model on the FFHQ dataset \cite{stylegan} was also from the same repository. 
Our fine-tuning applied LoRA \cite{hu2021lora} to both the generator and discriminator, with a rank of 16 for both the generator's progression blocks and RGB layers, as well as the discriminator's convolutional blocks and linear layer. 
The RMSProp optimizer was used with learning rates of 0.002 for the LoRA parameters. 
The generator's LoRA parameters were updated using moving averaging with a decay factor of 0.999. 
Training was conducted for 50 epochs with a batch size of 8, and a gradient penalty was applied.

FID was used as the evaluation metric for both influence estimation and model evaluation. 
We computed activations from the pool-3 layer of pre-trained InceptionV3 \cite{inceptionnet} for FID, density, and coverage computations. 
To ensure the covariance matrix for FID was full rank, we augmented validation and test instances with horizontal flipping. 
Test activations from InceptionV3 were also used to evaluate the density and coverage shown in \cref{table:density_coverage}.

For instance selection approaches, we employed our ITD-EIGEM and AID-EIGEM methods, as well as isolation forest and random selection as baselines. 
ITD-EIGEM performed iterations as explained in \cref{app:itd_ext}. 
In the one-epoch retraining, ITD-EIGEM computed influence by only tracing back the training iterations in the last epoch. 
AID-EIGEM was applied to the final discriminator and the averaged generator, with parameters \(M=1000\) and \(\eta=0.001\). 
Both ITD-EIGEM and AID-EIGEM computed the influence on FID evaluated on the validation dataset. 
The isolation forest scored harmfulness using InceptionV3 activations from the validation dataset.

Counterfactual training was performed by removing harmful instances identified through influence scores, while varying removal rates from 0.001 to 0.9. 
For each removal rate, we retrained the model from the initial epoch for the full-epoch retraining and from the 49th epoch for the one-epoch retraining, and evaluated the model on the test dataset. 
During counterfactual training, we reproduced the original training's randomness, including noise input and style mixing step indices.

\subsubsection{Detailed Results}
\cref{table:ll,table:fid,table:inception_score} show the detailed results of the data cleansing for LQGAN and MNIST, which the statistical information excluded from \cref{fig:clean} for visibility.

Regarding the data cleansing for StyleGAN, the supplementary figures \cref{fig:instance_comparison}, \cref{fig:visual_cat_app1}, and \cref{fig:visual_cat_app2} further support the observations discussed in \cref{sec:quality}.
\cref{fig:instance_comparison} shows harmful and helpful instances identified by ITD-EIGEM with full-epoch and one-epoch tracking, as well as those identified by AID-EIGEM and the isolation forest.
Recalling our observation in \cref{sec:quality}, harmful instances predicted by ITD-EIGEM (\cref{fig:instance_comparison}\subref{sub:fid_itd_helpful_50epoch}) show common patterns like yellow cats with stripes, while helpful instances (\cref{fig:instance_comparison}\subref{sub:fid_itd_helpful_50epoch}) show rare patterns like cats without stripes or seal point cats.
In contrast, the harmful and helpful instances predicted by AID-EIGEM (\cref{fig:instance_comparison}\subref{sub:fid_aid_harmful_50epoch}-\subref{sub:fid_aid_helpful_50epoch}) and the isolation forest (\cref{fig:instance_comparison}\subref{sub:if_harmful_50epoch}-\subref{sub:if_helpful_50epoch}) do not show such a clear tendency in their patterns.
The generation results in \cref{fig:visual_cat_app1} and \cref{fig:visual_cat_app2} further illustrate the effect of data cleansing with a larger number of samples.
As noted in \cref{sec:quality}, our cleansed model by ITD-EIGEM seems to have reassigned latent variables originally associated with common patterns, such as yellow cats with stripes, to rare patterns, such as cats without stripes or seal point cats. 
This tendency is consistently observed across other generated samples shown in \cref{fig:visual_cat_app1} and \cref{fig:visual_cat_app2}.
In other methods, such reassignments are not or only partially observed.

\begin{table}[t]
\caption{Model Architecture of CNN Classifier of MNIST in Section \ref{sec:exp1} and \ref{sec:exp2}}
\label{table:cnn}
\begin{center}
    \begin{tabular}{cccccccc}
    \toprule
    Stage & Operation & Stride & Filter Shape & Bias & Norm. & Activation & Output \\
    \midrule
    0 & Input & -& -& - &  -&- & [28, 28, 1]  \\
    1 & Conv2D &1 & [5, 5]& \checkmark& - & Sigmoid & [25, 25, 8] \\
    2 & Conv2D &1 & [5, 5]& \checkmark& - & Sigmoid & [12, 12, 8] \\
    3 & MaxPooling &2 & [2, 2] & - &- & Sigmoid & [392] \\
    4 & Linear &1 & - & \checkmark& -& Sigmoid & [128] \\
    5 & Linear &1 & - & \checkmark& -& Sigmoid & [10] \\
    \bottomrule
    \end{tabular}
\end{center}
\end{table}

\begin{table}[t]
\caption{Model Architecture of DCGAN in Section \ref{sec:exp1} and \ref{sec:exp2}}
\label{table:dcgan}
\begin{center}
    \begin{tabular}{ccccccccc}
    \toprule
    Net. & Stage & Operation & Stride & Filter Shape & Bias & Norm. & Activation & Output \\
    \midrule
    - & 0 & Input & -& -& - &  -&- &[32]  \\
    \(G\)&1 & Deconv2D &1 & [2, 2] &\checkmark&\checkmark & Sigmoid & [2, 2, \(h_G\)] \\
    \(G\)&2 & Deconv2D &1 & [3, 3]& \checkmark&\checkmark & Sigmoid & [4, 4, \(h_G\)] \\
    \(G\)&3 & Deconv2D &2 & [3, 3]& \checkmark&\checkmark & Sigmoid & [9, 9, \(h_G\)] \\
    \(G\)&4 & Deconv2D &1 & [2, 2]& \checkmark&\checkmark & Sigmoid & [10, 10, \(h_G\)] \\
    \(G\)&5 & Deconv2D &1 & [3, 3]& \checkmark&\checkmark & Sigmoid & [12, 12, \(h_G\)] \\
    \(G\)&6 & Deconv2D &2 & [3, 3]& \checkmark&\checkmark & Sigmoid & [25, 25, \(h_G\)] \\
    \(G\)&7 & Deconv2D &1 & [4, 4]& \checkmark&\checkmark & Sigmoid & [28, 28, \(h_G\)] \\
    \(G\)&8 & Conv2D &1 & [1, 1]& \checkmark&- & Tanh & [28, 28, 1] \\
    \(D\)&9 & Conv2D &1 & [4, 4]& \checkmark&\checkmark & Sigmoid & [25, 25, \(h_D\)] \\
    \(D\)&10 & Conv2D &2 & [3, 3]& \checkmark&\checkmark & Sigmoid & [12, 12, \(h_D\)] \\
    \(D\)&11 & Conv2D &1 & [3, 3]& \checkmark&\checkmark & Sigmoid & [10, 10, \(h_D\)] \\
    \(D\)&12& Conv2D &1 & [2, 2]& \checkmark&\checkmark & Sigmoid & [9, 9, \(h_D\)] \\
    \(D\)&13 & Conv2D &2 & [3, 3]& \checkmark&\checkmark & Sigmoid & [4, 4, \(h_D\)] \\
    \(D\)&14& Conv2D &1 & [3, 3] &\checkmark&\checkmark & Sigmoid & [2, 2, \(h_D\)] \\
    \(D\)&15& Conv2D &1 & [2, 2] &\checkmark&\checkmark & Sigmoid & [1, 1, \(h_D\)] \\
    \(D\)&16& Linear &- & - &\checkmark&-& Sigmoid & [1] \\
    \bottomrule
    \end{tabular}
\end{center}
\end{table}

\begin{table}[t]
\centering
\caption{Average Kendal's Tau (10\% percentile) (90\% percentile) of Estimated and True Influence on ALL Computed Over Randomly Selected 100 Training Instances (Bold Numbers Express Statistically Significantly Larger Values Than Random with \(p < .05\))}
\label{table:tau_toy}
\begin{tabular}{ccccccccccccc}
\toprule
 & \multicolumn{12}{c}{\(T\)}\\ \cmidrule(lr){2-13}
&                            2     &                            5     &                            10    &                            20    &                            50    &                            100   &                            200   &                            500   &                            1000  &                            2000  &                            5000  &                            10000 \\
\midrule
\thead{AID \\ (\(M\)=10)}    &  \textbf{\thead{1.00 \\ (1.00)  \\ (1.00)}} &  \textbf{\thead{1.00 \\ (1.00)  \\ (1.00)}} &   \textbf{\thead{1.00 \\ (1.00)  \\ (1.00)}} &   \textbf{\thead{1.00 \\ (0.99)  \\ (1.00)}} &  \textbf{\thead{0.94 \\ (0.81)  \\ (1.00)}} &  \textbf{\thead{0.88 \\ (0.57)  \\ (1.00)}} &   \textbf{\thead{0.71 \\ (0.42)  \\ (0.99)}} &  \textbf{\thead{0.37 \\ (0.14)  \\ (0.51)}} &  \textbf{\thead{0.27 \\ (-0.00)  \\ (0.51)}} &            \thead{0.16 \\ (-0.21) \\ (0.70)} &  \textbf{\thead{0.11 \\ (-0.04)  \\ (0.31)}} &  \textbf{\thead{0.18 \\ (0.02)  \\ (0.33)}} \\
\thead{AID \\ (\(M\)=100)}   &  \textbf{\thead{0.99 \\ (0.99)  \\ (1.00)}} &  \textbf{\thead{0.99 \\ (0.99)  \\ (1.00)}} &   \textbf{\thead{0.99 \\ (0.99)  \\ (1.00)}} &   \textbf{\thead{0.98 \\ (0.93)  \\ (1.00)}} &  \textbf{\thead{0.93 \\ (0.76)  \\ (1.00)}} &  \textbf{\thead{0.88 \\ (0.55)  \\ (1.00)}} &   \textbf{\thead{0.77 \\ (0.58)  \\ (1.00)}} &  \textbf{\thead{0.65 \\ (0.54)  \\ (0.74)}} &   \textbf{\thead{0.38 \\ (0.17)  \\ (0.55)}} &  \textbf{\thead{0.34 \\ (-0.00)  \\ (0.79)}} &   \textbf{\thead{0.29 \\ (0.21)  \\ (0.43)}} &  \textbf{\thead{0.34 \\ (0.23)  \\ (0.44)}} \\
\thead{AID \\ (\(M\)=1000)}  &  \textbf{\thead{0.87 \\ (0.54)  \\ (1.00)}} &  \textbf{\thead{0.87 \\ (0.56)  \\ (1.00)}} &   \textbf{\thead{0.86 \\ (0.60)  \\ (1.00)}} &   \textbf{\thead{0.82 \\ (0.47)  \\ (1.00)}} &  \textbf{\thead{0.53 \\ (0.19)  \\ (0.94)}} &  \textbf{\thead{0.38 \\ (0.15)  \\ (0.63)}} &  \textbf{\thead{0.37 \\ (-0.03)  \\ (0.74)}} &  \textbf{\thead{0.71 \\ (0.49)  \\ (0.89)}} &   \textbf{\thead{0.86 \\ (0.80)  \\ (0.94)}} &   \textbf{\thead{0.71 \\ (0.42)  \\ (0.91)}} &   \textbf{\thead{0.89 \\ (0.85)  \\ (0.94)}} &  \textbf{\thead{0.87 \\ (0.84)  \\ (0.90)}} \\
\thead{AID \\ (\(M\)=10000)} &  \textbf{\thead{0.75 \\ (0.64)  \\ (1.00)}} &  \textbf{\thead{0.79 \\ (0.70)  \\ (1.00)}} &  \textbf{\thead{0.60 \\ (-0.08)  \\ (1.00)}} &  \textbf{\thead{0.61 \\ (-0.50)  \\ (1.00)}} &           \thead{0.25 \\ (-0.38) \\ (0.99)} &  \textbf{\thead{0.36 \\ (0.09)  \\ (0.57)}} &   \textbf{\thead{0.43 \\ (0.10)  \\ (0.79)}} &  \textbf{\thead{0.73 \\ (0.57)  \\ (0.86)}} &   \textbf{\thead{0.89 \\ (0.80)  \\ (0.97)}} &   \textbf{\thead{0.73 \\ (0.55)  \\ (0.92)}} &   \textbf{\thead{0.99 \\ (0.98)  \\ (1.00)}} &  \textbf{\thead{1.00 \\ (1.00)  \\ (1.00)}} \\
ITD                          &  \textbf{\thead{1.00 \\ (1.00)  \\ (1.00)}} &  \textbf{\thead{1.00 \\ (1.00)  \\ (1.00)}} &   \textbf{\thead{1.00 \\ (1.00)  \\ (1.00)}} &   \textbf{\thead{1.00 \\ (1.00)  \\ (1.00)}} &  \textbf{\thead{1.00 \\ (1.00)  \\ (1.00)}} &  \textbf{\thead{1.00 \\ (1.00)  \\ (1.00)}} &   \textbf{\thead{1.00 \\ (1.00)  \\ (1.00)}} &  \textbf{\thead{1.00 \\ (1.00)  \\ (1.00)}} &   \textbf{\thead{1.00 \\ (1.00)  \\ (1.00)}} &   \textbf{\thead{1.00 \\ (0.99)  \\ (1.00)}} &   \textbf{\thead{1.00 \\ (1.00)  \\ (1.00)}} &  \textbf{\thead{1.00 \\ (1.00)  \\ (1.00)}} \\
\bottomrule
\end{tabular}
\end{table}

\begin{table}[t]
\centering
\caption{Average Kendal's Tau (10\% percentile) (90\% percentile) of Estimated and True Influence on IS Computed Over Randomly Selected 100 Training Instances (Bold Numbers Express Statistically Significantly Larger Values Than Random with \(p < .05\))}
\label{table:tau_is}
\begin{tabular}{cccccccc}
\toprule
 & \multicolumn{7}{c}{\(T\)}\\ \cmidrule(lr){2-8}
{} &                                     100   &                                     200   &                                      500   &                                     1000  &                                      2000  &                                     5000  &                             10000 \\
\midrule
\thead{AID \\ (\(M\)=10)}    &  \textbf{\thead{0.48 \\ (0.30)  \\ (0.63)}} &  \textbf{\thead{0.32 \\ (0.15)  \\ (0.48)}} &  \textbf{\thead{0.07 \\ (-0.00)  \\ (0.19)}} &           \thead{0.08 \\ (-0.04) \\ (0.20)} &           \thead{-0.04 \\ (-0.18) \\ (0.11)} &          \thead{-0.00 \\ (-0.06) \\ (0.06)} &  \thead{-0.04 \\ (-0.12) \\ (0.08)} \\
\thead{AID \\ (\(M\)=100)}   &  \textbf{\thead{0.50 \\ (0.36)  \\ (0.65)}} &  \textbf{\thead{0.32 \\ (0.15)  \\ (0.47)}} &   \textbf{\thead{0.09 \\ (0.01)  \\ (0.21)}} &           \thead{0.09 \\ (-0.05) \\ (0.22)} &           \thead{-0.03 \\ (-0.17) \\ (0.10)} &           \thead{0.01 \\ (-0.05) \\ (0.07)} &  \thead{-0.04 \\ (-0.12) \\ (0.09)} \\
\thead{AID \\ (\(M\)=1000)}  &  \textbf{\thead{0.49 \\ (0.34)  \\ (0.64)}} &  \textbf{\thead{0.32 \\ (0.15)  \\ (0.47)}} &   \textbf{\thead{0.09 \\ (0.01)  \\ (0.22)}} &           \thead{0.08 \\ (-0.07) \\ (0.23)} &           \thead{-0.03 \\ (-0.16) \\ (0.10)} &           \thead{0.01 \\ (-0.06) \\ (0.08)} &  \thead{-0.04 \\ (-0.14) \\ (0.08)} \\
\thead{AID \\ (\(M\)=10000)} &  \textbf{\thead{0.50 \\ (0.37)  \\ (0.63)}} &  \textbf{\thead{0.33 \\ (0.15)  \\ (0.50)}} &   \textbf{\thead{0.09 \\ (0.01)  \\ (0.21)}} &           \thead{0.09 \\ (-0.05) \\ (0.23)} &           \thead{-0.03 \\ (-0.17) \\ (0.09)} &           \thead{0.01 \\ (-0.06) \\ (0.08)} &  \thead{-0.04 \\ (-0.13) \\ (0.07)} \\
ITD                          &  \textbf{\thead{0.94 \\ (0.88)  \\ (0.97)}} &  \textbf{\thead{0.93 \\ (0.88)  \\ (0.97)}} &   \textbf{\thead{0.88 \\ (0.80)  \\ (0.95)}} &  \textbf{\thead{0.58 \\ (0.10)  \\ (0.92)}} &  \textbf{\thead{0.29 \\ (-0.16)  \\ (0.70)}} &  \textbf{\thead{0.20 \\ (0.01)  \\ (0.42)}} &   \thead{0.14 \\ (-0.08) \\ (0.33)} \\
\bottomrule
\end{tabular}
\end{table}

\begin{table}[t]
\centering
\caption{Improvements of Test ALL (\(\pm\)STD) after the Data Cleansing of 1D-Normal (Values are Highlighted when the Improvement is Statistically Significant with a Significant Level 0.05)}
\label{table:ll}
\begin{tabular}{lcccccc}
\toprule
{} &                              0.01 &                              0.02 &                               0.05 &                               0.10 &                               0.20 &                              0.50 \\
\midrule
Influence on ALL by ITD (Ours) &  \textbf{\thead{+3.29 \\ (0.23)}} &  \textbf{\thead{+5.90 \\ (0.34)}} &  \textbf{\thead{+10.98 \\ (0.83)}} &  \textbf{\thead{+11.39 \\ (1.00)}} &   \textbf{\thead{+9.63 \\ (1.02)}} &  \textbf{\thead{+2.49 \\ (1.22)}} \\
Influence on ALL by AID (Ours) &  \textbf{\thead{+3.06 \\ (0.88)}} &  \textbf{\thead{+5.52 \\ (1.45)}} &  \textbf{\thead{+10.16 \\ (2.89)}} &  \textbf{\thead{+10.48 \\ (3.22)}} &   \textbf{\thead{+9.65 \\ (1.05)}} &  \textbf{\thead{+2.78 \\ (0.96)}} \\
Influence on Disc. Loss by ITD &  \textbf{\thead{+2.46 \\ (1.39)}} &  \textbf{\thead{+4.31 \\ (2.60)}} &   \textbf{\thead{+7.62 \\ (4.98)}} &   \textbf{\thead{+7.45 \\ (5.57)}} &   \textbf{\thead{+5.40 \\ (5.87)}} &           \thead{-4.97 \\ (6.69)} \\
Influence on Disc. Loss by AID &  \textbf{\thead{+2.78 \\ (1.14)}} &  \textbf{\thead{+4.99 \\ (2.12)}} &   \textbf{\thead{+9.31 \\ (3.98)}} &   \textbf{\thead{+9.30 \\ (4.48)}} &   \textbf{\thead{+7.41 \\ (4.78)}} &           \thead{-0.81 \\ (5.45)} \\
Isolation Forest               &  \textbf{\thead{+2.43 \\ (0.32)}} &  \textbf{\thead{+4.00 \\ (0.57)}} &   \textbf{\thead{+8.72 \\ (1.01)}} &  \textbf{\thead{+11.48 \\ (1.03)}} &  \textbf{\thead{+10.10 \\ (1.10)}} &  \textbf{\thead{+1.26 \\ (2.26)}} \\
Random                         &           \thead{-0.03 \\ (0.23)} &           \thead{-0.11 \\ (0.23)} &            \thead{-0.46 \\ (0.25)} &            \thead{-1.07 \\ (0.39)} &            \thead{-2.12 \\ (0.53)} &           \thead{-6.61 \\ (1.09)} \\
\bottomrule
\end{tabular}
\end{table}

\begin{table}[t]
\centering
\caption{Improvements of Test Inception Score (\(\pm\)STD) after the Data Cleansing of MNIST (Values are Highlighted when the Improvement is Statistically Significant with a Significant Level 0.05)}
\label{table:inception_score}
\begin{tabular}{llcccccccccc}
\toprule
   && \multicolumn{10}{c}{Rate of Instances removed \(n_h / N_x\)} \\  \cmidrule(lr){2-11}
& {} &                     0.01 &                              0.02 &                     0.05 &                              0.10 &                     0.20 &                              0.30 &                     0.40 &                     0.50 &                     0.70 &                     0.90 \\
\midrule
 \multirow{8}{*}[-1.8cm]{Full-epoch retraining} &Influence on FID by ITD (Ours) &  \thead{+0.02 \\ (0.08)} &           \thead{+0.02 \\ (0.07)} &  \thead{+0.05 \\ (0.12)} &  \textbf{\thead{+0.07 \\ (0.14)}} &  \thead{+0.07 \\ (0.17)} &  \textbf{\thead{+0.09 \\ (0.21)}} &  \thead{+0.05 \\ (0.20)} &  \thead{+0.03 \\ (0.22)} &  \thead{-0.12 \\ (0.21)} &  \thead{-2.76 \\ (1.21)} \\
 &Influence on FID by AID (Ours) &  \thead{+0.01 \\ (0.10)} &           \thead{-0.00 \\ (0.15)} &  \thead{-0.14 \\ (0.27)} &           \thead{-0.25 \\ (0.30)} &  \thead{-0.54 \\ (0.40)} &           \thead{-0.66 \\ (0.55)} &  \thead{-0.82 \\ (0.60)} &  \thead{-0.74 \\ (0.57)} &  \thead{-0.77 \\ (0.54)} &  \thead{-2.18 \\ (0.67)} \\
 &Influence on IS by ITD (Ours)  &  \thead{+0.01 \\ (0.06)} &  \textbf{\thead{+0.03 \\ (0.05)}} &  \thead{+0.05 \\ (0.13)} &  \textbf{\thead{+0.07 \\ (0.14)}} &  \thead{+0.05 \\ (0.14)} &           \thead{+0.08 \\ (0.24)} &  \thead{+0.06 \\ (0.20)} &  \thead{+0.02 \\ (0.21)} &  \thead{-0.10 \\ (0.20)} &  \thead{-2.72 \\ (1.05)} \\
 &Influence on IS by AID (Ours)  &  \thead{+0.01 \\ (0.13)} &           \thead{-0.00 \\ (0.10)} &  \thead{-0.13 \\ (0.25)} &           \thead{-0.27 \\ (0.31)} &  \thead{-0.49 \\ (0.46)} &           \thead{-0.79 \\ (0.48)} &  \thead{-0.80 \\ (0.52)} &  \thead{-0.70 \\ (0.49)} &  \thead{-0.82 \\ (0.63)} &  \thead{-2.24 \\ (0.67)} \\
 &Influence on Disc. Loss by ITD &  \thead{+0.01 \\ (0.07)} &           \thead{-0.03 \\ (0.16)} &  \thead{+0.00 \\ (0.14)} &           \thead{-0.00 \\ (0.21)} &  \thead{+0.02 \\ (0.21)} &           \thead{+0.02 \\ (0.24)} &  \thead{-0.03 \\ (0.25)} &  \thead{-0.01 \\ (0.21)} &  \thead{-0.14 \\ (0.21)} &  \thead{-2.29 \\ (0.82)} \\
 &Influence on Disc. Loss by AID &  \thead{+0.04 \\ (0.11)} &           \thead{+0.03 \\ (0.13)} &  \thead{-0.13 \\ (0.24)} &           \thead{-0.24 \\ (0.21)} &  \thead{-0.31 \\ (0.23)} &           \thead{-0.45 \\ (0.29)} &  \thead{-0.44 \\ (0.32)} &  \thead{-0.50 \\ (0.30)} &  \thead{-0.61 \\ (0.24)} &  \thead{-2.02 \\ (0.42)} \\
 &Isolation Forest               &  \thead{-0.01 \\ (0.07)} &           \thead{-0.06 \\ (0.13)} &  \thead{-0.02 \\ (0.16)} &           \thead{-0.08 \\ (0.20)} &  \thead{-0.17 \\ (0.26)} &           \thead{-0.20 \\ (0.31)} &  \thead{-0.47 \\ (0.27)} &  \thead{-0.70 \\ (0.33)} &  \thead{-1.13 \\ (0.42)} &  \thead{-2.52 \\ (0.93)} \\
 &Random                         &  \thead{-0.00 \\ (0.04)} &           \thead{+0.01 \\ (0.06)} &  \thead{+0.00 \\ (0.07)} &           \thead{-0.02 \\ (0.11)} &  \thead{+0.01 \\ (0.13)} &           \thead{+0.04 \\ (0.18)} &  \thead{+0.05 \\ (0.16)} &  \thead{+0.04 \\ (0.16)} &  \thead{-0.02 \\ (0.16)} &  \thead{-1.99 \\ (0.84)} \\
\midrule
  \multirow{8}{*}[-1.8cm]{One-epoch retraining} &Influence on FID by ITD (Ours) &           \thead{+0.01 \\ (0.07)} &           \thead{+0.02 \\ (0.09)} &           \thead{+0.03 \\ (0.09)} &  \thead{+0.03 \\ (0.12)} &           \thead{+0.04 \\ (0.15)} &           \thead{+0.04 \\ (0.15)} &           \thead{+0.06 \\ (0.15)} &           \thead{+0.06 \\ (0.15)} &  \thead{+0.04 \\ (0.15)} &  \thead{-0.20 \\ (0.27)} \\
 &Influence on FID by AID (Ours) &  \textbf{\thead{+0.02 \\ (0.04)}} &  \textbf{\thead{+0.03 \\ (0.05)}} &           \thead{+0.03 \\ (0.08)} &  \thead{+0.04 \\ (0.12)} &           \thead{-0.01 \\ (0.16)} &           \thead{-0.07 \\ (0.19)} &           \thead{-0.11 \\ (0.19)} &           \thead{-0.15 \\ (0.23)} &  \thead{-0.27 \\ (0.24)} &  \thead{-0.81 \\ (0.52)} \\
 &Influence on IS by ITD (Ours)  &           \thead{+0.01 \\ (0.05)} &           \thead{+0.02 \\ (0.07)} &  \textbf{\thead{+0.04 \\ (0.08)}} &  \thead{+0.04 \\ (0.12)} &  \textbf{\thead{+0.07 \\ (0.14)}} &  \textbf{\thead{+0.10 \\ (0.15)}} &  \textbf{\thead{+0.10 \\ (0.16)}} &  \textbf{\thead{+0.10 \\ (0.16)}} &  \thead{+0.08 \\ (0.18)} &  \thead{-0.16 \\ (0.30)} \\
 &Influence on IS by AID (Ours)  &  \textbf{\thead{+0.01 \\ (0.02)}} &  \textbf{\thead{+0.01 \\ (0.03)}} &           \thead{+0.01 \\ (0.05)} &  \thead{-0.00 \\ (0.09)} &           \thead{-0.04 \\ (0.14)} &           \thead{-0.10 \\ (0.15)} &           \thead{-0.16 \\ (0.18)} &           \thead{-0.18 \\ (0.19)} &  \thead{-0.33 \\ (0.22)} &  \thead{-1.00 \\ (0.63)} \\
 &Influence on Disc. Loss by ITD &           \thead{-0.07 \\ (0.09)} &           \thead{-0.11 \\ (0.13)} &           \thead{-0.11 \\ (0.15)} &  \thead{-0.14 \\ (0.23)} &           \thead{-0.11 \\ (0.14)} &           \thead{-0.09 \\ (0.15)} &           \thead{-0.08 \\ (0.14)} &           \thead{-0.08 \\ (0.14)} &  \thead{-0.13 \\ (0.15)} &  \thead{-0.22 \\ (0.13)} \\
 &Influence on Disc. Loss by AID &           \thead{+0.00 \\ (0.01)} &           \thead{-0.01 \\ (0.06)} &           \thead{-0.01 \\ (0.10)} &  \thead{-0.04 \\ (0.15)} &           \thead{-0.10 \\ (0.16)} &           \thead{-0.17 \\ (0.17)} &           \thead{-0.25 \\ (0.20)} &           \thead{-0.27 \\ (0.19)} &  \thead{-0.35 \\ (0.21)} &  \thead{-0.77 \\ (0.25)} \\
 &Isolation Forest               &           \thead{+0.00 \\ (0.03)} &           \thead{+0.00 \\ (0.02)} &           \thead{+0.00 \\ (0.02)} &  \thead{+0.00 \\ (0.04)} &           \thead{-0.00 \\ (0.05)} &           \thead{-0.04 \\ (0.07)} &           \thead{-0.08 \\ (0.12)} &           \thead{-0.19 \\ (0.20)} &  \thead{-0.49 \\ (0.34)} &  \thead{-1.84 \\ (0.88)} \\
 &Random                         &           \thead{-0.00 \\ (0.02)} &           \thead{+0.00 \\ (0.01)} &  \textbf{\thead{+0.01 \\ (0.02)}} &  \thead{+0.00 \\ (0.02)} &           \thead{+0.00 \\ (0.03)} &           \thead{-0.01 \\ (0.04)} &           \thead{+0.00 \\ (0.03)} &           \thead{-0.00 \\ (0.04)} &  \thead{-0.00 \\ (0.07)} &  \thead{+0.00 \\ (0.09)} \\
\bottomrule
\end{tabular}
\end{table}

\begin{table}[t]
\centering
\caption{Improvements of Test FID (\(\pm\)STD) after the Data Cleansing of MNIST (Values are Highlighted when the Improvement is Statistically Significant with a Significant Level 0.05)}
\label{table:fid}
\begin{tabular}{llllllllllll}
\toprule
  && \multicolumn{10}{c}{Rate of Instances removed \(n_h / N_x\)}\\  \cmidrule(lr){2-11}
& {} &                              0.01 &                              0.02 &                              0.05 &                              0.10 &                              0.20 &                              0.30 &                              0.40 &                              0.50 &                              0.70 &                       0.90 \\
\midrule
  \multirow{8}{*}[-1.8cm]{Full-epoch retraining} &Influence on FID by ITD (Ours) &  \textbf{\thead{-0.13 \\ (0.17)}} &  \textbf{\thead{-0.15 \\ (0.13)}} &  \textbf{\thead{-0.23 \\ (0.32)}} &  \textbf{\thead{-0.32 \\ (0.34)}} &  \textbf{\thead{-0.37 \\ (0.46)}} &  \textbf{\thead{-0.40 \\ (0.50)}} &  \textbf{\thead{-0.34 \\ (0.45)}} &  \textbf{\thead{-0.38 \\ (0.47)}} &           \thead{-0.12 \\ (0.51)} &  \thead{+12.69 \\ (12.36)} \\
 &Influence on FID by AID (Ours) &           \thead{-0.01 \\ (0.27)} &           \thead{+0.06 \\ (0.41)} &           \thead{+0.51 \\ (0.77)} &           \thead{+0.93 \\ (0.87)} &           \thead{+1.93 \\ (1.50)} &           \thead{+2.48 \\ (2.08)} &           \thead{+3.04 \\ (2.41)} &           \thead{+2.56 \\ (2.03)} &           \thead{+2.55 \\ (1.95)} &    \thead{+6.58 \\ (2.86)} \\
 &Influence on IS by ITD (Ours)  &  \textbf{\thead{-0.10 \\ (0.14)}} &  \textbf{\thead{-0.14 \\ (0.13)}} &  \textbf{\thead{-0.26 \\ (0.27)}} &  \textbf{\thead{-0.32 \\ (0.36)}} &  \textbf{\thead{-0.34 \\ (0.40)}} &  \textbf{\thead{-0.36 \\ (0.55)}} &  \textbf{\thead{-0.35 \\ (0.45)}} &  \textbf{\thead{-0.38 \\ (0.48)}} &           \thead{-0.15 \\ (0.49)} &   \thead{+10.91 \\ (9.35)} \\
 &Influence on IS by AID (Ours)  &           \thead{-0.06 \\ (0.36)} &           \thead{+0.08 \\ (0.38)} &           \thead{+0.47 \\ (0.80)} &           \thead{+0.94 \\ (0.95)} &           \thead{+1.67 \\ (1.46)} &           \thead{+2.82 \\ (1.88)} &           \thead{+2.83 \\ (1.95)} &           \thead{+2.35 \\ (2.05)} &           \thead{+2.75 \\ (2.26)} &    \thead{+6.66 \\ (2.89)} \\
 &Influence on Disc. Loss by ITD &           \thead{-0.07 \\ (0.17)} &           \thead{-0.04 \\ (0.30)} &           \thead{-0.15 \\ (0.39)} &           \thead{-0.17 \\ (0.50)} &  \textbf{\thead{-0.23 \\ (0.51)}} &  \textbf{\thead{-0.33 \\ (0.51)}} &  \textbf{\thead{-0.22 \\ (0.50)}} &  \textbf{\thead{-0.31 \\ (0.45)}} &           \thead{-0.04 \\ (0.53)} &    \thead{+7.31 \\ (4.75)} \\
 &Influence on Disc. Loss by AID &  \textbf{\thead{-0.15 \\ (0.27)}} &           \thead{-0.15 \\ (0.37)} &           \thead{+0.29 \\ (0.61)} &           \thead{+0.64 \\ (0.52)} &           \thead{+0.81 \\ (0.71)} &           \thead{+1.25 \\ (0.92)} &           \thead{+1.27 \\ (0.94)} &           \thead{+1.38 \\ (0.89)} &           \thead{+1.40 \\ (0.63)} &    \thead{+5.01 \\ (1.61)} \\
 &Isolation Forest               &           \thead{+0.08 \\ (0.20)} &           \thead{+0.19 \\ (0.37)} &           \thead{+0.29 \\ (0.48)} &           \thead{+0.58 \\ (0.57)} &           \thead{+1.16 \\ (0.85)} &           \thead{+1.44 \\ (1.05)} &           \thead{+2.50 \\ (0.92)} &           \thead{+3.40 \\ (1.31)} &           \thead{+5.03 \\ (1.74)} &   \thead{+11.36 \\ (4.91)} \\
 &Random                         &           \thead{-0.00 \\ (0.11)} &           \thead{-0.04 \\ (0.15)} &  \textbf{\thead{-0.05 \\ (0.11)}} &           \thead{-0.06 \\ (0.22)} &           \thead{-0.14 \\ (0.34)} &  \textbf{\thead{-0.20 \\ (0.45)}} &  \textbf{\thead{-0.25 \\ (0.35)}} &  \textbf{\thead{-0.34 \\ (0.40)}} &  \textbf{\thead{-0.25 \\ (0.46)}} &    \thead{+5.39 \\ (3.73)} \\
  \midrule
  \multirow{8}{*}[-1.8cm]{One-epoch retraining} &Influence on FID by ITD (Ours) &  \textbf{\thead{-0.11 \\ (0.11)}} &  \textbf{\thead{-0.14 \\ (0.15)}} &  \textbf{\thead{-0.21 \\ (0.17)}} &  \textbf{\thead{-0.28 \\ (0.21)}} &  \textbf{\thead{-0.34 \\ (0.23)}} &  \textbf{\thead{-0.37 \\ (0.26)}} &  \textbf{\thead{-0.40 \\ (0.26)}} &  \textbf{\thead{-0.39 \\ (0.27)}} &  \textbf{\thead{-0.31 \\ (0.28)}} &           \thead{+0.23 \\ (0.65)} \\
 &Influence on FID by AID (Ours) &  \textbf{\thead{-0.05 \\ (0.09)}} &  \textbf{\thead{-0.07 \\ (0.12)}} &  \textbf{\thead{-0.12 \\ (0.19)}} &  \textbf{\thead{-0.16 \\ (0.30)}} &           \thead{-0.06 \\ (0.32)} &           \thead{+0.11 \\ (0.44)} &           \thead{+0.18 \\ (0.40)} &           \thead{+0.30 \\ (0.49)} &           \thead{+0.70 \\ (0.61)} &           \thead{+2.57 \\ (1.98)} \\
 &Influence on IS by ITD (Ours)  &  \textbf{\thead{-0.08 \\ (0.12)}} &  \textbf{\thead{-0.11 \\ (0.14)}} &  \textbf{\thead{-0.17 \\ (0.15)}} &  \textbf{\thead{-0.22 \\ (0.21)}} &  \textbf{\thead{-0.30 \\ (0.23)}} &  \textbf{\thead{-0.36 \\ (0.26)}} &  \textbf{\thead{-0.36 \\ (0.29)}} &  \textbf{\thead{-0.36 \\ (0.29)}} &  \textbf{\thead{-0.25 \\ (0.35)}} &           \thead{+0.28 \\ (0.73)} \\
 &Influence on IS by AID (Ours)  &  \textbf{\thead{-0.04 \\ (0.07)}} &  \textbf{\thead{-0.03 \\ (0.06)}} &           \thead{-0.03 \\ (0.13)} &           \thead{+0.00 \\ (0.28)} &           \thead{+0.09 \\ (0.33)} &           \thead{+0.25 \\ (0.42)} &           \thead{+0.42 \\ (0.47)} &           \thead{+0.50 \\ (0.49)} &           \thead{+0.96 \\ (0.62)} &           \thead{+3.35 \\ (2.49)} \\
 &Influence on Disc. Loss by ITD &           \thead{+0.08 \\ (0.16)} &           \thead{+0.15 \\ (0.23)} &           \thead{+0.14 \\ (0.26)} &           \thead{+0.17 \\ (0.42)} &           \thead{+0.07 \\ (0.27)} &           \thead{+0.05 \\ (0.27)} &           \thead{+0.03 \\ (0.26)} &           \thead{+0.06 \\ (0.25)} &           \thead{+0.25 \\ (0.28)} &           \thead{+0.50 \\ (0.43)} \\
 &Influence on Disc. Loss by AID &  \textbf{\thead{-0.02 \\ (0.04)}} &           \thead{-0.01 \\ (0.12)} &           \thead{-0.05 \\ (0.23)} &           \thead{-0.04 \\ (0.35)} &           \thead{+0.02 \\ (0.44)} &           \thead{+0.17 \\ (0.46)} &           \thead{+0.36 \\ (0.55)} &           \thead{+0.44 \\ (0.50)} &           \thead{+0.68 \\ (0.57)} &           \thead{+1.75 \\ (0.82)} \\
 &Isolation Forest               &           \thead{+0.01 \\ (0.05)} &           \thead{+0.02 \\ (0.06)} &           \thead{+0.04 \\ (0.05)} &           \thead{+0.11 \\ (0.10)} &           \thead{+0.20 \\ (0.12)} &           \thead{+0.37 \\ (0.20)} &           \thead{+0.57 \\ (0.37)} &           \thead{+0.90 \\ (0.59)} &           \thead{+2.05 \\ (1.24)} &           \thead{+7.61 \\ (4.30)} \\
 &Random                         &           \thead{+0.01 \\ (0.03)} &           \thead{-0.01 \\ (0.02)} &  \textbf{\thead{-0.02 \\ (0.03)}} &           \thead{-0.01 \\ (0.05)} &           \thead{-0.01 \\ (0.07)} &           \thead{-0.01 \\ (0.08)} &           \thead{-0.02 \\ (0.08)} &           \thead{-0.03 \\ (0.08)} &           \thead{-0.04 \\ (0.15)} &  \textbf{\thead{-0.16 \\ (0.21)}} \\
\bottomrule
\end{tabular}
\end{table}
\begin{figure*}[t]
\begin{minipage}{\linewidth}
\centering
\subfloat[Harmful (Influence on FID by ITD with Full-epoch)]{
\includegraphics[width=0.48\linewidth,clip]{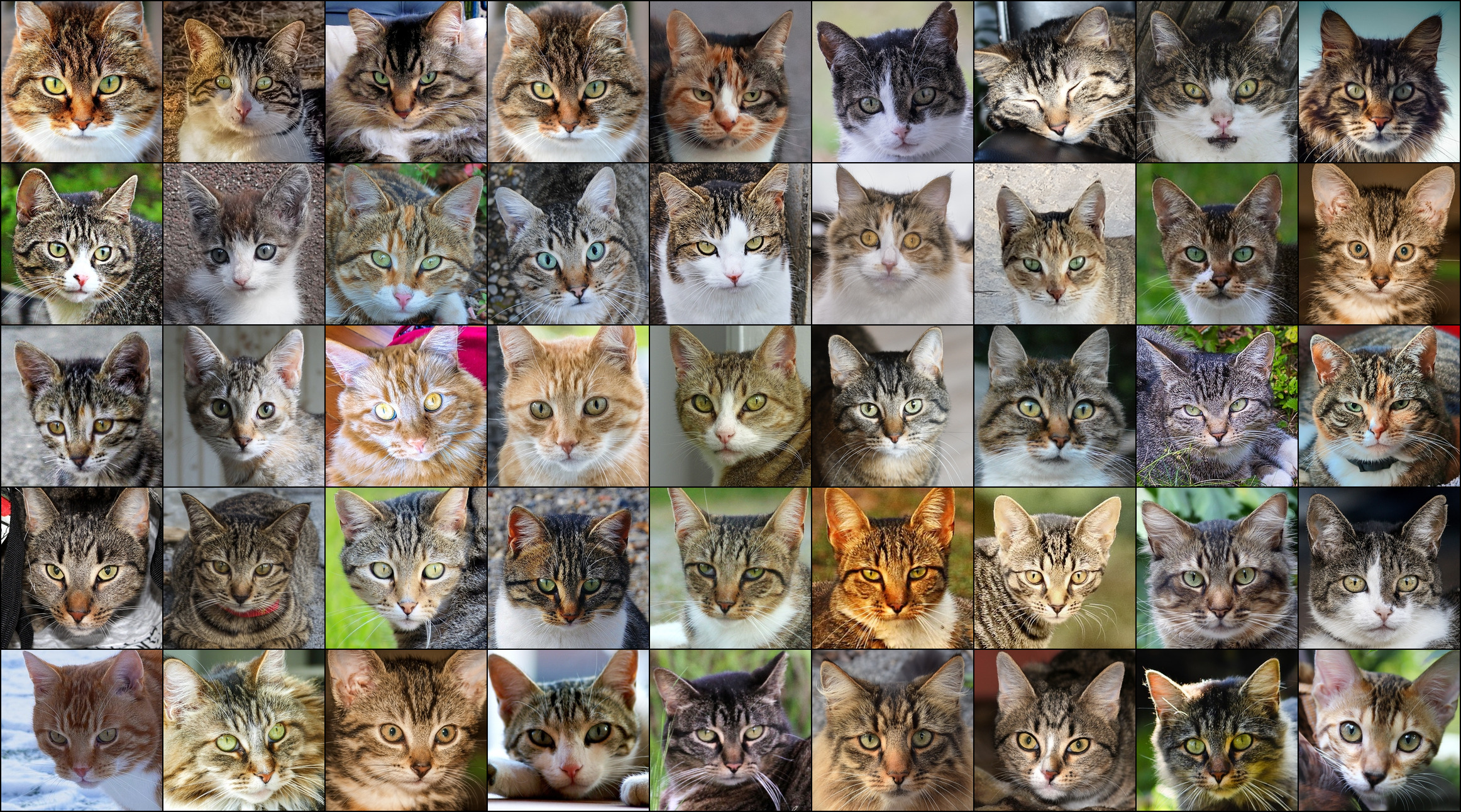}
    \label{sub:fid_itd_harmful_50epoch}}
\hfill
\subfloat[Helpful (Influence on FID by ITD with Full-epoch)]{
\includegraphics[width=0.48\linewidth,clip]{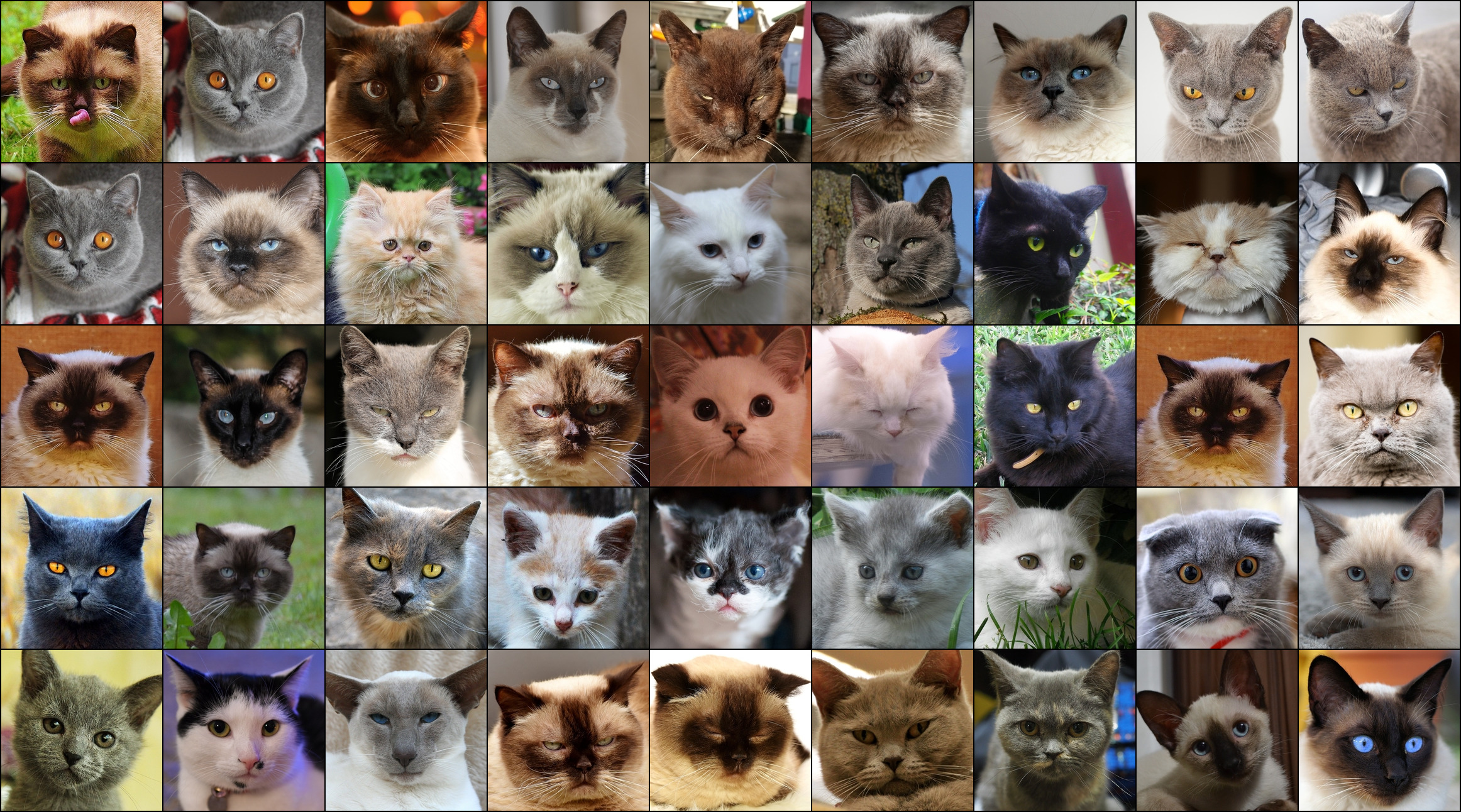}
    \label{sub:fid_itd_helpful_50epoch}}

\subfloat[Harmful (Influence on FID by ITD with One-epoch)]{
\includegraphics[width=0.48\linewidth,clip]{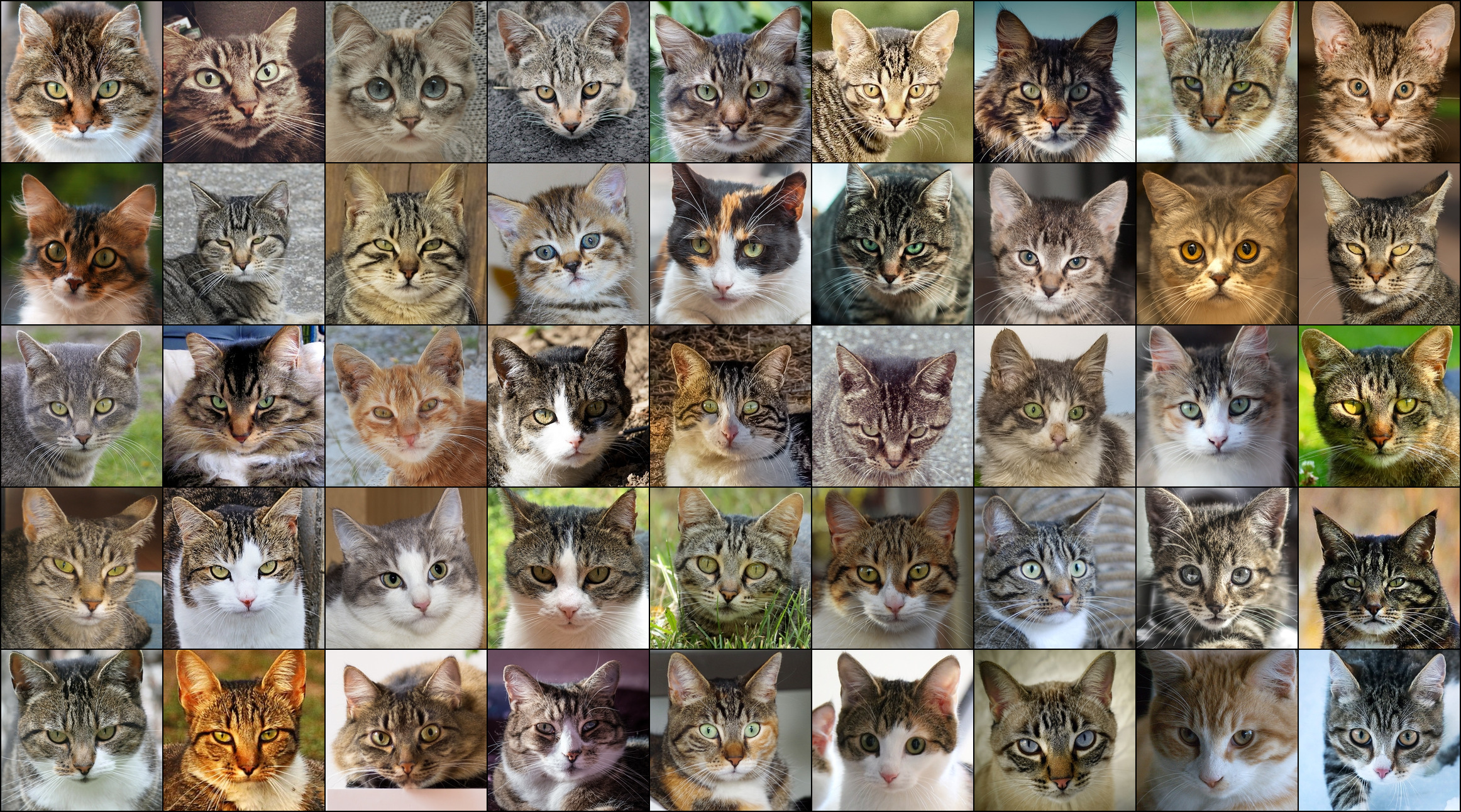}
    \label{sub:fid_itd_harmful_1epoch}}
\hfill
\subfloat[Helpful (Influence on FID by ITD with One-epoch)]{
\includegraphics[width=0.48\linewidth,clip]{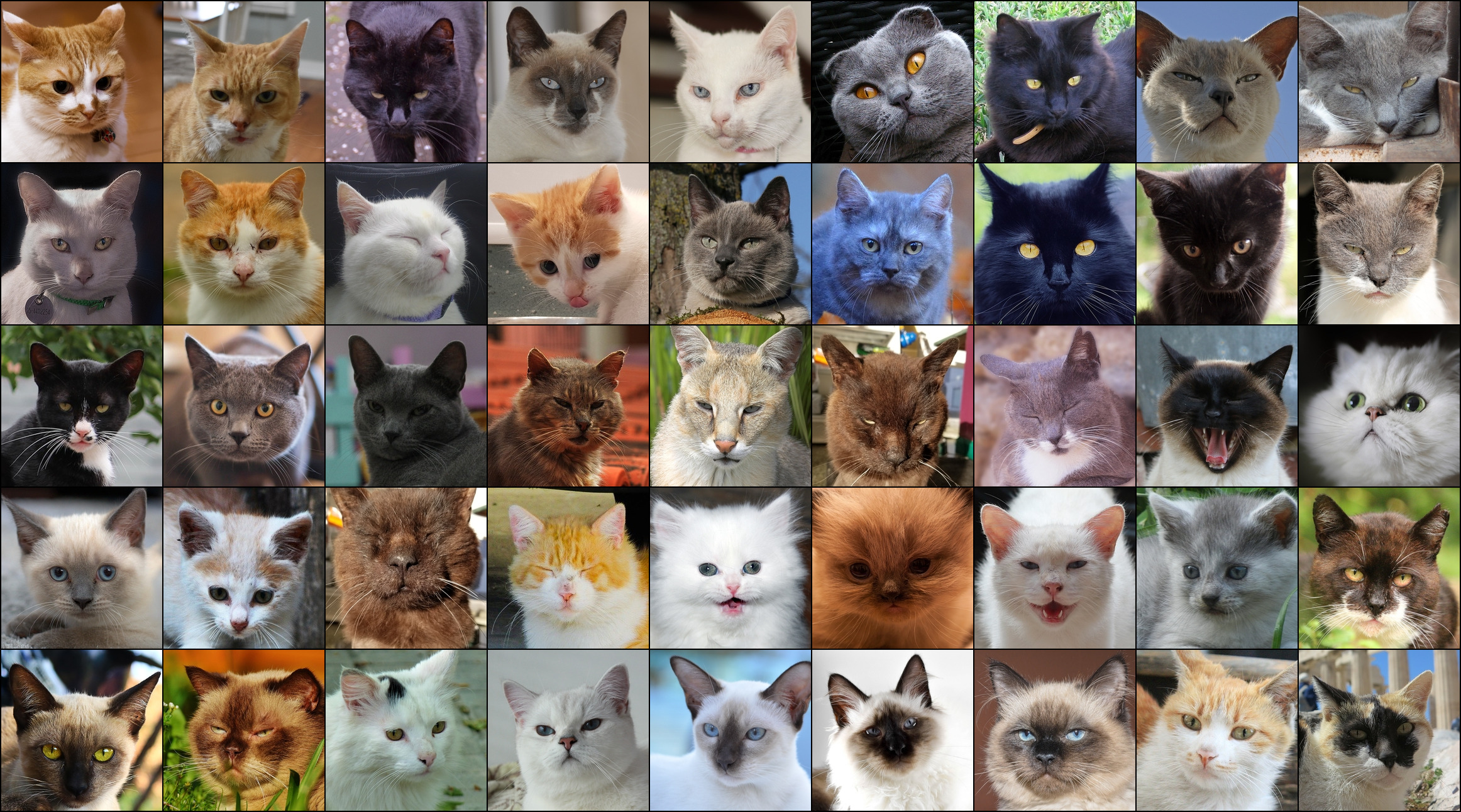}
    \label{sub:fid_itd_helpful_1epoch}}

\subfloat[Harmful (Influence on FID by AID)]{
\includegraphics[width=0.48\linewidth,clip]{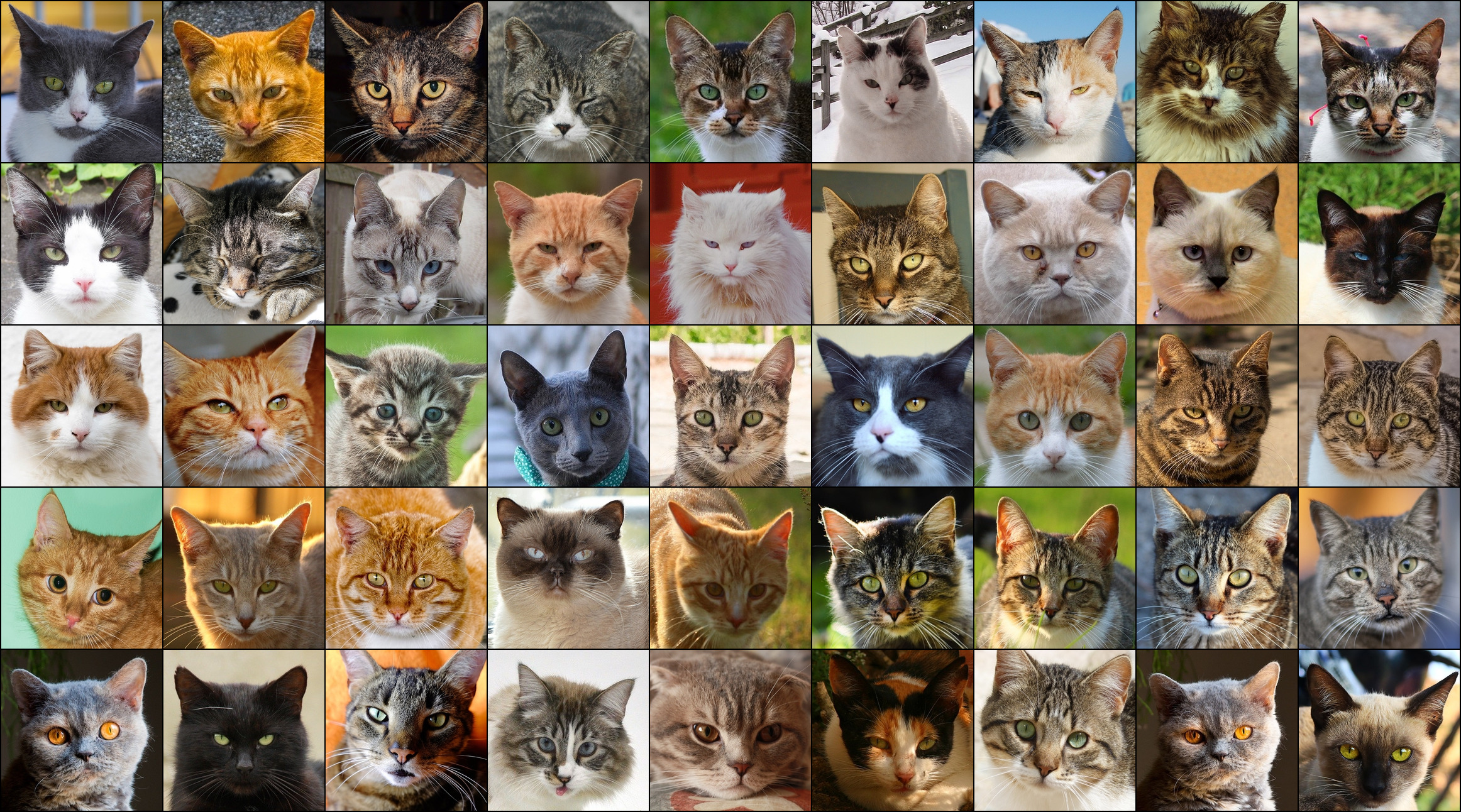}
    \label{sub:fid_aid_harmful_50epoch}}
\hfill
\subfloat[Helpful (Influence on FID by AID)]{
\includegraphics[width=0.48\linewidth,clip]{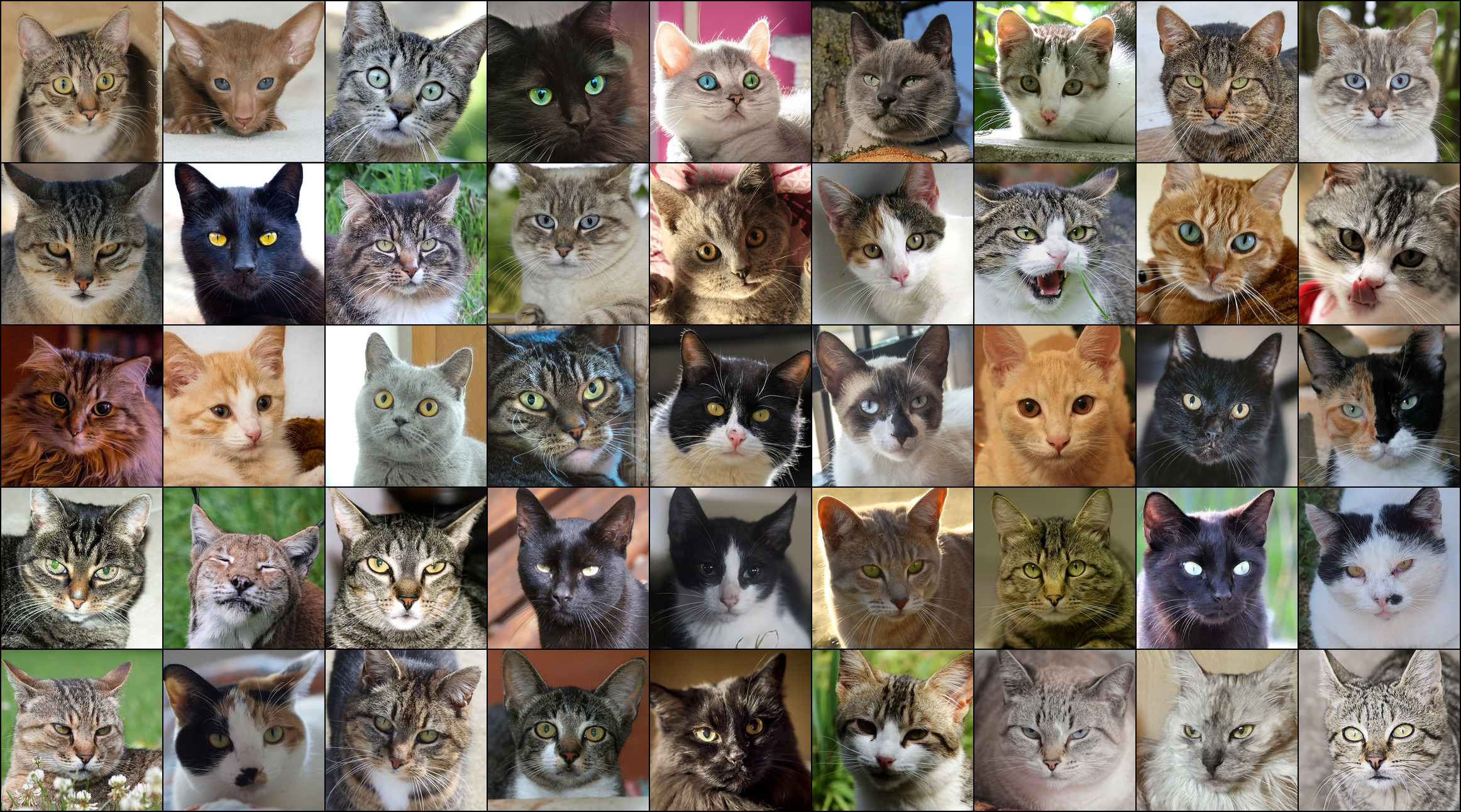}
    \label{sub:fid_aid_helpful_50epoch}}

\subfloat[Harmful (Isolation Forest)]{
\includegraphics[width=0.48\linewidth,clip]{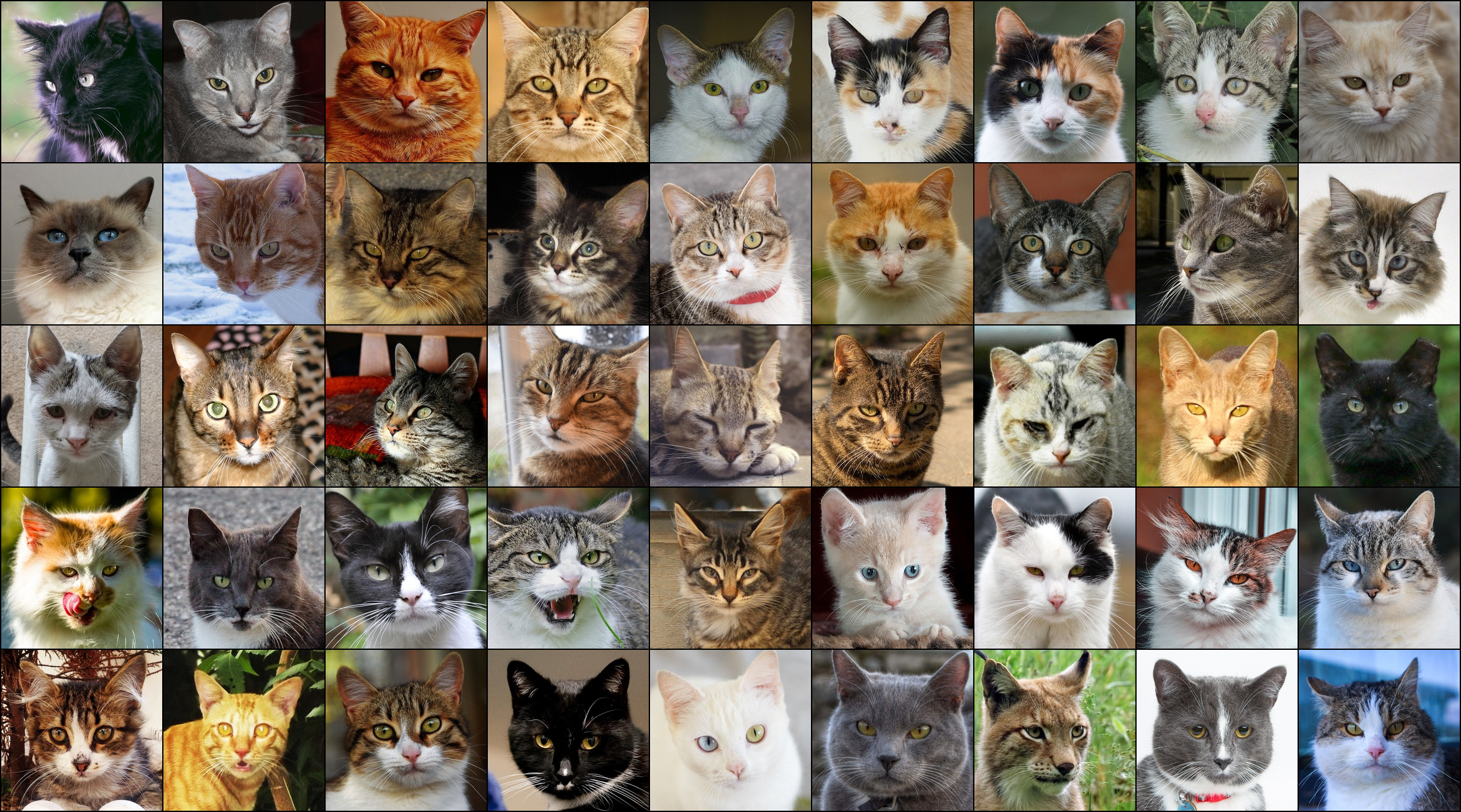}
    \label{sub:if_harmful_50epoch}}
\hfill
\subfloat[Helpful (Isolation Forest)]{
\includegraphics[width=0.48\linewidth,clip]{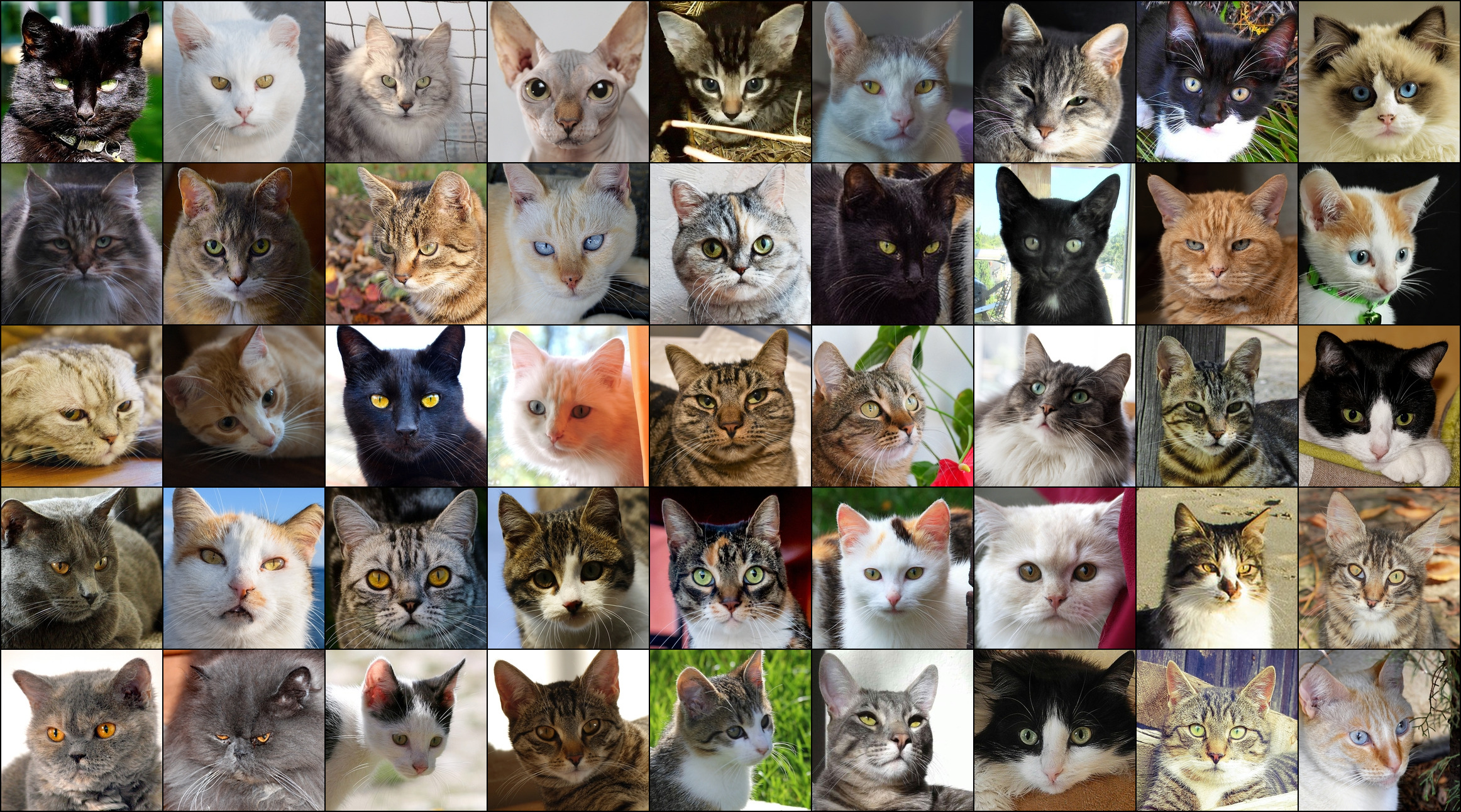}
    \label{sub:if_helpful_50epoch}}
\caption[]{
\noindent Top 45 harmful and helpful instances suggested by our approaches and the isolation forest. 
\subref{sub:fid_itd_harmful_50epoch} and \subref{sub:fid_itd_helpful_50epoch} show the harmful and helpful instances predicted by ITD-EIGEM that traced back full-epochs of fine-tuning, while \subref{sub:fid_itd_harmful_1epoch} and \subref{sub:fid_itd_helpful_1epoch} show those predicted by ITD-EIGEM that traced back only the last epoch.
}
\label{fig:instance_comparison}
\end{minipage}
\end{figure*}

\def\widthcat{0.16}
\def\scalecat{1.0}
\begin{figure}[t]
\begin{minipage}{0.49\linewidth}
\centering
\begin{minipage}{\widthcat\linewidth}
\centering
\subfloat[Ori.]{
\includegraphics[width=\scalecat\linewidth,clip]{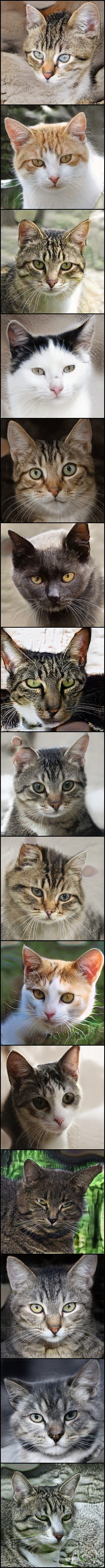}
    \label{sub:visual_cat_no_removal_app1}}
\end{minipage} 
\begin{minipage}{\widthcat\linewidth}
\centering
\subfloat[ITD]{
\includegraphics[width=\scalecat\linewidth,clip]{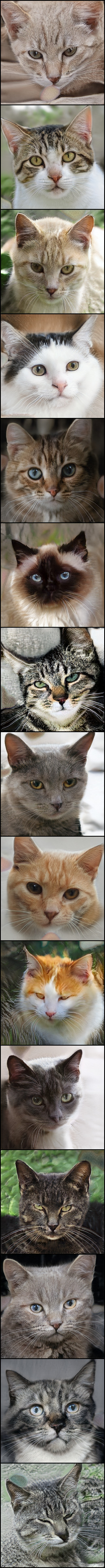}
    \label{sub:visual_cat_fid_itd_app1}}
\end{minipage} 
\begin{minipage}{\widthcat\linewidth}
\centering
\subfloat[AID]{
\includegraphics[width=\scalecat\linewidth,clip]{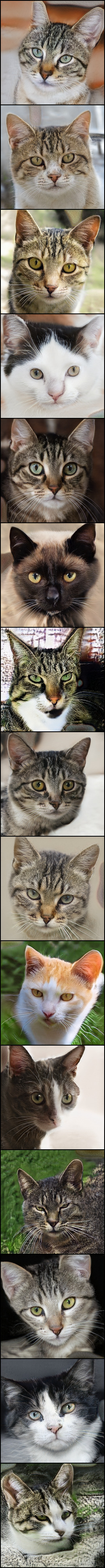}
    \label{sub:visual_cat_fid_aid_app1}}
\end{minipage} 
\begin{minipage}{\widthcat\linewidth}
\centering
\subfloat[IF]{
\includegraphics[width=\scalecat\linewidth,clip]{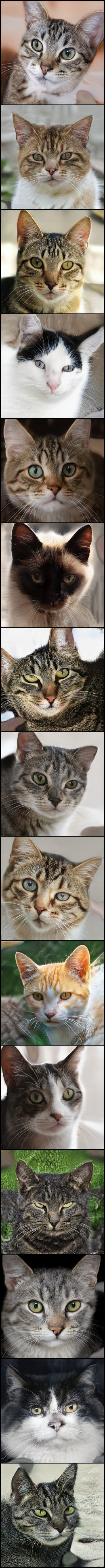}
    \label{sub:visual_cat_if_app1}}
\end{minipage} 
\begin{minipage}{\widthcat\linewidth}
\centering
\subfloat[Rand.]{
\includegraphics[width=\scalecat\linewidth,clip]{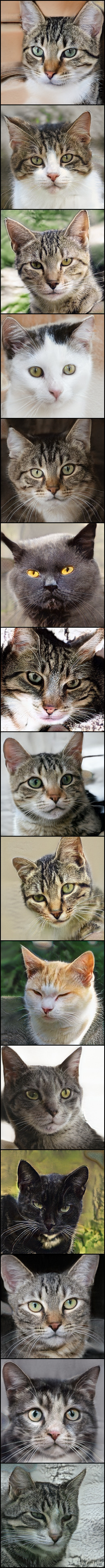}
    \label{sub:visual_cat_random_app1}}
\end{minipage} 
\caption[]{Generated images before and after data cleansing using different methods. ``Ori.'' represents the model without data cleansing, ``ITD'' refers to ITD-EIGEM with full-epoch iterations, ``AID'' denotes AID-EIGEM, ``IF''stands for isolation forest, and ``Rand.'' indicates the random selection. These images are generated by the same procedure as in \cref{fig:visual_cat_main} with different latent variables from \cref{fig:visual_cat_app1}.}
\label{fig:visual_cat_app1}
\end{minipage}
\hfill
\begin{minipage}{0.49\linewidth}
\centering
\begin{minipage}{\widthcat\linewidth}
\centering
\subfloat[Ori.]{
\includegraphics[width=\scalecat\linewidth,clip]{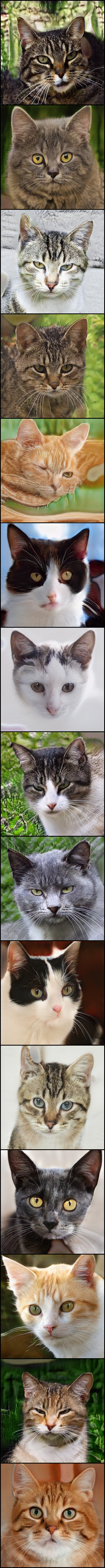}}
\end{minipage} 
\begin{minipage}{\widthcat\linewidth}
\centering
\subfloat[ITD]{
\includegraphics[width=\scalecat\linewidth,clip]{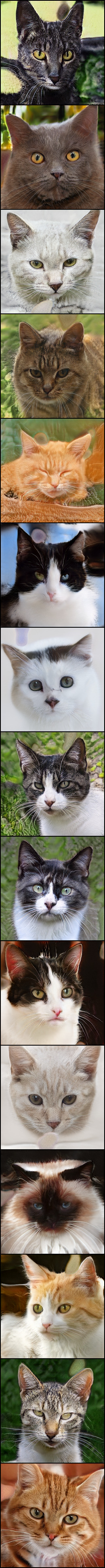}}
\end{minipage} 
\begin{minipage}{\widthcat\linewidth}
\centering
\subfloat[AID]{
\includegraphics[width=\scalecat\linewidth,clip]{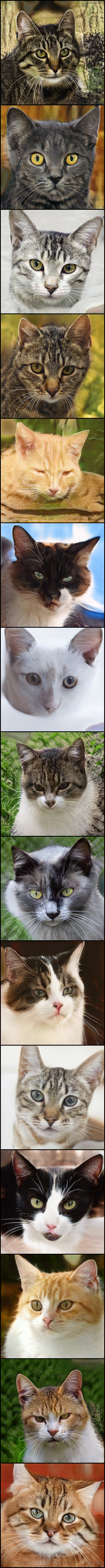}}
\end{minipage} 
\begin{minipage}{\widthcat\linewidth}
\centering
\subfloat[IF]{
\includegraphics[width=\scalecat\linewidth,clip]{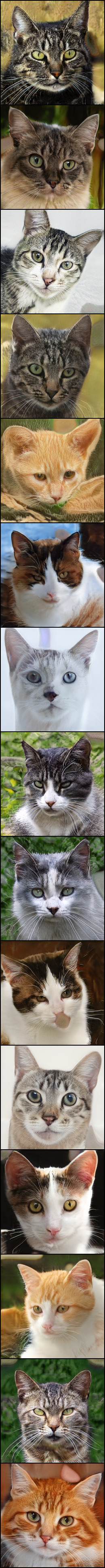}}
\end{minipage} 
\begin{minipage}{\widthcat\linewidth}
\centering
\subfloat[Rand.]{
\includegraphics[width=\scalecat\linewidth,clip]{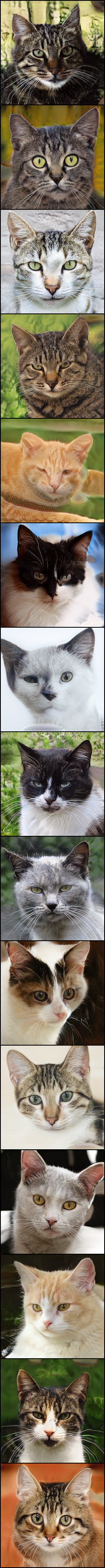}}
\end{minipage} 
\caption[]{Generated images before and after the data cleansing generated using the same procedure as in \cref{fig:visual_cat_main} with different latent variables from \cref{fig:visual_cat_app1}.}
\label{fig:visual_cat_app2}
\end{minipage}
\end{figure}

\clearpage
\subsection{Applicability to a Larger GAN: StyleGAN2}
\noindent To evaluate the scalability and effectiveness of our data cleansing method on more complex models, we extended our experiments to StyleGAN2~\cite{karras2020analyzing}, a larger and more advanced GAN architecture.

\subsubsection{Settings}
\noindent We followed the same experimental setup as with StyleGAN to ensure consistency, with the primary difference being the architectural enhancements inherent to StyleGAN2.
In this experiment, we also applied LoRA with a rank of 32 to all fully connected and convolutional layers in both the generator and discriminator.

We used ITD-EIGEM for identifying harmful instances because it showed better performance than AID-EIGEM in the StyleGAN case (\cref{sec:exp2}).
Random instance removal served as a baseline to ensure that any performance improvements were due to the precise elimination of harmful data points rather than simply reducing the dataset size.

\subsubsection{Results}
\noindent \cref{fig:clean_stylegan2}\subref{sub:clean_stylegan2} and \ref{fig:clean_stylegan2}\subref{sub:clean_stylegan2_1epoch} show the results of data cleansing using full-epoch retraining and one-epoch retraining strategies, respectively.
The combination of ITD-EIGEM with full-epoch retraining led to noticeable improvements in test FID scores, indicating enhanced generative performance after cleansing.
Consistent with our observations in the StyleGAN experiments, full-epoch retraining generally yielded better performance than one-epoch retraining.
We observed that random instance removal often degraded the test FID scores.
This confirms that the performance gains from our method are not simply due to reducing the training dataset size but are the result of accurately identifying and removing harmful instances.

Overall, these results demonstrate that our data cleansing method is broadly applicable and effective across various GAN architectures, including large and complex models like StyleGAN2.

\begin{figure*}[t]
\centering
    \begin{minipage}{0.9\linewidth}
        \centering
        \includegraphics[trim={-2.0cm -0.5cm 0.0cm 0.0cm},clip,width=0.7\linewidth]{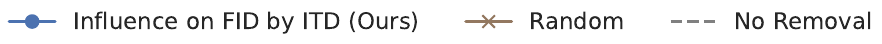}
    \end{minipage}
    \newline
    \begin{minipage}{0.45\linewidth}
    \centering
    \subfloat[][Full-epoch retraining]{
        \includegraphics[trim={.0cm .8cm 1.5cm 1cm},clip,width=\linewidth]{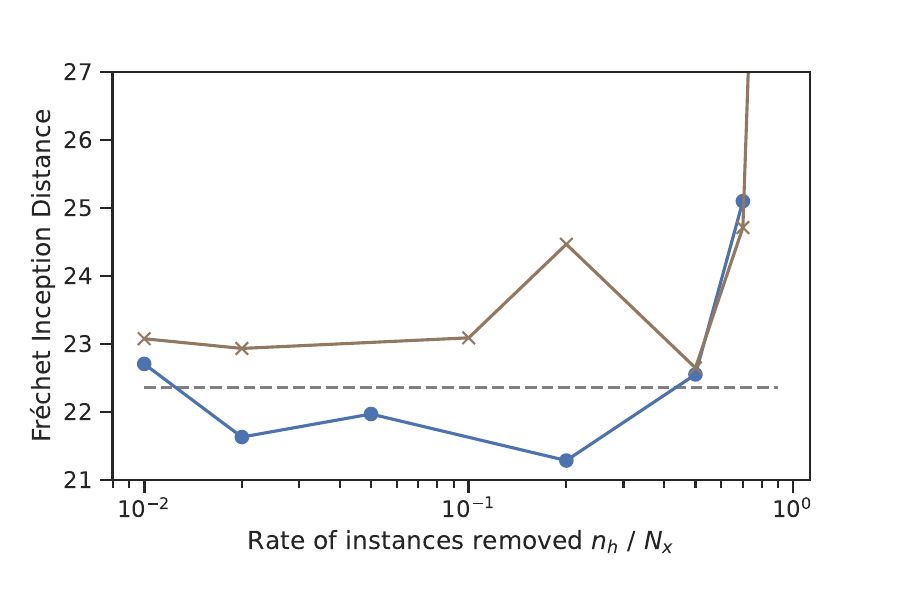}
        \label{sub:clean_stylegan2}
    }
    \end{minipage}
    \hspace{0.52cm}
    \begin{minipage}{0.45\linewidth}
    \centering
    \subfloat[][One-epoch retraining]{
        \includegraphics[trim={.0cm .8cm 1.5cm 1cm},clip,width=\linewidth]{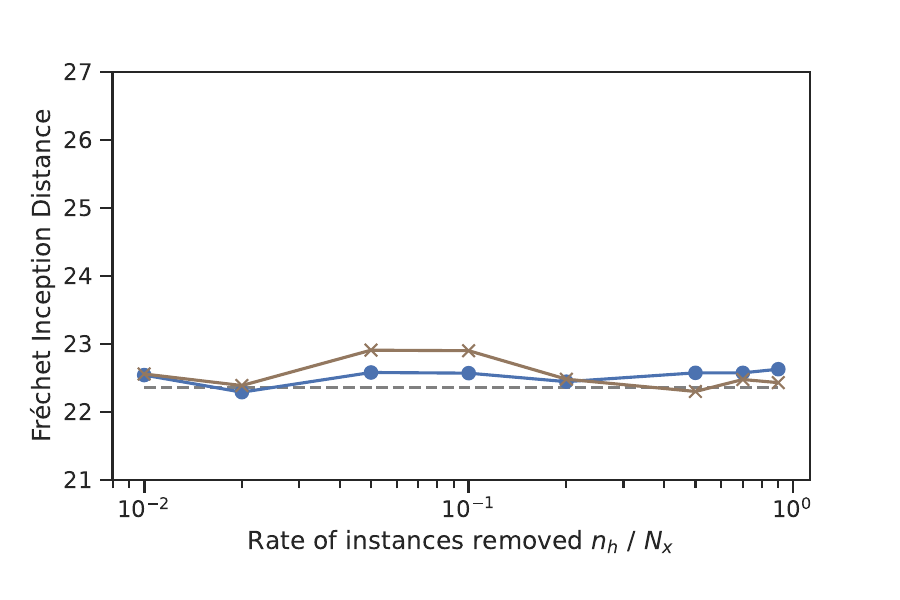}
        \label{sub:clean_stylegan2_1epoch}
}
    \end{minipage}
    \caption[]{
The average test FID after the data cleansing on StyleGAN2 finetuned for AFHQ-CAT.
A higher value indicates better generative performance.
The test FID after full-epoch retraining by the random removal with the removal rate of 0.05 and by the influence estimation with removal rates of 0.1 and 0.9 are not shown as their retraining diverged.
}
    \label{fig:clean_stylegan2}
\end{figure*}

\putbib[main]
\end{bibunit}

\end{document}